\documentclass[]{style}
\usepackage{float}
\usepackage{graphicx}
\usepackage{titletoc}

\usepackage[english]{babel}
\usepackage{xcolor}
\usepackage{color}
\usepackage{wrapfig}
\usepackage[table]{xcolor}

\usepackage{amsmath}
\usepackage{amssymb}
\usepackage{amsthm}
\usepackage{graphicx}
\usepackage{booktabs}      
\usepackage{array}
\usepackage{multirow}      
\usepackage{caption}  
\usepackage{amsmath}
\usepackage[table]{xcolor}  
\usepackage{xcolor}
\usepackage{microtype}
\usepackage{parskip}

\definecolor{bg}{HTML}{E7E7FE}
\definecolor{accent}{HTML}{6B67EE}
\definecolor{accent-bright}{HTML}{6B67EE}
\definecolor{mybg}{HTML}{EAE9FF} 
\definecolor{lightbg}{HTML}{F7F7FF}
\definecolor{frame}{HTML}{9793F8}

\usepackage{listings,xcolor}
\definecolor{infillkw}{RGB}{0,0,180}
\definecolor{infillcm}{RGB}{180,30,30}
\lstdefinestyle{infillpy}{
  language=Python, basicstyle=\ttfamily\scriptsize,
  keywordstyle=\color{infillkw}, commentstyle=\color{infillcm}\bfseries,
  showstringspaces=false, breaklines=true, columns=fullflexible,
  aboveskip=2pt, belowskip=2pt,
}
\lstdefinestyle{gsmtext}{basicstyle=\scriptsize\ttfamily, breaklines=true,
   columns=fullflexible, keepspaces=true, breakindent=0pt, aboveskip=2pt, belowskip=2pt}

\newtcolorbox{figurepanel}{
    enhanced,
    colback=mybg,
    colframe=frame,
    boxrule=0.8pt,
    arc=6pt,
    left=10pt,right=10pt,top=10pt,bottom=10pt
}

\newtcolorbox{samplebox}{
    enhanced,
    colback=lightbg,
    colframe=frame,
    boxrule=0.6pt,
    arc=4pt,
    left=8pt,right=8pt,top=7pt,bottom=7pt
}
\newcommand{\nfeheader}[1]{%
    {\sffamily\bfseries\large #1}
}
\newcommand{\methodtitle}[1]{%
    {\bfseries #1}
}

\usepackage{algorithm}
\usepackage{algpseudocode}
\newcommand{\insertat}[1]{\triangleleft_{#1}}
\usepackage[noEnd=false]{algpseudocodex}

\usepackage{amsmath,amsfonts,bm}

\def\1{\bm{1}}

\DeclareMathAlphabet{\mathsfit}{\encodingdefault}{\sfdefault}{m}{sl}
\SetMathAlphabet{\mathsfit}{bold}{\encodingdefault}{\sfdefault}{bx}{n}

\usepackage{thmtools}
\usepackage{thm-restate}

\newtcolorbox{propbox}{
    enhanced,
    colback=accent!4!bg!60!white,
    colframe=accent, 
    boxrule=0pt,
    leftrule=0pt,      
    arc=0pt,
    left=8pt, right=8pt, top=8pt, bottom=8pt,
    fonttitle=\bfseries\sffamily,
    coltitle=accent-bright
}
\tcolorboxenvironment{proposition}{enhanced,
    colback=accent!4!bg!60!white,colframe=accent,boxrule=0pt,leftrule=0pt,arc=0pt,
    left=8pt, right=8pt, top=8pt,bottom=8pt,fonttitle=\bfseries\sffamily,coltitle=accent-bright
}
\tcolorboxenvironment{remark}{enhanced,
    colback=accent!4!bg!60!white,colframe=accent,boxrule=0pt,leftrule=0pt,arc=0pt,
    left=8pt, right=8pt, top=8pt,bottom=8pt,fonttitle=\bfseries\sffamily,coltitle=accent-bright
}
\tcolorboxenvironment{definition}{enhanced,
    colback=accent!4!bg!60!white,colframe=accent,boxrule=0pt,leftrule=0pt,arc=0pt,
    left=8pt, right=8pt, top=8pt,bottom=8pt,fonttitle=\bfseries\sffamily,coltitle=accent-bright
}
\tcolorboxenvironment{box*}{enhanced,
    colback=accent!4!bg!60!white,colframe=accent,boxrule=0pt,leftrule=0pt,arc=0pt,
    left=8pt, right=8pt, top=8pt,bottom=8pt,fonttitle=\bfseries\sffamily,coltitle=accent-bright
}
\tcolorboxenvironment{lemma}{enhanced,
    colback=accent!4!bg!60!white,colframe=accent,boxrule=0pt,leftrule=0pt,arc=0pt,
    left=8pt, right=8pt, top=8pt,bottom=8pt,fonttitle=\bfseries\sffamily,coltitle=accent-bright
}
\tcolorboxenvironment{restatable}{enhanced,
    colback=accent!4!bg!60!white,colframe=accent,boxrule=0pt,leftrule=0pt,arc=0pt,
    left=8pt, right=8pt, top=8pt,bottom=8pt,fonttitle=\bfseries\sffamily,coltitle=accent-bright
}
\tcolorboxenvironment{theorem}{enhanced,
    colback=accent!4!bg!60!white,colframe=accent,boxrule=0pt,leftrule=0pt,arc=0pt,
    left=8pt, right=8pt, top=8pt,bottom=8pt,fonttitle=\bfseries\sffamily,coltitle=accent-bright
}
\tcolorboxenvironment{corollary}{enhanced,
    colback=accent!4!bg!60!white,colframe=accent,boxrule=0pt,leftrule=0pt,arc=0pt,
    left=8pt, right=8pt, top=8pt,bottom=8pt,fonttitle=\bfseries\sffamily,coltitle=accent-bright
}

\declaretheorem[
  name=Definition,
  numberwithin=section,
  refname={Definition,Definition},
  Refname={Definition,Definition}
]{definition}

\declaretheorem[
  name=Proposition,
  numberwithin=section,
  refname={Proposition,Propositions},
  Refname={Proposition,Propositions}
]{proposition}

\declaretheorem[
  name=Lemma,
  numberwithin=section,
  refname={Lemma,Lemmas},
  Refname={Lemma,Lemmas}
]{lemma}

\declaretheorem[
  name=Corollary,
  numberwithin=section,
]{corollary}

\declaretheorem[]{box*}

\title{\fontsize{0.6cm}{0.72cm}\selectfont A2D2: Fine-Tuning Any-Length Discrete \\[0.3cm] Diffusion for Adaptive Decoding}

\author[1]{Sophia Tang}
\author[2]{\, Yuchen Zhu}
\author[2]{\, Molei Tao}
\author[1,3]{\, Pranam Chatterjee}

\affiliation[1]{Department of Computer and Information Science, University of Pennsylvania}
\affiliation[2]{School of Mathematics, Georgia Institute of Technology}
\affiliation[3]{Department of Bioengineering, University of Pennsylvania}

\abstract{Discrete diffusion models offer a simple and stable likelihood-based framework for sequence generation, recently extended to any-length settings via token insertion. Principled reward-guided fine-tuning for any-length discrete diffusion, however, remains largely unexplored. We introduce Fine-Tuning \textbf{A}ny-Length \textbf{D}iscrete \textbf{D}iffusion for \textbf{A}daptive Decoding (\textbf{A2D2}), a unified framework for reward-guided fine-tuning of any-length discrete diffusion models via joint optimization of the insertion and unmasking policies together with a quality-based inference schedule. We derive the Radon–Nikodym derivative for the joint insertion–unmasking path measures, enabling theoretically guaranteed convergence to the intractable reward-tilted sequence distribution without requiring target samples. Building on this, we establish unmasking and insertion quality as tractable approaches for minimizing decoding error and introduce the \textbf{Adaptive Joint Decoding (AJD)} loss, which provably yields the optimal path measure that generates the reward-tilted distribution. Empirically, A2D2 improves reward optimization while enhancing generation flexibility and accuracy over prior fixed-length fine-tuning and inference-time guidance methods.

\vspace{0.5em}%
\textbf{\sffamily\bfseries Correspondance: }\href{mailto:sophtang@seas.upenn.edu}{\texttt{sophtang@seas.upenn.edu}}, \href{mailto:pranam@upenn.edu}{\texttt{pranam@upenn.edu}}}
\begin{document}
\maketitle

\renewcommand{\footnoterule}{%
  \kern -3pt
  \hrule width \linewidth
  \kern 2.6pt
}
\vspace{-0.2cm}

\section{Introduction}
Discrete diffusion models \citep{lou2023discrete, austin2021structured} have emerged as a leading paradigm for sequence generation, overcoming key limitations of autoregressive models through bidirectional context, any-order generation \citep{shi2024simplified}, flexible reward guidance \citep{nisonoff2024unlocking, tang2025peptune}, and parallel decoding \citep{christopher2025speculative, ren2025fast}. These properties have driven strong performance across biological sequence design \citep{wang2024dplm, gruver2023protein}, reasoning \citep{ye2024beyond, zhu2025enhancing}, and efficient sampling \citep{holderrieth2025leaps, zhu2025mdns}. Among these approaches, \textbf{masked discrete diffusion models} (MDMs) \citep{sahoo2024simple, shi2024simplified, ou2024your, zheng2024masked} stand out for their simple design and stable likelihood-based objective, and have recently been extended to any-length generation by allowing masked-token insertion at arbitrary positions during the diffusion process \citep{kim2025any}. However, scaling any-length discrete diffusion to reward optimization and fine-tuning remains largely unexplored.

While principled approaches have been introduced for reward fine-tuning fixed-length discrete diffusion models \citep{tang2025tr2, zekri2025fine, wang2025fine, cao2025glid}, any-length MDMs introduce a substantially larger action space that spans both unmasking and variable-length insertions, making performance highly sensitive to the chosen trajectory. This motivates our \textbf{key question}:

\begin{center}
    \textit{Can we efficiently fine-tune any-length discrete diffusion models to sample from an intractable reward-tilted distribution while preserving high generation quality?}
\end{center}

To answer this, we introduce Fine-Tuning \textbf{A}ny-Length \textbf{D}iscrete \textbf{D}iffusion for \textbf{A}daptive Decoding (\textbf{A2D2}), a unified framework for reward-guided fine-tuning of any-length discrete diffusion models via joint optimization of the insertion and unmasking policies and quality-based inference schedule. 
\begin{figure*}
    \centering
    \includegraphics[width=0.9\linewidth]{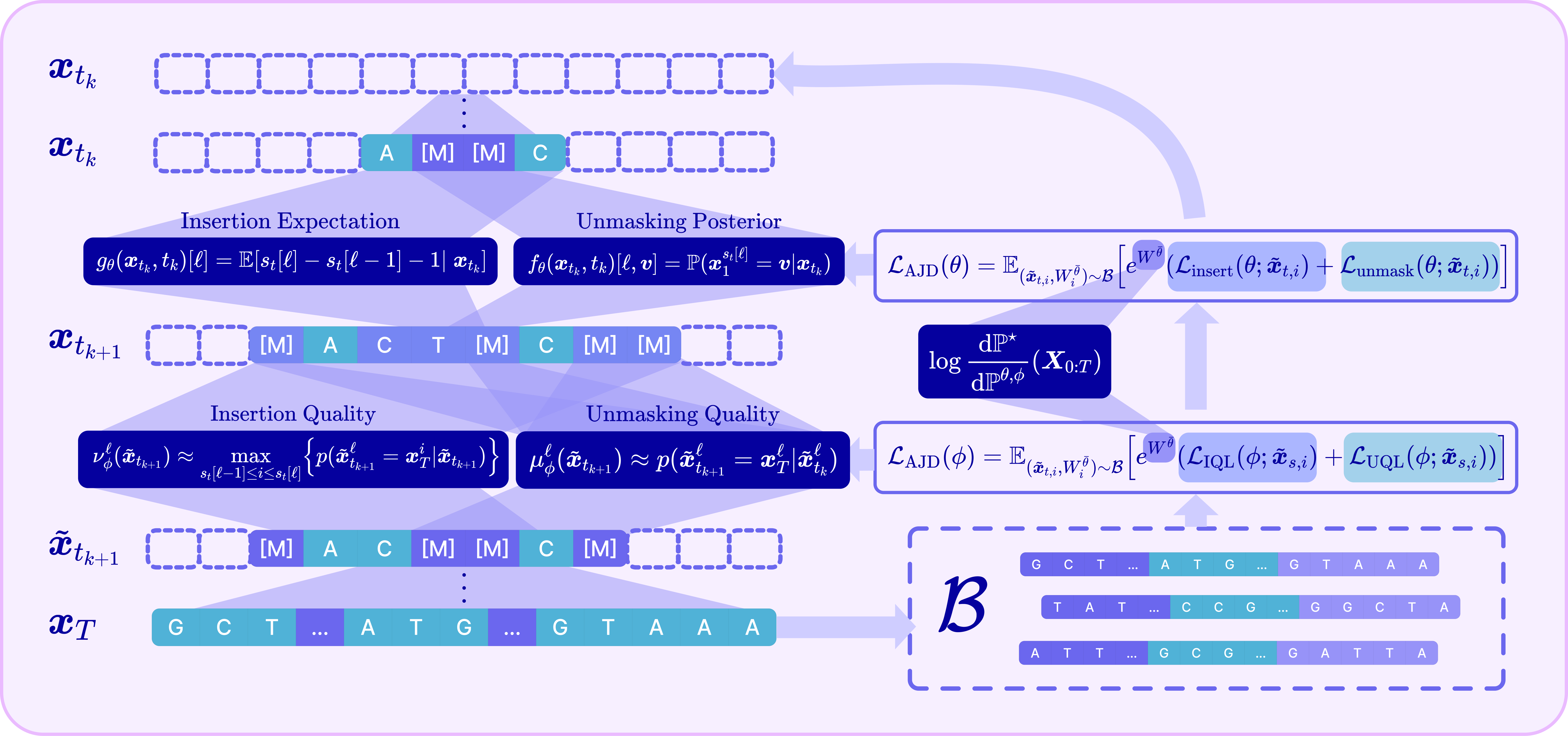}
    \caption{\textbf{Fine-Tuning Any-Length Discrete Diffusion for Adaptive Decoding.} A2D2 jointly fine-tunes the insertion and unmasking policies together with a quality classifier to sample from the optimal reward-tilted distribution $\mathbb{P}^\star$. Sequences are generated via adaptive inference and stored in a replay buffer, then used to compute the Adaptive Joint Decoding loss $\mathcal{L}_{\text{AJD}}$, which weights each sequence's contribution by an importance weight $W^{\bar{\theta},\bar{\phi}}$.}
    \label{fig:1}
\end{figure*}

Our \textbf{main contributions} can be summarized in the following points: 
\begin{enumerate}
    \item  We derive the Radon-Nikodym derivative for the joint insertion and unmasking path measures, which enables theoretically guaranteed convergence to an intractable reward-tilted sequence distribution.
    \item We establish unmasking and insertion quality as tractable methods of minimizing compounding parallelization error (CPE) and introduce the \textbf{Adaptive Joint Decoding (AJD)} loss, which provably yields the optimal path measure that \textit{minimizes error} and \textit{generates the reward-tilted distribution}. 
    \item We demonstrate that A2D2 simultaneously optimizes rewards while enhancing generation flexibility and accuracy over prior fixed-length fine-tuning and inference-time guidance approaches.
\end{enumerate}

\paragraph{Related Works} A discussion of additional related works is provided in App \ref{app:related-work}.

\section{Preliminaries}

\paragraph{Masked Discrete Diffusion Models}
The discrete diffusion paradigm models sequence generation as a continuous-time Markov chain (CTMC), which starts at an uninformative prior distribution $p_{\text{prior}}$ and makes discrete jumps over a finite time horizon $t\in [0,1]$ to generate a sample from the data distribution $p_{\text{data}}$. \textbf{Masked discrete diffusion models} (MDMs) have a unique prior distribution $p_{\text{prior}}$ consisting of fully masked sequences $(\boldsymbol{M})^L$ \citep{shi2024simplified, sahoo2024simple, ou2024your, zheng2024masked}. Given a clean data distribution $p_{\text{data}}$, the MDM's forward process progressively converts clean tokens to masked tokens according to a schedule. Then, the training process aims to reconstruct the clean sequence $\boldsymbol{x}_1\sim p_{\text{data}}$ from intermediate partially masked sequences $\boldsymbol{x}_t$ by minimizing the \textbf{denoising cross-entropy (DCE)} loss:
\begin{align}
    \mathcal{L}_{\text{DCE}}(\theta;\boldsymbol{x}_1):=\underset{t\sim \mathcal{U}(0,1)}{\mathbb{E}}\underset{\boldsymbol{x}_t\sim p_t(\cdot |\boldsymbol{x}_1)}{\mathbb{E}} \left[-\frac{\dot{\alpha}_t}{1-\alpha_t}\sum_{\ell:\boldsymbol{x}^\ell_t=\boldsymbol{M}}\log f_\theta(\boldsymbol{x}_t,t)[\ell, \boldsymbol{x}_1^{s_t[\ell]}]\right]
\end{align}
where $\alpha_t:[0,1]\to [0, 1]$ is the unmasking schedule bounded by $\alpha_0=0$ and $\alpha_1=1$ where $\alpha_t$ is the probability that a given token is unmasked at time $t$ in the forward masking process and the unmasking occurs in the reverse process if the time $t$ is greater than the unmasking time $t_u^\ell\sim \dot{\alpha}_tdt$, i.e. $t\geq t_u^\ell$. Under the standard linear schedule, $\alpha_t=t$.

\paragraph{Any-Length Discrete Diffusion}
To overcome fixed-length initialization, any-length discrete diffusion models \citep{havasi2025edit, kim2025any} \textit{insert} tokens along the generation process. We build on the \textbf{flexible-length masked diffusion model} framework of \citet{kim2025any}, which defines a \textit{joint stochastic interpolant} over insertion and unmasking.

The forward corruption process gradually masks \textit{and} removes tokens by drawing an insertion time $t_i^\ell\sim \dot{\alpha}_t\,dt$ and unmasking time $t_u^\ell \sim \boldsymbol{1}[t\geq t^\ell_i]\frac{\dot{\beta}t}{1-\beta{t_i^\ell}}\,dt$, enforcing that a token is unmasked only after insertion ($t^\ell_u>t_i^\ell$). We track the position of each clean token via $s_t[\ell]$, which is the index in $\boldsymbol{x}_1$ of the $\ell$th token in $\boldsymbol{x}_t$, so that $\boldsymbol{x}_t^\ell=\boldsymbol{x}_1^{s_t[\ell]}$. We parameterize both an \textbf{unmasking posterior} $f_\theta(\boldsymbol{x}_t,t)[\ell]\in \Delta^{V-1}$ and an \textbf{insertion expectation} $g_\theta(\boldsymbol{x}_t,t)[\ell]\in \mathbb{R}_{\geq 0}$, the predicted number of tokens to insert at each gap, and minimize:
\begin{small}
\begin{align}
    \mathcal{L}_{\text{insert}}(\theta;\boldsymbol{x}_1)&=\underset{t\sim \mathcal{U}(0,1)}{\mathbb{E}}\underset{\boldsymbol{x}_t\sim p_t(\cdot |\boldsymbol{x}_1)}{\mathbb{E}} \left[-\frac{\dot{\alpha}_t}{1-\alpha_t}\sum_{\ell=0}^{\text{len}(\boldsymbol{x}_t)}\phi\left(s_t[\ell]-s_t[\ell-1]-1, g_\theta(\boldsymbol{x}_t,t)[\ell]\right)\right]\label{loss:insert}\\
    \mathcal{L}_{\text{unmask}}(\theta;\boldsymbol{x}_1)&=\underset{t\sim \mathcal{U}(0,1)}{\mathbb{E}}\underset{\boldsymbol{x}_t\sim p_t(\cdot |\boldsymbol{x}_1)}{\mathbb{E}} \bigg[-\frac{\dot{\beta}_t}{1-\beta_t}\sum_{\ell:\boldsymbol{x}^\ell_t=\boldsymbol{M}}\log f_\theta(\boldsymbol{x}_t,t)[\ell, \boldsymbol{x}_1^{s_t[\ell]}]\bigg]\label{loss:unmask}
\end{align}
\end{small}
where $\phi(x, y)=y-x\log y$ is the scalar Bregman divergence. A detailed background on any-length MDMs is provided in App \ref{app:any-length}.

\section{Defining the Unmasking and Insertion Quality}
\label{sec:quality}
In this section, we first introduce \textit{unmasking quality} and relate it to the error accumulated in a single unmasking step (Sec~\ref{subsec:per-token-quality}), then derive an analogous, theoretically grounded notion of \textit{insertion quality} (Sec~\ref{subsec:insert-quality}).

\subsection{Unmasking Quality}
\label{subsec:per-token-quality}
\paragraph{Defining Unmasking Quality}
We define the \textbf{unmasking quality} as the probability that the unmasked token is sampled from the unmasking posterior given the context from the rest of the sequence. Concretely, given a clean sequence from the target $\boldsymbol{x}_1\sim p_{\text{target}}$, a sample along the interpolant $\tilde{\boldsymbol{x}}_t$, a masked position $\tilde{\boldsymbol{x}}_t^\ell=\boldsymbol{M}$, a set of masked tokens $\mathcal{M}_t$ that are unmasked to obtain $\boldsymbol{y}$ where $\boldsymbol{y}^\ell\sim f_\theta(\tilde{\boldsymbol{x}}_t, t)[\ell]$, and the unmasking posterior at position $\ell$ given by $p(\boldsymbol{y}^\ell=\cdot|\tilde{\boldsymbol{x}}_t)=f_\theta(\tilde{\boldsymbol{x}}_t, t)[\ell]$, we define the unmasking quality as:
\begin{align}
    \mu_\star^\ell(\boldsymbol{y}):=p(\boldsymbol{y}^{\ell}=\boldsymbol{x}_1^{s_t[\ell]}|\boldsymbol{y})=f_\theta(\tilde{\boldsymbol{x}}_t, t) [\ell, \boldsymbol{x}^\ell_1]
\end{align}
which aligns with the definition of the \textit{per-token quality} introduced in \citep{kim2025fine}. When the probability of the true token $p(\boldsymbol{y}^{\ell}=\boldsymbol{x}_1^{s_t[\ell]}|\boldsymbol{y})$ for $\tilde{\boldsymbol{x}}_t^\ell=\boldsymbol{M}$ is high given the unmasked context $\tilde{\boldsymbol{x}}_t^{\text{UM}}$, then the quality is high, and if the probability is low, then the quality is low. To predict the quality of an already unmasked token $\boldsymbol{y}^\ell$ such that it estimates the unmasking posterior defined on its masked state $\boldsymbol{y}^{\mathcal{M}_t\gets\boldsymbol{M}}=\tilde{\boldsymbol{x}}_t$, we train a parameterized model $\mu_\phi:\mathcal{X}\to [0,1]$ to approximate $\mu_\star^\ell$ given its unmasked state $\boldsymbol{y}$ as:
\begin{align}
    \mu_\phi^\ell(\boldsymbol{y})\approx p(\boldsymbol{y}^{\ell}=\boldsymbol{x}_1^{s_t[\ell]}|\boldsymbol{y})
\end{align}
During inference, when we do not have access to the ground-truth sequence $\boldsymbol{x}_1$, the unmasking quality determines which tokens are inconsistent and should be re-masked.

\paragraph{Training the Unmasking Quality Predictor}
To train $\mu_\phi^\ell:\mathcal{X}\to [0,1]$, we start with a clean sequence from the target distribution $\boldsymbol{x}_1\sim p_{\text{target}}$ and sample an intermediate state from the joint interpolant $\tilde{\boldsymbol{x}}_t$ by partially masking and removing tokens in $\boldsymbol{x}_1$. Then, we take a single unmasking step to obtain $\boldsymbol{y}$ by replacing a subset $\mathcal{M}$ of tokens in $\tilde{\boldsymbol{x}}_t$ with the clean token sampled from the unmasking posterior $\boldsymbol{y}^\ell\sim f_\theta(\tilde{\boldsymbol{x}}_t,t)[\ell]$. Finally, we minimize an \textbf{Unmasking Quality Loss} (UQL) defined as:
\begin{small}
\begin{align}
    \mathcal{L}_{\text{UQL}}(\phi; \boldsymbol{x}_1):=\underset{t\sim \mathcal{U}(0,1)}{\mathbb{E}}\underset{\tilde{\boldsymbol{x}}_t\sim p_t(\cdot |\boldsymbol{x}_1), \boldsymbol{y}}{\mathbb{E}}\underset{\sim p_{s|t}(\cdot |\tilde{\boldsymbol{x}}_t)}{\mathbb{E}}\left[\sum_{\ell\in \mathcal{M}}\text{BCE}\left(\boldsymbol{1}[\boldsymbol{y}^\ell=\boldsymbol{x}_1^{s_t[\ell]}], \mu_\phi^\ell(\boldsymbol{y})\right)\right]\label{loss:uql}
\end{align}
\end{small}
where BCE denotes the \textbf{binary cross-entropy} loss defined as $\text{BCE}(b, c)=-b\log c-(1-b)\log (1-c)$ for $b\in \{0,1\}$ and $c\in [0,1]$. The \textit{unique minimizer} of $\mathcal{L}_{\text{UQL}}(\phi)$ is exactly the unmasking quality.

\begin{restatable}[Unique Minimizer of Unmasking Quality Loss]{proposition}{unmaskloss}\label{prop:unmaskloss}
    The unique minimizer of the unmasking quality loss $\mathcal{L}_{\text{UQL}}(\phi)$ is the true unmasking quality:
    \begin{align}
        \mu_\star=\mu_{\phi^\star}, \quad\text{where}\quad  \phi^\star=\underset{\phi}{\arg\min}\mathcal{L}_{\text{UQL}}(\phi)
    \end{align}
\end{restatable}
The proof is stated in App \ref{app-prop:unmaskloss}. We note that this coincides with the loss defined in \citet{kim2025fine} for parameterizing the unmasking quality of fixed-length MDMs.

\paragraph{Compounding Parallelization Error of Unmasking Steps}
For \textbf{unmasking}, the compounding parallelization error is the KL divergence between the true joint distribution over tokens unmasked in a step and the product of their marginals \citep{park2024jump}, defined as the gap between sequential single-token unmasking and the factorized approximation made when unmasking in parallel.

\begin{definition}[Unmasking Compounding Parallelization Error]
    Consider a unmasking step from state $\boldsymbol{x}_s$ to $\boldsymbol{x}_t$, where $\{\boldsymbol{x}_t^{\ell_k}\}_{k=1}^K$ denotes the set of $K$ tokens unmasked in parallel at time $t$. The unmasking CPE over the transition $s\to t$ is defined as:
    \begin{align}
        \mathcal{E}_{\text{CPE}}^{\text{unmsk}}(s\to t)&=\mathbb{E}_{\boldsymbol{x}_s\sim \mathbb{P}_s}\bigg[D_{\text{KL}}\bigg(p(\boldsymbol{x}_t^{\ell_1}, \dots, \boldsymbol{x}_t^{\ell_K}|\boldsymbol{x}_s)\bigg\|\prod_{k=1}^Kp(\boldsymbol{x}_t^{\ell_k}|\boldsymbol{x}_s)\bigg)\bigg]
    \end{align}
    where $\mathbb{P}_s$ is the marginal distribution of sequences at time $s$.
\end{definition}

\paragraph{Optimal Parallel Unmasking as Maximizing the Unmasking Quality}
We now show that unmasking quality provides an upper bound for the probability of successful parallel unmasking.

\begin{restatable}[Relationship Between Unmasking Quality and Parallel Unmasking]{proposition}{optimalunmasking}\label{prop:optimalunmasking}
    Assume that in a parallel unmasking step on indices $\mathcal{M}_t=\{\ell_k\}_{k=1}^K$, the unmasked tokens are conditionally independent given the unchanged context $\bar{\boldsymbol{x}}_s$ and the current state $\boldsymbol{x}_s$, i.e. $p(\boldsymbol{x}_t^{\mathcal{M}_t}| \bar{\boldsymbol{x}}_s,\boldsymbol{x}_s)=\prod_{k=1}^K p(\boldsymbol{x}_t^{\ell_k}| \bar{\boldsymbol{x}}_s,\boldsymbol{x}_s)$. If we define the unmasking quality $\mu_\star^{\ell_k}(\boldsymbol{x}_t):=p(\boldsymbol{x}_t^{\ell_k}=\boldsymbol{x}_1^{\ell_k}| \boldsymbol{x}_t^{\ell_k\leftarrow \boldsymbol{M}})$ and note that $p(\boldsymbol{x}_t^{\ell_k}| \bar{\boldsymbol{x}}_s,\boldsymbol{x}_s)=p(\boldsymbol{x}_t^{\ell_k}|\boldsymbol{x}_t^{\ell_k\leftarrow \boldsymbol{M}})$, then
    \begin{align}
        p(\boldsymbol{x}_t^{\mathcal{M}_t}| \bar{\boldsymbol{x}}_s,\boldsymbol{x}_s)=\prod_{k=1}^K \mu_\star^{\ell_k}(\boldsymbol{x}_t)
    \end{align}
\end{restatable}
The proof is deferred to App \ref{app-prop:optimalunmasking}. This establishes the theoretical grounding for our adaptive unmasking strategy, which remasks low-quality tokens in $\boldsymbol{x}_t$ to ensure only mutually high-quality tokens are unmasked in parallel, which effectively maximizes the probability of optimal parallel unmasking.

\subsection{Insertion Quality}
\label{subsec:insert-quality}

\paragraph{Defining Insertion Quality}
It is natural to define the \textbf{insertion quality} as the probability that the insertion is likely to be decoded into a true token in the corresponding gap of the target sequence. Concretely, given a clean sequence from the target distribution $\boldsymbol{x}_1\sim p_{\text{target}}$ and a sample along the interpolant $\tilde{\boldsymbol{x}}_t$, we define quality of a mask inserted between positions $\ell-1$ and $\ell$, denoted $\boldsymbol{y}:=\tilde{\boldsymbol{x}}_t^{\insertat{\ell}\boldsymbol{M}}$, as:
\begin{small}
\begin{align}
    \nu_\star(\boldsymbol{y}):=\underset{s_t[\ell-1]< i < s_t[\ell]}{\max}\left\{p(\boldsymbol{y}^{\ell}=\boldsymbol{x}_1^{i}|\boldsymbol{y})\right\}
\end{align}
\end{small}
which returns the highest probability of any ground truth token between positions $s_t[\ell-1]$ and $s_t[\ell]$ in the clean sequence $\boldsymbol{x}_1$ being predicted by the unmasking posterior $p(\boldsymbol{y}^i=\cdot |\boldsymbol{y})=f_\theta(\boldsymbol{y},t)[i]$ at the newly inserted mask token $\boldsymbol{y}^i= (\tilde{\boldsymbol{x}}_t^{\insertat{\ell}\boldsymbol{M}} )^i=\boldsymbol{M}$. Alternatively, we can define the set of tokens $\mathcal{S}_\ell:=\{\boldsymbol{x}^i_1:s_t[\ell-1]<i< s_t[\ell]\}$ in the gap between positions $\ell-1$ and $\ell$ and the insertion quality as the sum of the probabilities of the tokens in $\mathcal{S}_\ell$, given by:
\begin{align}
    \nu^\ell_\star(\boldsymbol{y}):=\sum_{\boldsymbol{v}\in \mathcal{S}_\ell}p(\boldsymbol{y}^{\ell}=\boldsymbol{v}|\boldsymbol{y})=p(\boldsymbol{y}^\ell \in \mathcal{S}_\ell|\boldsymbol{y})\label{eq:insertion-quality}
\end{align}
which is high when several potential tokens have a \textbf{high potential} to be unmasked to one of the tokens within the gap. Note that when there is \textbf{no token} between positions $s_t[\ell-1]$ and $s_t[\ell]$ in $\boldsymbol{x}_1$, the quality is zero. Since we have no access to the ground truth sequence at inference, we aim to train a parameterized model $\nu_\phi$ that takes only the sequence \textit{after the insertion step} $\boldsymbol{y}$ as input and approximates the quality of the insertion:
\begin{align}
    \nu^{i}_\phi(\boldsymbol{y})\approx \sum_{\boldsymbol{v}\in \mathcal{S}_\ell}p(\boldsymbol{y}^{\ell}=\boldsymbol{x}_1^{i}|\boldsymbol{y})\in [0,1]
\end{align}
This allows us to evaluate the quality of an insertion and remove low-quality insertions that are likely to result in mistakes in unmasking.

\paragraph{Training the Insertion Quality Predictor}
To train $\nu^\ell_\star(\boldsymbol{y})$, we start with a clean target sequence $\boldsymbol{x}_1\sim p_{\text{target}}$, we apply the standard joint interpolant to sample $\tilde{\boldsymbol{x}}_t$ by partially masking and removing tokens in $\boldsymbol{x}_1$. Then, we insert the set of masks $\mathcal{I}=\{\boldsymbol{y}^i\;|\;\boldsymbol{x}_t^{\insertat{\ell}\boldsymbol{M}}; g_\theta(\boldsymbol{x}_t,t)[\ell]\}$ predicted by the insertion expectation $g_\theta(\boldsymbol{x}_t,t)$ at each gap to get $\boldsymbol{y}$ and predict the unmasking posterior over the newly inserted masked positions $f_\theta(\boldsymbol{y}, t)[i]$. Finally, we minimize the \textbf{Insertion Quality Loss} (IQL), defined as:
\begin{align}
    \mathcal{L}_{\text{IQL}}(\phi; \boldsymbol{x}_1):=\underset{t\sim \mathcal{U}(0,1)}{\mathbb{E}}\underset{\tilde{\boldsymbol{x}}_t\sim p_t(\cdot |\boldsymbol{x}_1)}{\mathbb{E}}\underset{\boldsymbol{y}\sim p_{s|t}(\cdot|\tilde{\boldsymbol{x}}_t)}{\mathbb{E}}\left[\sum_{i\in \mathcal{I}}\text{BCE}\big(\nu_\star^i(\boldsymbol{y}), \nu_\phi^i(\boldsymbol{y})\big)\right]\label{loss:iql}
\end{align}
where BCE denotes the \textbf{Bernoulli cross-entropy} loss for two values in defined as $\text{BCE}(b, c)=-b\log c-(1-b) \log (1-c)$ for two values $b, c \in [0,1]$. If there are multiple masks $(\boldsymbol{y}^{i_1}, \dots, \boldsymbol{y}^{i_D})$ inserted in a gap $(\boldsymbol{x}_1^{s_t[\ell-1]+1}, \dots, \boldsymbol{x}_1^{s_t[\ell]- 1})$, then for each $i_d$, we sum over the probabilities over the range $s_t[\ell-1]+d$ to $s_t[\ell] -1-(D-d)$. We show that the minimizer of the insertion quality loss $\mathcal{L}_{\text{IQL}}(\phi)$ is the true insertion quality $\nu_\star$ defined in (\ref{eq:insertion-quality}) with proof in App \ref{app-prop:insertionloss}. 

\begin{restatable}[Minimizer of Insertion Quality Loss]{proposition}{insertionloss}\label{prop:insertionloss}
    The unique minimizer of the insertion quality loss $\mathcal{L}_{\text{IQL}}(\phi)$ is the true insertion quality:
    \begin{align}
        \nu_\star=\nu_{\phi^\star}, \quad \text{where}\quad \phi^\star=\underset{\phi}{\arg\min}\mathcal{L}_{\text{IQL}}(\phi)
    \end{align}
\end{restatable}

\paragraph{Optimizing Reconstruction Likelihood by Maximizing Insertion Quality}
To theoretically justify maximizing insertion quality, we prove that it provides an upper bound for the probability of reconstructing a clean sequence $\boldsymbol{x}_1\sim p_{\text{target}}$. 
\begin{restatable}[Insertion Quality as an Upper Bound on Reconstruction via Insertions]{proposition}{optimalinsertion}\label{prop:optimalinsertion}
    Consider a clean target sequence $\boldsymbol{x}_1\sim p_{\text{target}}$ and an intermediate sequence $\boldsymbol{x}_s\sim p_s(\cdot|\boldsymbol{x}_1)$. Given a set of indices of inserted masks at each gap $\mathcal{I}_t$ which yields $\boldsymbol{x}_t^{\mathcal{I}_t}$, the probability of reconstructing $\boldsymbol{x}_1$ is upper-bounded by the product of insertion qualities:
    \begin{align}
        p(\boldsymbol{x}^{\mathcal{I}_t}_t=\boldsymbol{x}^{\mathcal{I}_t}_1|\boldsymbol{x}_t)\leq \prod_{i\in \mathcal{I}_t}\nu_\star^i(\boldsymbol{x}_t)\approx \prod_{i\in \mathcal{I}_t}\nu_\phi^i(\boldsymbol{x}_t)
    \end{align}
    where the unmasking posterior for each inserted mask is conditionally independent given $\boldsymbol{x}_t$.
\end{restatable}
We defer the proof to App \ref{app-prop:optimalinsertion}. This justifies maximizing insertion quality as optimizing reconstruction quality.


\section{\textbf{A2D2}: Fine-Tuning Any-Length Discrete Diffusion for Adaptive Decoding}
\label{sec:a2d2}
In this section, we introduce Fine-Tuning \textbf{A}ny-Length \textbf{D}iscrete \textbf{D}iffusion for \textbf{A}daptive Decoding (\textbf{A2D2}), a novel framework tailored for any-length masked diffusion models that jointly optimizes the parameterized quality predictors and policy to sample from the reward-tilted data distribution.

\subsection{Adaptive Joint Decoding Loss}
\label{sec:ajd-loss}
\paragraph{Optimal Reward-Tilted Path Measure}
Since the path measure $\mathbb{P}^v$ of any-length MDMs consists of both insertion steps \textit{and} unmasking steps, which form the joint CTMC with rates $\boldsymbol{R}_t^v$ and $\boldsymbol{Q}_t^v$, respectively, we aim to define the tilted path measure $\mathbb{P}^\star$ that optimally samples from the reward-tilted distribution $p_{\text{target}}$. Concretely, we define the tilted path measure as:
\begin{small}
\begin{align}
    \mathbb{P}^\star(\boldsymbol{X}_{0:1}):=\frac{1}{Z}\mathbb{P}^{\text{pre}}(\boldsymbol{X}_{0:1})\exp\left(\frac{r(\boldsymbol{X}_1)}{\alpha}\right), \quad \mathbb{P}^\star_1(\boldsymbol{x})=\frac{1}{Z}p_{\text{data}}(\boldsymbol{x})\exp\left(\frac{r(\boldsymbol{x})}{\alpha}\right)=:p_{\text{target}}(\boldsymbol{x})
\end{align}
\end{small}
where $\boldsymbol{X}_{0:1}=(\boldsymbol{X}_t)_{t\in [0,1]}$ is a joint CTMC induced by a path measure and $\mathbb{P}^{\text{pre}}$ is the reference path measure induced by the pre-trained insertion rate $\boldsymbol{R}^{\text{pre}}_t$ and unmasking rate $\boldsymbol{Q}^{\text{pre}}_t$. Given parameterized insertion and unmasking rates $\boldsymbol{R}^{\theta}_t$ and $\boldsymbol{Q}^{\theta}_t$, we can optimize $\theta$ such that the path measure $\mathbb{P}^\theta$ matches $\mathbb{P}^\star$ by minimizing the following \textbf{entropy-regularized reward optimization} problem \citep{uehara2024understanding}:
\begin{align}
    \min_\theta\mathbb{E}_{\boldsymbol{X}_{0:1}\sim \mathbb{P}^\theta}\left[r(\boldsymbol{X}_1)\right]-\alpha D_{\text{KL}}(\mathbb{P}^\theta\|\mathbb{P}^{\text{pre}})
\end{align}
which is uniquely minimized when $\mathbb{P}^\theta=\mathbb{P}^\star$. For any-length joint CTMCs, the KL-divergence takes a unique form, as derived in App \ref{app-cor:jointkl}.

\paragraph{Weighted Cross-Entropy Objective}
To derive a loss that provably converges to $\mathbb{P}^\star$, let's consider directly minimizing the KL objective $D_{\text{KL}}(\mathbb{P}^{\theta, \phi}, \mathbb{P}^\star)$. While its minimizer indeed matches the optimal measure $\mathbb{P}^\star$, it is difficult to optimize due to the expectation over $\mathbb{P}^{\theta, \phi}$, which changes as the model is updated. Instead, we use the \textbf{cross-entropy loss} which takes the expectation over the \textbf{fixed} target path measure $\mathbb{P}^\star$:
\begin{align}
    \mathcal{F}_{\text{CE}}(\mathbb{P}^{\theta, \phi}, \mathbb{P}^\star):=D_{\text{KL}}(\mathbb{P}^\star\|\mathbb{P}^{\theta, \phi})=\mathbb{E}_{\mathbb{P}^\star}\left[\log \frac{\mathrm{d}\mathbb{P}^\star}{\mathrm{d}\mathbb{P}^{\theta, \phi}}\right]=\mathbb{E}_{\mathbb{P}^{v}}\bigg[\frac{\mathrm{d}\mathbb{P}^\star}{\mathrm{d}\mathbb{P}^{v}}\log \frac{\mathrm{d}\mathbb{P}^\star}{\mathrm{d}\mathbb{P}^{\theta, \phi }}\bigg]
\end{align}
where the final equality allows us to sample trajectories from any arbitrary path measure $\mathbb{P}^v$ and \textbf{reweight} them such that they approximate the target path measure $\mathbb{P}^\star$ at the infinite sampling limit. The reweighting term $\frac{\mathrm{d}\mathbb{P}^\star}{\mathrm{d}\mathbb{P}^v}$ is the Radon-Nikodym derivative (RND) between the two path measures. We derive the full tractable form of the RND in Prop \ref{prop:rnd-param}.

\begin{restatable}[Radon-Nikodym Derivative of Parameterized Rates]{proposition}{rndparam}\label{prop:rnd-param}
    Let the fine-tuned unmasking rate be $f^v(\boldsymbol{x}_t,t)[\ell]\in \Delta^{V-1} $ and the insertion rate be $g^v(\boldsymbol{x}_t, t)[\ell]\in \mathbb{R}_{\geq 0}$ that generates the path measure $\mathbb{P}^v$. Then, given optimal rates $f^{\text{pre}}(\boldsymbol{x}_t,t)$ and $g^{\text{pre}}(\boldsymbol{x}_t,t)$ and reward function $r:\mathcal{X}\to \mathbb{R}$, the log RND between the optimal joint CTMC and the fine-tuned CTMC over the trajectory $\boldsymbol{X}_{0:1}=(\boldsymbol{X}_t)_{t\in [0,1]}$ is defined as:
    \begin{small}
    \begin{align}
        &\log\frac{\mathrm{d}\mathbb{P}^\star}{\mathrm{d}\mathbb{P}^{v}}(\boldsymbol{X}_{0:1})=\frac{r(\boldsymbol{X}_1)}{\alpha}-\log Z+\sum_{t_u:\boldsymbol{X}_{t_u}\neq  \boldsymbol{X}_{t_u-}}\sum_{\ell: \boldsymbol{X}^\ell_{t_u}\neq \boldsymbol{X}^\ell_{t_u-}}\log \frac{f^{\text{pre}}(\boldsymbol{X}_{t_u}, t_u)[\ell, \boldsymbol{v}]}{f^v(\boldsymbol{X}_{t_u}, t_u)[\ell, \boldsymbol{v}]}\nonumber\\
        &+\sum_{t_i: \boldsymbol{X}_{t_i}\neq \boldsymbol{X}_{t_i-}}\sum_{\ell: \boldsymbol{X}^\ell_{t_i}\neq \boldsymbol{X}_{t_i-}^\ell}\log \frac{g^{\text{pre}}(\boldsymbol{X}_{t_i}, t_i)[\ell]}{g^v(\boldsymbol{X}_{t_i}, t_i)[\ell]}+\int_0^1\frac{\dot{\alpha}_t}{1-\alpha_t}\left(\sum_\ell (g^v- g^{\text{pre}})(\boldsymbol{X}_{t},t)[\ell]\right)dt\label{eq:log-rnd}
    \end{align}
    \end{small}
    where $t_i\in [0,1]$ denotes the times of insertion events and $t_u\in [0,1]$ denotes the times of unmasking events, with $t_i^-$ and $t_u^-$ being the left limits denoting the pre-jump state in the CTMC.
\end{restatable}
The proof is given in App \ref{app-prop:rndparam}. Then, defining $W^v(\boldsymbol{X}_1):=\log\frac{\mathrm{d}\mathbb{P}^\star}{\mathrm{d}\mathbb{P}^{\bar{\theta},\bar{\phi}}}(\boldsymbol{X}_{0:1})$, we can write the cross-entropy loss as $\mathcal{F}_{\text{CE}}(\mathbb{P}^{\theta, \phi}, \mathbb{P}^\star)=\mathbb{E}_{\mathbb{P}^{v}}[\frac{1}{Z}e^{W^v}\log \frac{\mathrm{d}\mathbb{P}^\star}{\mathrm{d}\mathbb{P}^{\theta, \phi}}]$.

\paragraph{Adaptive Joint Decoding Loss}
Now, we derive our \textbf{Adaptive Joint Decoding (AJD)} loss, which allows us to optimize the cross-entropy loss without computing the inner RND $\frac{\mathrm{d}\mathbb{P}^\star}{\mathrm{d}\mathbb{P}^{\theta, \phi}}$ over trajectories after each update to the model parameters. 
\begin{align}
    &\min_{\theta, \phi}\mathbb{E}_{\mathbb{P}^{v}}\left[\frac{1}{Z}e^{W^v}\log \frac{\mathrm{d}\mathbb{P}^\star}{\mathrm{d}\mathbb{P}^{\theta, \phi}}\right]=\min_{\theta, \phi}\mathbb{E}_{\mathbb{P}^{v}}\left[\frac{1}{Z}e^{W^v}[-\log \mathbb{P}^{\theta, \phi}]\right]\nonumber\\
    &=\min_{\theta, \phi}\mathbb{E}_{\mathbb{P}^{v}}\left[\frac{1}{Z}e^{W^v}\left[\sum_{t_i:\boldsymbol{X}_{t_i}\neq \boldsymbol{X}_{t_i-}}-\log \mathbb{P}^{\theta, \phi}(\boldsymbol{X}_{t_i}|\boldsymbol{X}_{t_i^-} )+\sum_{t_u:\boldsymbol{X}_{t_u}\neq \boldsymbol{X}_{t_u-} }-\log \mathbb{P}^{\theta, \phi}(\boldsymbol{X}_{t_u}|\boldsymbol{X}_{t_u^-})\right]\right]
\end{align}
where the first equality follows from dropping the $\log \mathbb{P}^\star$ term independent of $\theta, \phi$, and the second equality follows from decomposing the probability path into the sum of probabilities of each insertion and unmasking step at times $t_i$ and $t_u$, respectively. Rather than explicitly computing the sum over the full trajectory, we observe that each inner term can be reframed as an expectation over samples from the interpolant $\tilde{\boldsymbol{x}}_t\sim p_t(\cdot|\boldsymbol{X}_1)$ given a clean sample $\boldsymbol{X}_1$ from $\boldsymbol{X}_{0:1}\sim \mathbb{P}^v$. Given the clean sample $\boldsymbol{x}_1$, the optimal $f_\theta$, $g_\theta$, $\mu_\phi$, and $\nu_\phi$ can be obtained by minimizing $\mathcal{L}_{\text{unmask}}(\theta;\boldsymbol{x}_1)$ (\ref{loss:unmask}), $\mathcal{L}_{\text{insert}}(\theta;\boldsymbol{x}_1)$ (\ref{loss:insert}),  $\mathcal{L}_{\text{UQL}}(\phi;\boldsymbol{x}_1)$ (\ref{loss:uql}), and $\mathcal{L}_{\text{IQL}}(\phi;\boldsymbol{x}_1)$ (\ref{loss:iql}), respectively. This yields our \textbf{Adaptive Joint Decoding (AJD)} loss, defined as:
\begin{align}
    \mathcal{L}_{\text{AJD}}(\theta, \phi):=\underset{\boldsymbol{X}_{0:1}\sim \mathbb{P}^v}{\mathbb{E}}\left[\frac{1}{Z}\mathrm{e}^{W^v}\big[\mathcal{L}_{\text{unmask}}(\theta;\boldsymbol{X}_1)+\mathcal{L}_{\text{insert}}(\theta;\boldsymbol{x}_1)+\lambda_{\text{quality}}(\mathcal{L}_{\text{UQL}}(\phi;\boldsymbol{X}_1)+\mathcal{L}_{\text{IQL}}(\phi;\boldsymbol{X}_1))\big]\right]\label{loss:ajd}
\end{align}
where the normalizing constant $Z$ is approximated as $Z\approx \mathbb{E}_{\boldsymbol{X}_{0:1}\sim \mathbb{P}^v}\left[\mathrm{e}^{W^v}\right]$. This loss can be seen as an extension of the weighted denoising cross-entropy loss used for fixed-length discrete diffusion sampling \citep{zhu2025mdns} and fine-tuning \citep{tang2025tr2}. This unique minimizer of the AJD loss is the optimal unmasking generator $\boldsymbol{Q}^\star$ and insertion generator $\boldsymbol{R}^\star$ of the reward-tilted path measure $\mathbb{P}^\star$, with proof deferred to App \ref{app:ajd-proofs}.

\subsection{Fine-Tuning Any-Order Discrete Diffusion}

\paragraph{Off-Policy Reinforcement Learning}
To optimize the AJD loss, we leverage an \textbf{off-policy RL} strategy with the fixed path measure $\mathbb{P}^v:=\mathbb{P}^{\bar{\theta}, \bar{\phi}}$ defined by detached parameters $\bar{\theta}:=\texttt{stopgrad}(\theta)$ and $\bar{\phi}:=\texttt{stopgrad}(\phi)$. Each fine-tuning iteration: \textbf{(1)} we sample a batch of $B$ sequences $\boldsymbol{x}_1$ while computing their log RND $W^{\bar{\theta}, \bar{\phi}}$ and store them in a replay buffer $\mathcal{B}=\{(\boldsymbol{x}_{1, i},W^{\bar{\theta}, \bar{\phi}})\}_{i=1}^B$, \textbf{(2)} for each $\boldsymbol{x}_{1, i}$, we sample $R$ intermediate sequences from the interpolant $\tilde{\boldsymbol{x}}_{t, j}\sim p_t(\cdot|\boldsymbol{x}_{1, i})$, and \textbf{(3)} we compute the AJD loss (\ref{loss:ajd}) and update the parameters $\theta, \phi$.

\paragraph{Alternating Optimization of Policy and Quality Predictors}
Since the number of parameters $\theta$ in the insertion and unmasking policy is orders of magnitude larger than the number of parameters $\phi$ in the quality prediction heads, optimizing both models simultaneously results in unstable training and sub-optimal convergence of the quality prediction model. Following prior work, we alternate between updating $\theta$ for $N_{\text{alt}}$ iterations with $\phi$ frozen and updating $\phi$ for $N_{\text{alt}}$ iterations with $\theta$ frozen \citep{zhou2025co}. Since the quality heads are initialized from scratch, we sample the replay buffer using only the policy model $\theta$ for $N_{\text{warmup}}$ iterations. Beyond stabilizing training, this alternating schedule acts as an implicit form of loss balancing: because $\theta$ and $\phi$ are never updated in the same pass, the relative weighting $\lambda_{\text{quality}}$ between the policy losses ($\mathcal{L}_{\text{unmask}}, \mathcal{L}_{\text{insert}}$) and the quality losses ($\mathcal{L}_{\text{UQL}}, \mathcal{L}_{\text{IQL}}$) no longer needs to be tuned. Full pseudo-code for fine-tuning with A2D2 is given in Algorithm \ref{alg:a2d2}.

\paragraph{Adaptive Inference with A2D2}
Adaptive inference with the unmasking and insertion policy $f_\theta, g_\theta$ and quality predictors $\mu_\phi, \nu_\phi$ starts from an empty sequence of pad tokens $\boldsymbol{x}_0$ and at each discrete time step $[t_k, t_{k+1}]$, performs:
\begin{enumerate}
    \item[(i)] \textbf{Adaptive Unmasking}: We sample a subset of mask tokens $\mathcal{M}$ to unmask via the parameterized unmasking posterior $f_\theta(\boldsymbol{x}_t,t)[\ell]$. Then, we predict the unmasking quality for each token $\mu_\phi(\boldsymbol{x}_t,t)[\ell]$ and re-mask low-quality tokens that either fall below a threshold or until the expected number of masked tokens at time $t$ is reached.
    \item[(ii)] \textbf{Adaptive Insertion}: We insert a set of masks $\mathcal{I}$ according by sampling insertion counts $I^\ell_t\sim \text{Poisson}(g_\theta(\boldsymbol{x}_t,t)[\ell]\cdot\Delta t)$ for each gap. Then, we predict the insertion quality $\nu_\phi(\boldsymbol{x}_t,t)[i]$ for each inserted mask and subsequently remove low-quality masks. 
\end{enumerate}
We stop when no masks remain, and the insertion expectation falls below $0.5$ or when the total number of time steps is reached. The full inference procedure is detailed in Algorithm \ref{alg:quality-sample}.

\section{Experiments}
We evaluate \textbf{A2D2} on two reward-guided sequence generation tasks where any-order, any-length discrete diffusion is particularly well-suited: drug-like small-molecule design (Sec~\ref{exp:molecule}), multi-objective therapeutic peptide generation (Sec~\ref{exp:peptide}), and language reasoning for math word problems and code infilling (Sec~\ref{exp:language}).

\begin{table*}[t]
\caption{\textbf{Drug-like small molecule design results.} Metrics were computed for $1000$ generated sequences, and the mean and standard deviations were computed across $3$ seeds. }
\label{table:molecule-results}
\begin{center}
\begin{small}
\resizebox{\linewidth}{!}{
\begin{tabular}{@{}lccccccc@{}}
\toprule
\textbf{Method} &  Validity ($\%$ total) & Uniqueness ($\%$ valid) & Quality ($\%$ total) & QED ($\uparrow$) & Synthetic Accessibility ($\downarrow$) &  Diversity ($\uparrow$)  \\
\midrule
GenMol w/o confidence & $98.467_{\pm 0.094}$ & $99.763_{\pm 0.048}$ & $73.667_{\pm 1.159}$ & $0.763_{\pm 0.002}$ & $3.024_{\pm 0.017}$ & $0.859_{\pm 0.000}$ \\
GenMol w/ confidence & $99.733_{\pm 0.047}$ & $99.532_{\pm 0.125}$ & $84.000_{\pm 0.566}$ & $0.813_{\pm 0.004}$ & $2.881_{\pm 0.020}$ & $0.817_{\pm 0.000}$ \\
\midrule
\rowcolor{mybg} Pre-trained (Any Length) &  $95.800_{\pm 0.294}$ & $99.582_{\pm 0.086}$ & $44.167_{\pm 0.574}$ & $0.641_{\pm 0.002}$ & $3.401_{\pm 0.032}$  & $0.908_{\pm 0.000}$  \\
\rowcolor{mybg} A2D2 w/o quality & $96.400_{\pm  0.535}$ & $62.484_{\pm 0.452}$ &$76.067_{\pm 0.896}$ & $0.795_{\pm 0.005}$ & $2.306_{\pm 0.024}$ & $0.730_{\pm 0.003}$ \\
\rowcolor{mybg} A2D2 w/o insertion quality & $92.200_{\pm 0.082}$ & $93.746_{\pm 0.335}$ & $71.067_{\pm 0.411}$ & $0.752_{\pm 0.003}$ & $2.654_{\pm 0.014}$ & $0.832_{\pm 0.001}$ \\
\rowcolor{mybg} A2D2 w/o unmasking quality & $98.600_{\pm 0.082}$ & $85.801_{\pm 0.387}$ & $82.367_{\pm 0.736}$ & $0.818_{\pm 0.003}$ & $2.753_{\pm 0.026}$ & $0.789_{\pm  0.002}$ \\
\rowcolor{mybg} \textbf{A2D2 w/ both quality} & $94.533_{\pm 0.665}$ & $93.695_{\pm 0.925}$ & $71.300_{\pm 0.852}$ & $0.762_{\pm 0.004}$ & $2.870_{\pm 0.018}$ & $0.843_{\pm  0.001}$ \\
\bottomrule
\end{tabular}
}
\end{small}
\end{center}
\end{table*}

\begin{table*}[t]
\caption{\textbf{Multi-objective peptide design results.} Metrics were computed for $1000$ generated peptide sequences, except for multi-objective guidance, which was evaluated on 100 sequences in the Pareto-optimal set, and the mean and standard deviations were computed across $3$ seeds. \textbf{Pre-trained}: unconditional sampling from the fixed-length SMILES MDM from \citet{tang2025peptune} \textbf{(Fixed)} and our any-length model \textbf{(Any)}. \textbf{Multi-Objective Guidance (Fixed)}: 100 iterations of inference-time multi-objective guidance on the fixed-length model with PepTune \citep{tang2025peptune}. \textbf{Off-Policy RL (Fixed)}: fixed-length model fine-tuned for 500 iterations using off-policy RL fine-tuning \citep{tang2025tr2}. \textbf{A2D2 w/o quality}: any-length model fine-tuned with the AJD loss but without quality predictor training or sampling. \textbf{A2D2}: any-length model fine-tuned with the AJD loss and quality predictor optimization. * Only valid sequences are added to the optimal set for multi-objective guidance.}
\label{table:peptide-multi}
\begin{center}
\begin{small}
\resizebox{\linewidth}{!}{
\begin{tabular}{@{}llcccccc@{}}
\toprule
\textbf{Target Protein} & \textbf{Method} & Validity ($\%$) & Binding Affinity ($\uparrow$) & Solubility ($\uparrow$) & Non-hemolysis ($\uparrow$) & Non-fouling ($\uparrow$) & Permeability ($\uparrow$) \\
\midrule
TfR & Pre-trained (Fixed) & $32.446_{\pm 0.662}$ & $8.776_{\pm 0.023}$ & $0.707_{\pm  0.001}$ & $0.901_{\pm 0.003}$ & $0.215_{\pm 0.001}$ & $-7.145_{\pm 0.003}$ \\
& Pre-trained (Any) & $10.064_{\pm 0.360}$ & $7.756_{\pm 0.015}$ & $0.689_{\pm 0.010}$ & $0.851_{\pm 0.001}$ & $0.169_{\pm 0.004}$ & $-7.185_{\pm 0.009}$ \\
& Multi-Objective Guidance (Fixed) & N/A* & $9.043_{\pm 0.168}$ & $0.717_{\pm 0.004}$ & $0.897_{\pm 0.009}$ & $0.206_{\pm 0.001}$ & $-7.126_{\pm 0.004}$ \\
& Off-Policy RL (Fixed) & $33.433_{\pm 1.190}$ & $10.614_{\pm 0.011}$ & $0.655_{\pm 0.004}$ & $0.810_{\pm 0.002}$ & $0.224_{\pm 0.002}$ & $-7.109_{\pm  0.005}$ \\
\rowcolor{mybg} & A2D2 w/o quality & $41.267_{\pm 1.520}$ & $10.241_{\pm 0.012}$ & $0.792_{\pm 0.011}$ & $0.831_{\pm  0.004}$ & $0.096_{\pm 0.001}$ & $-6.102_{\pm 0.006}$ \\
\rowcolor{mybg} & \textbf{A2D2 w/ both quality} & $48.600_{\pm 1.445}$ & $10.190_{\pm 0.033}$ & $0.776_{\pm 0.003}$ & $0.831_{\pm 0.003}$ & $0.118_{\pm 0.002}$ & $-6.338_{\pm 0.003}$ \\
\midrule
GLP-1R & Pre-trained (Fixed) & $32.446_{\pm 0.662}$ & $8.781_{\pm 0.018}$ & $0.707_{\pm 0.001}$ & $0.901_{\pm 0.003}$ & $0.215_{\pm 0.001}$ & $-7.145_{\pm 0.003}$ \\
& Pre-trained (Any) & $10.064_{\pm 0.360}$ & $8.008_{\pm 0.019}$ & $0.689_{\pm 0.010}$ & $0.851_{\pm 0.001}$ & $0.169_{\pm 0.004}$ & $-7.185_{\pm 0.009}$ \\
& Multi-Objective Guidance (Fixed) & N/A* & $8.897_{\pm 0.041}$ & $0.721_{\pm 0.010}$ & $0.896_{\pm 0.015}$ & $0.214_{\pm 0.014}$ & $-7.111_{\pm 0.025}$ \\
& Off-Policy RL (Fixed) & $48.767_{\pm 0.492}$ & $9.377_{\pm 0.025}$ & $0.684_{\pm 0.003}$ & $0.804_{\pm 0.010}$ & $0.247_{\pm 0.002}$ & $-7.019_{\pm 0.009}$ \\
\rowcolor{mybg} & A2D2 w/o quality & $25.033_{\pm 2.014}$ & $9.698_{\pm 0.003}$ & $0.830_{\pm 0.003}$ & $0.849_{\pm 0.004}$ & $0.618_{\pm 0.006}$ & $-7.272_{\pm 0.003}$ \\
\rowcolor{mybg} & \textbf{A2D2 w/ both quality} & $29.167_{\pm 1.634}$ & $9.520_{\pm 0.017}$ & $0.854_{\pm 0.002}$ & $0.840_{\pm 0.002}$ & $0.601_{\pm 0.007}$ & $-7.209_{\pm 0.008}$ \\
\bottomrule
\end{tabular}
}
\end{small}
\end{center}
\end{table*}

\subsection{Drug-Like Small Molecule Design}
\label{exp:molecule}
\paragraph{Setup and Baselines}
We pre-train an any-length MDM on the SAFE dataset \citep{noutahi2024gotta}, containing $\sim950$M molecules from ZINC \citep{irwin2012zinc} and Unichem \citep{chambers2013unichem} in SAFE notation ($V=1880$). We fine-tune with \textbf{A2D2} using QED and SA as rewards, evaluating with the following metrics: \textbf{quality} (fraction of valid, unique, drug-like, and synthesizable molecules), \textbf{validity}, \textbf{uniqueness}, and \textbf{diversity} (mean Tanimoto distance between Morgan fingerprints). Following \citet{jin2020multi}, molecules are drug-like if QED \citep{bickerton2012quantifying} $> 0.6$ and synthesizable if SA \citep{ertl2009estimation} $\leq 4$, following the evaluation in prior work \citep{noutahi2024gotta, lee2025genmol}. We compare against the pre-trained MDM and ablate A2D2's insertion and unmasking quality components (App~\ref{app:molecule-exp}). We additionally benchmark against \textbf{GenMol} \citep{lee2025genmol}, the state-of-the-art \textit{fixed-length} masked diffusion model for SAFE generation. Confidence-based sampling with GenMol uses a confidence score computed as the sum of the log-probability the model assigned to its own prediction with scaled Gumbel noise. Additional experiment details and ablations are provided in App \ref{app:molecule-exp}.

\paragraph{Results}  
Fine-tuning any-length MDMs with our AJD-weighted loss yields substantial improvements in molecular rewards over the pre-trained any-length baseline, raising QED from $0.641$ to $0.762$ and lowering SA from $3.40$ to $2.87$, while improving quality from $44.2\%$ to $71.3\%$ (Table~\ref{table:molecule-results}). Notably, A2D2 without unmasking and insertion quality yields lower uniqueness, indicating that quality-guided inference is needed to maintain sample diversity.  These gains close the gap to GenMol \citep{lee2025genmol} in validity, diversity, and quality, while operating in the strictly more general any-length setting, where both the generable space of molecules and the per-step action space are substantially larger. This performance holds across all A2D2 variants, confirming that our learned quality predictors recover the benefit of GenMol's confidence-based sampling.

\begin{figure*}[t]
    \centering
    \includegraphics[width=\linewidth]{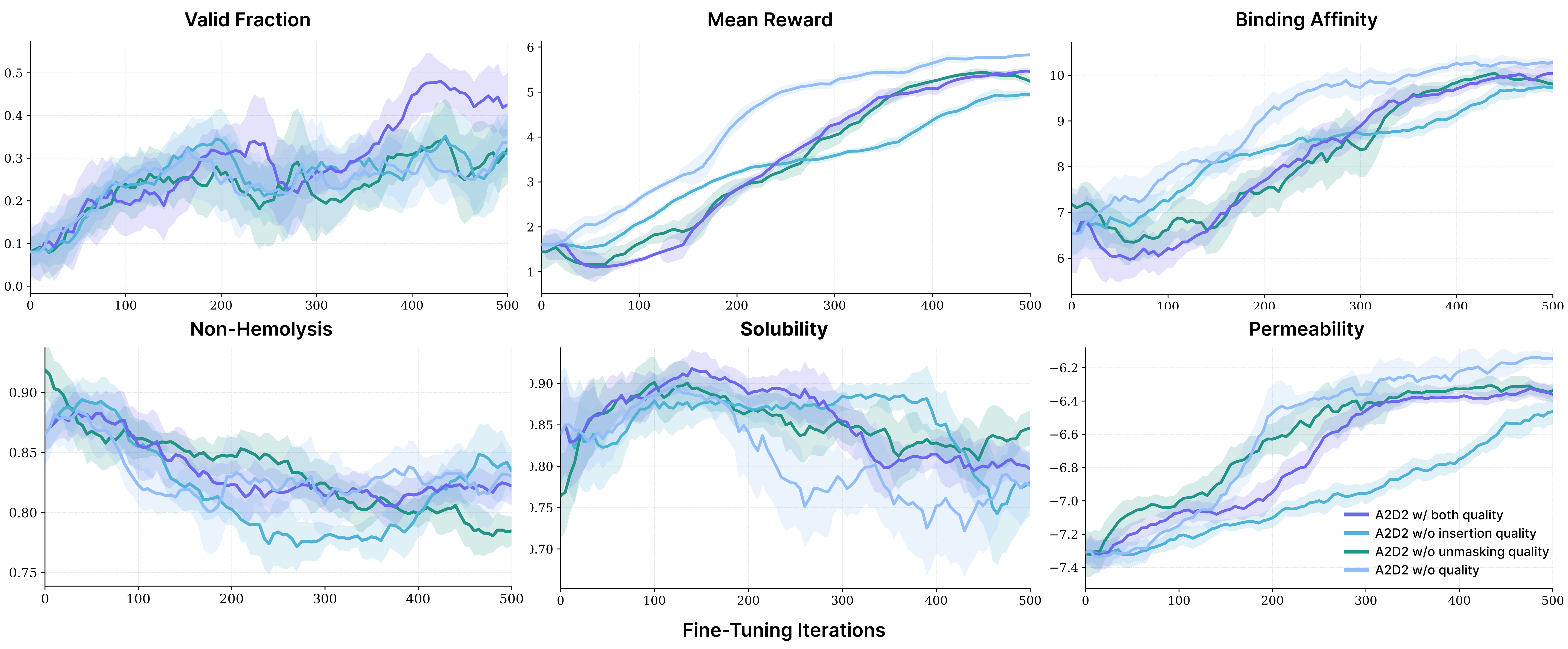}
    \caption{\textbf{Multi-objective peptide fine-tuning curves with and without quality-based adaptive inference.} Any-length peptide SMILES models are fine-tuned by optimizing the five rewards: binding affinity to TfR, solubility, non-hemolysis, non-fouling, and permeability. Curves show the fraction of valid peptides, the mean reward, and per-property rewards over fine-tuning iterations for the full model with both insertion and unmasking quality (\textbf{indigo}) against ablations that remove the insertion quality (\textbf{light blue}), the unmasking quality (\textbf{teal}), or both qualities (\textbf{pale blue}).}
    \label{fig:peptide-curves}
    \vspace{-10pt}
\end{figure*}

\subsection{Multi-Objective Therapeutic Peptide Generation}
\label{exp:peptide}
\paragraph{Setup and Baselines}
We pre-train an any-length MDM on $11$M peptide SMILES, comprising $7,451$ sequences from CycPeptMPDB \citep{li2023cycpeptmpdb}, $825,632$ unique peptides from SmProt \citep{li2021smprot}, and $\sim\!10$M modified peptides from CycloPs \citep{duffy2011cyclops, feller2025peptide}. We fine-tune with A2D2 on five therapeutic properties: binding affinity to a target protein, solubility, non-hemolysis, non-fouling, and permeability, and evaluate validity using the \texttt{SMILES2PEPTIDE} decoder from \citet{tang2025peptune}. As no prior work addresses inference-time guidance or fine-tuning for any-length discrete diffusion, we compare against fixed-length masked diffusion baselines: \textbf{Multi-Objective Guidance (Fixed)}, which uses Monte Carlo Tree Guidance (MCTG) \citep{tang2025peptune} for inference-time multi-objective guidance, and \textbf{Off-Policy RL (Fixed)}, which uses the off-policy RL algorithm from \citet{tang2025tr2} for fine-tuning fixed-length MDMs. We also compare against a pre-trained fixed-length peptide SMILES MDM \citep{tang2025peptune}, our pre-trained any-length MDM, and an A2D2 ablation without insertion and unmasking quality (\textbf{A2D2 w/o quality}). Additional experiment details and ablations are provided in App \ref{app:peptide-exp}.

\paragraph{Results}
Fine-tuning any-length MDMs with our AJD weighted loss yields substantial reward gains across all objectives, with or without quality-based adaptive inference (Fig.~\ref{fig:peptide-curves}). At equal sampling steps, \textbf{A2D2} outperforms the any-length pre-trained baseline by a wide margin on every therapeutic property, and surpasses both the fixed-length RL fine-tuning and the substantially more expensive multi-objective guidance baselines on \textit{almost all objectives} (Table~\ref{table:peptide-multi}). For TfR, it increases binding affinity from $7.756$ to $10.190$ while improving solubility and permeability; for GLP-1R, it raises non-fouling and solubility above the guidance and RL baselines, all without inference-time search. Crucially, quality-based adaptive inference markedly improves sequence validity while optimizing the \textit{same} rewards: A2D2 reaches $48.60\%$ validity on TfR, exceeding every baseline and improving over the any-length pre-trained model nearly fivefold. This confirms that explicitly maximizing quality translates into more accurate, decodable generation at no cost to reward optimization.

\subsection{Language Reasoning Experiment}
\label{exp:language}
\paragraph{Setup and Baselines}
Due to the lack of publicly available pre-trained any-length MDM checkpoints, we adapt the 8 billion parameter LLaDA-8B-Base \citep{nie2026large} fixed-length MDM to an any-length MDM following the same procedure as \citet{kim2025any} by adding time embedding layers and an insertion prediction head and training the frozen backbone parameters with attached LoRA adapters on an equal split of OpenWebText \citep{Gokaslan2019OpenWeb} and Proof-Pile-2 \citep{azerbayev2023llemma}. Pre-training ran for less than 3 days ($240,000$ steps, $1.10$ epochs) on 8 NVIDIA B200 GPUs. Then, we performed instruction fine-tuning (IFT) for math word problems with GSM8K \citep{cobbe2021training} and code infilling with opc-sft-stage-2 \citep{huang2025opencoder}. The IFT models are evaluated with GSM8K-test and HumanEval-infill \citep{bavarian2022efficient} as the \textbf{Pretrained + IFT} baseline. Then, we additionally perform RL fine-tuning with A2D2 with and without the insertion and unmasking quality heads. Additional experiment details are provided in App \ref{app:language}.

\begin{wraptable}{r}{0.4\textwidth}
\caption{\textbf{Reasoning experiment results.} GSM8k performance is evaluated with Pass@1 (\%) out of $1319$ total questions. Code infilling performance is evaluated with exact match to ground truth code (\%) out of $1033$ total examples. }
\label{table:language-results}
\begin{center}
\begin{small}
\resizebox{\linewidth}{!}{
\begin{tabular}{@{}lcccc@{}}
\toprule
& \multicolumn{4}{c}{\textbf{GSM8K}}  \\
\midrule

\textbf{Sampling Steps} & $\mathbf{128}$ & $\mathbf{256}$ & $\mathbf{512}$ & $\mathbf{1024}$ \\
\midrule
Pre-trained + IFT & $35.71$ & $37.98$ & $39.88$ & $41.39$ \\
\rowcolor{mybg} \textbf{A2D2} & $60.96$ & $61.03$ & $57.54$ & $57.62$  \\
\midrule
& \multicolumn{4}{c}{\textbf{HumanEval In-Filling}} \\
\midrule
\textbf{Sampling Steps} & $\mathbf{128}$ & $\mathbf{256}$ & $\mathbf{512}$ & $\mathbf{1024}$ \\
\midrule 
Pre-trained + IFT & $44.14$ & $44.82$ & $47.53$ & $48.89$ \\
\rowcolor{mybg} \textbf{A2D2} & $49.37$ & $53.44$ & $53.73$ & $57.12$ \\
\bottomrule
\end{tabular}
}
\end{small}
\end{center}
\end{wraptable}

\paragraph{Results}
Across both reasoning tasks, RL fine-tuning with \textbf{A2D2} yields large gains over the \textbf{Pre-trained + IFT} baseline at every matched inference budget (Table~\ref{table:language-results}). On GSM8K, A2D2 raises Pass@1 from $35.71\%$ to $60.96\%$ at $128$ sampling steps, a $25.3$-point improvement, and the advantage persists at $256$, $512$, and $1024$ steps. Notably, A2D2 reaches its best accuracy with only $128$–$256$ steps and exceeds the baseline's strongest result ($41.39\%$ at $1024$ steps) by roughly $20$ points while using a fraction of the steps, demonstrating that reward alignment with A2D2 translates directly into greater inference efficiency. The slight decline after $256$ steps, which then plateaus at $512$ and $1024$ steps, further indicates that, after fine-tuning, additional decoding steps are unnecessary and that A2D2 concentrates accuracy in the low-sampling-step regime. On HumanEval infilling, A2D2 again dominates the baseline at all budgets, improving exact match from $44.14\%$ to $49.37\%$ at $128$ steps and from $48.89\%$ to $57.12\%$ at $1024$ steps. Unlike GSM8K, infilling performance continues to improve monotonically with more sampling steps. Together, these results show that A2D2 simultaneously improves task accuracy and decoding efficiency for any-length language reasoning, mirroring the gains observed in the molecule and peptide settings.

\section{Conclusion}
In this work, we introduce Fine-Tuning \textbf{A}ny-Length \textbf{D}iscrete \textbf{D}iffusion for \textbf{A}daptive Decoding (\textbf{A2D2}), a novel framework that unlocks reward-guided fine-tuning for any-length discrete diffusion models with adaptive quality-based inference. Driven by the importance of the inference schedule in any-length decoding, we define unmasking and insertion quality, establish their theoretical connection to compounding parallelization error and reconstruction likelihood, and co-optimize the policy and quality predictors to provably generate the optimal any-length path measure that samples from an intractable reward-tilted distribution. Our approach is a significant step toward any-length generation for diverse reward-based tasks.

\section*{Declarations}
\paragraph{Acknowledgements} We thank Mark III Systems for providing database and hardware support that has contributed to the research reported within this manuscript. 

\paragraph{Author Contributions} S.T. devised and developed model architectures and theoretical formulations and performed experiments. S.T. drafted the manuscript and designed the figures. Y.Z. advised on model architectures and theoretical formulations. P.C. and M.T. supervised and directed the study, and reviewed and finalized the manuscript.

\paragraph{Data and Materials Availability} The codebase is freely accessible to the academic community at \url{https://github.com/sophtang/A2D2} and at \url{https://huggingface.co/ChatterjeeLab/A2D2}.

\paragraph{Funding Statement} This research was supported by NIH grant R35GM155282 to the lab of P.C.. Yuchen Zhu and Molei Tao are grateful for partial supports by NSF Grant DMS-2513699 (YZ \& MT), DOE Grants NA0004261 (MT), SC0026274 (YZ \& MT),  Richard Duke Fellowship (YZ \& MT), and Simons Institute for the Theory of Computing at UC Berkeley (MT).

\bibliographystyle{acl_natbib.bst}
\bibliography{citation.bib}

\clearpage
\beginappendix
\startcontents[app]
\printcontents[app]{l}{1}{\setcounter{tocdepth}{2}}

\newpage

\section{Related Work}
\label{app:related-work}

\paragraph{Inference-Time Scaling of Discrete Diffusion}
Inference-time scaling of diffusion models seeks to maximize the capabilities of pre-trained diffusion models for specialized tasks during inference, such as sampling from a reward-tilted distribution \citep{skreta2025feynman, singhal2025general, chen2025solving}, constrained sampling \citep{fishman2023diffusion}, or entropy maximization \citep{de2025provable}. Specifically in the discrete state space, search-based methods \citep{tang2025peptune, tang2025tr2, li2025dynamic, ramesh2025test}, importance-weighting techniques \citep{chatterjee2018sample}, reward-gradient methods \citep{song2020score, bansal2023universal}, and classifier-based and classifier-free guidance \citep{nisonoff2024unlocking, rector-brooks2025steering, wang2024diffusion, schiff2024simple, rojas2025theory, guo2024plug} have been explored. 

\paragraph{Training-Time Fine-Tuning of Discrete Diffusion}
Training-time fine-tuning methods instead update the model parameters to optimize a downstream reward. One line of work backpropagates rewards through the denoising trajectory, as in DRAKES \citep{wang2025fine}, which makes the otherwise non-differentiable discrete trajectories differentiable via the Gumbel-Softmax trick to optimize sequences that are both high-reward and natural-like. A complementary line adapts policy-gradient reinforcement learning to discrete diffusion: SEPO \citep{zekri2025fine} provides a framework that handles non-differentiable rewards via self-normalized importance sampling, while a growing body of work adapts GRPO to masked diffusion language models, including diffu-GRPO \citep{zhao2026d1} and Coupled GRPO \citep{gong2025diffucoder}. TR2-D2 \citep{tang2025tr2} instead casts fine-tuning as stochastic optimal control, using Monte Carlo tree search to curate replay buffers of high-reward trajectories for off-policy RL, and extends to multi-objective reward optimization.

\paragraph{Optimization of Inference Schedule for Discrete Diffusion}
Closest to our work is PRISM \citep{kim2025fine}, which introduces \textit{per-token quality} as a learned measure that determines whether to re-mask already unmasked tokens during inference. However, it remains limited to fixed-length masked diffusion for optimizing sampling of the data distribution. Co-GRPO \citep{zhou2025co}, which introduces a reward fine-tuning strategy for fixed-length masked diffusion that simultaneously optimizes the inference schedule and policy model with a shared reward signal via group-relative policy optimization (GRPO). Jump Your Steps (JYS) \citep{park2024jump} optimizes the times at which discrete jumps are made during inference by minimizing a KL upper bound on the compounding decoding error for fixed-length masked diffusion models. Related inference-time mechanisms include \textit{planner}-guided decoding, which learns a planner to choose which tokens to reveal at each step \citep{liu2024think, peng2025path, peng2025planner}, as well as remasking and corrector-style methods that iteratively revise intermediate predictions to improve sample quality \citep{wang2025remasking, zhao2024informed, tian2025finish}.

\section{Extended Theoretical Background}
\label{app:ext-background}

\subsection{Notation} 
We denote the CTMC of a diffusion trajectory as $\boldsymbol{X}_{0:1}:=(\boldsymbol{X}_t)_{t\in [0,1]}$ which lies in the variable-length discrete state space $\mathcal{X}=\bigcup_l\mathcal{V}^l$, where $\mathcal{V}:=\{1, \dots, V\}^l$ of sequences of length $l$ and vocabulary size $V$, where $l$ is the variable sequence length and $L$ is the maximum sequence length. An any-length discrete diffusion model is characterized by a path measure $\mathbb{P}\in \mathcal{P}(C([0,T]; \mathcal{X}))$ over trajectories that follow a joint generator $\boldsymbol{A}_t=\boldsymbol{Q}_t+\boldsymbol{R}_t$, where $\boldsymbol{Q}_t$ is the unmasking rate and $\boldsymbol{R}_t$ is the insertion rate. Specifically, we denote the pre-trained discrete diffusion model as $\mathbb{P}^{\text{pre}}$ and the optimal target model as $\mathbb{P}^\star$. For a corrupted sequence $\boldsymbol{x}_t$ at time $t$, the unmasking posterior is given by $f_\theta(\boldsymbol{x},t)$ with indices $[\ell, \boldsymbol{v}]$ denoting the probability of a single token at position $\ell$ and the insertion rate is given by $g_\theta(\boldsymbol{x}_t, t)[\ell]$ with index $[\ell]$ indicating the insertion expectation between positions $\ell-1$ and $\ell$ of the sequence. 

\subsection{Any-Length Masked Diffusion}
\label{app:any-length}
In this work, we build upon the any-length masked diffusion model framework introduced in \citet{kim2025any}, which leverages a pair of joint stochastic interpolants for the insertion and unmasking schedules. Here, we will provide a comprehensive background on the any-length MDM framework and highlight its distinction from standard fixed-length MDMs. 

\paragraph{Forward Noising Process}
In standard MDM, the forward process involves applying mask tokens gradually across the sequence until the sequence is fully masked at time $t=0$. For any-length MDMs, the forward process consists not only of the masking step but also of a \textbf{deletion} step, which serves as the inverse operation for the insertion step during inference. Therefore, we must define both the insertion time $t_i^\ell$ and the unmasking time $t_u^\ell$ for each token position $\boldsymbol{x}^\ell$ as:
\begin{align}
    t_i^\ell\sim \dot{\alpha}_tdt, \quad t_u^\ell\sim \boldsymbol{1}[t\geq t^\ell_i]\frac{\dot{\beta}}{1-\beta_{t^\ell_i}}dt, \quad \boldsymbol{x}^\ell_t=\begin{cases}
        (\text{empty}), & 0< t<t^\ell_i\\
        \boldsymbol{M}, & t^\ell_i\leq t\leq t^\ell_u\\
        \boldsymbol{x}_1^{s_t[\ell]}, & t^\ell_u\leq t\leq 1
    \end{cases}
\end{align}
where (empty) indicates a position that has not been inserted yet. We denote the index of a token in the partial subsequence with $s_t$ given by:
\begin{align}
    s_t:=\{\ell\in \{1, \dots, \text{len}(\boldsymbol{x}_1)\}|\;t^\ell_i\leq t\}
\end{align}
where $s_t[\ell]$ is the index of the token $\boldsymbol{x}^\ell_1$ in $\boldsymbol{x}_t$ at time $t$ and we have $\text{len}(s_t)\leq \text{len}(\boldsymbol{x}_t)$. 

\paragraph{Parameterization}
Just like fixed-length MDMs, the masking step in any-length MDMs has an inverse unmasking step that is parameterized with an \textbf{unmasking posterior} $f_\theta(\boldsymbol{x}_t,t)[\ell]:\mathcal{V}^L\times [0,1]\to \Delta^D$ that predicts the distribution over the token vocabulary for each masked position. This model is trained with the standard denoising cross-entropy loss defined as:
\begin{align}
    \mathcal{L}_{\text{unmask}}(\theta;\boldsymbol{x}_1)&=\underset{t\sim \mathcal{U}(0,1)}{\mathbb{E}}\underset{\boldsymbol{x}_t\sim p_t(\cdot |\boldsymbol{x}_1)}{\mathbb{E}} \bigg[-\frac{\dot{\beta}_t}{1-\beta_t}\sum_{\ell:\boldsymbol{x}^\ell_t=\boldsymbol{M}}\log f_\theta(\boldsymbol{x}_t,t)[\ell, \boldsymbol{x}_1^{s_t[\ell]}]\bigg]
\end{align}
Since any-length MDMs contain the additional deletion step, we parameterize the inverse \textbf{insertion} step with a \textbf{insertion expectation} $g_\theta(\boldsymbol{x}_t,t)[\ell]: \mathcal{V}^L\times[0,1] \to\mathbb{R}_{>0}$ which predicts the number of tokens that still need to be inserted between tokens $\boldsymbol{x}^{\ell-1}_t$ and $\boldsymbol{x}^{\ell}_t$. This model is trained by minimizing the scalar \textbf{Bregman divergence} $\phi(x, y)=y-x\log y$ between the predicted insertions and the true number of tokens that need to be inserted at each position:
\begin{align}
    \mathcal{L}_{\text{insert}}(\theta;\boldsymbol{x}_1)&=\underset{t\sim \mathcal{U}(0,1)}{\mathbb{E}}\underset{\boldsymbol{x}_t\sim p_t(\cdot |\boldsymbol{x}_1)}{\mathbb{E}} \left[-\frac{\dot{\alpha}_t}{1-\alpha_t}\sum_{\ell=0}^{\text{len}(\boldsymbol{x}_t)}\phi\left(s_t[\ell]-s_t[\ell-1]-1, g_\theta(\boldsymbol{x}_t,t)[\ell]\right)\right]
\end{align}

Minimizing both losses $\mathcal{L}_{\text{unmask}}(\theta)$ and $\mathcal{L}_{\text{insert}}(\theta)$ yields the \textit{unique} optimal value, which defines the rate matrices of the reverse process given by:
\begin{small}
\begin{align}
    \boldsymbol{R}_t(\boldsymbol{x}, \boldsymbol{x}^{\insertat{\ell}\boldsymbol{M}})&=\frac{\dot{\alpha}_t}{1-\alpha_t}\cdot\mathbb{E}[s_t[\ell]-s_t[\ell-1]-1|\;\boldsymbol{x}_t=\boldsymbol{x}]=\frac{\dot{\alpha}_t}{1-\alpha_t}\cdot g_\theta(\boldsymbol{x},t)[\ell]\tag{Insertion Rate}\\
    \boldsymbol{Q}_t(\boldsymbol{x}, \boldsymbol{x}^{\ell \gets \boldsymbol{v}})&=\frac{\dot{\beta}_t}{1-\beta_t}\cdot \mathbb{P}(\boldsymbol{x}_1^{s_t[\ell]}=\boldsymbol{v}|\boldsymbol{x}_t=\boldsymbol{x})=\frac{\dot{\beta}_t}{1-\beta_t}\cdot f_\theta(\boldsymbol{x}, t)[\ell, \boldsymbol{v}]\tag{Unmasking Rate}
\end{align}
\end{small}

\paragraph{Adaptive Sampling Guarantee}
A \textbf{crucial feature} of any-length MDMs is that the values of the rate matrix are \textit{decoupled from the choice of masking schedule $\beta_t$}. This means that the unmasking posterior learned by $f_\theta(\boldsymbol{x}_t,t)$ and insertion expectation $g_\theta(\boldsymbol{x}_t,t)$ capture all possible unmasking transitions that interpolate between the empty prior $p_0$ and the target distribution $p_{\text{target}}$. Given this, \citet{kim2025any} proves that any adaptive any-length MDM sampler for the $p_{\text{target}}$ will generate a sequence of intermediate states $\{\boldsymbol{x}_{t_i}\}_{i=1}^{N_{\text{steps}}}$ where the final sequence $\boldsymbol{x}_1$ is a sample from $p_{\text{target}}$.

\section{Theoretical Proofs}
\label{app:theory-proof}

\subsection{Observations}
\label{app:observations}
Throughout the theoretical proofs, we make the following \textbf{key observations}:
\begin{enumerate}
    \item [(i)] The unmasking rate $\boldsymbol{Q}_t$ and insertion rate $\boldsymbol{R}_t$ do not overlap, such that:
    \begin{align}
        \boldsymbol{Q}_t(\boldsymbol{x}, \boldsymbol{y})\boldsymbol{R}_t(\boldsymbol{x}, \boldsymbol{y})=0, \quad \forall t, \boldsymbol{x}, \boldsymbol{y}
    \end{align}
    which follows from the fact that insertion and unmasking actions are disjoint.
    \item[(ii)] The total generator $\boldsymbol{A}_t$ decomposes additively:
    \begin{align}
        \boldsymbol{A}_t(\boldsymbol{x}, \boldsymbol{y})=\boldsymbol{Q}_t(\boldsymbol{x}, \boldsymbol{y})+\boldsymbol{R}_t(\boldsymbol{x}, \boldsymbol{y}), \quad \forall t, \boldsymbol{x}, \boldsymbol{y}
    \end{align}
    \item[(iii)] Both insertion and unmasking rates are \textit{valid} such that:
    \begin{align}
        \boldsymbol{Q}_t(\boldsymbol{x}, \boldsymbol{x})=-\sum_{y \neq x}\boldsymbol{Q}_t(\boldsymbol{x}, \boldsymbol{y}), \quad \sum_{y \neq x}\boldsymbol{Q}_t(\boldsymbol{x}, \boldsymbol{y})< \infty\\
        \boldsymbol{R}_t(\boldsymbol{x}, \boldsymbol{x})=-\sum_{y \neq x}\boldsymbol{R}_t(\boldsymbol{x}, \boldsymbol{y}), \quad \sum_{y \neq x}\boldsymbol{R}_t(\boldsymbol{x}, \boldsymbol{y})< \infty
    \end{align}
    which ensures that the CTMC is well-defined such that reward tilting preserves the generator structure.
    \item[(iv)] The tilted generators $\boldsymbol{Q}^\star_t$ and $\boldsymbol{R}_t^\star$ and the reference generators $\boldsymbol{Q}^0_t$ and $\boldsymbol{R}_t^0$ have the same support, such that:
    \begin{align}
        \boldsymbol{Q}^0_t(\boldsymbol{x}, \boldsymbol{y})=0 \implies \boldsymbol{Q}^\star_t(\boldsymbol{x}, \boldsymbol{y})=0\\
        \boldsymbol{R}^0_t(\boldsymbol{x}, \boldsymbol{y})=0 \implies \boldsymbol{R}^\star_t(\boldsymbol{x}, \boldsymbol{y})=0
    \end{align}
    which ensures that the path measures $\mathbb{P}^0$ and $\mathbb{P}^\star$ are mutually absolutely continuous.
\end{enumerate}

\subsection{Proof of Proposition \ref{prop:unmaskloss}}
\label{app-prop:unmaskloss}
\begin{tcolorbox}[sharp corners, colback=mybg, boxrule=0pt]
\unmaskloss*
\end{tcolorbox}

\textit{Proof.} First, we recall the \textit{Unmasking Quality Loss} $\mathcal{L}_{\text{UQL}}$ defined in (\ref{loss:uql}) as:
\begin{align}
    \mathcal{L}_{\text{UQL}}(\phi; \boldsymbol{x}_1):=\underset{t\sim \mathcal{U}(0,1)}{\mathbb{E}}\underset{\tilde{\boldsymbol{x}}_t\sim p_t(\cdot |\boldsymbol{x}_1), \boldsymbol{y}}{\mathbb{E}}\underset{\sim p_{s|t}(\cdot |\tilde{\boldsymbol{x}}_t)}{\mathbb{E}}\left[\sum_{\ell\in \mathcal{M}}\text{BCE}\left(\boldsymbol{1}[\boldsymbol{y}^\ell=\boldsymbol{x}_1^{s_t[\ell]}], \mu_\phi^\ell(\boldsymbol{y})\right)\right]
\end{align}
Given a clean sequence $\boldsymbol{x}_1\sim p_{\text{data}}$, let $\tilde{\boldsymbol{x}}_t$ be a sample from the posterior $\tilde{\boldsymbol{x}}_t\sim p_t(\cdot|\boldsymbol{x}_1)$. Then, we sample a subset $\mathcal{M}_t$ of masked indices in $\tilde{\boldsymbol{x}}_t$ to unmask to obtain $\boldsymbol{y}$ where $\forall \ell \in \mathcal{M}_t$, we sample $\boldsymbol{y}^\ell \sim f_\theta(\tilde{\boldsymbol{x}}_t, t)[\ell]$.

Fix a position $\ell \in \mathcal{M}_t$ and consider the random binary variable $b_\ell$ defined by categorically sampling from $f_\theta(\tilde{\boldsymbol{x}}_t, t)[\ell]$ and the model predicted token quality $c_\ell(\boldsymbol{y}):=\mu_\phi^\ell(\boldsymbol{y})$:
\begin{align}
    b^\ell:=\boldsymbol{1}[\boldsymbol{y}^\ell=\boldsymbol{x}_1^{s_t[\ell]}]\in \{0,1\}, \quad c_\ell(\boldsymbol{y}):=\mu_\phi^\ell(\boldsymbol{y})\in [0,1]
\end{align}
For every position $\ell\in \mathcal{M}_t$, we minimize the following loss:
\begin{align}
    \mathcal{L}_{\text{UQL}}(\phi; \ell):=\mathbb{E}_{\boldsymbol{x}_1, \tilde{\boldsymbol{x}}_t, \boldsymbol{y}}\left[\text{BCE}\left(b_\ell, c_\ell(\boldsymbol{y})\right)\right]
\end{align}
Applying the law of total expectation, we have:
\begin{align}
    \mathbb{E}_{\boldsymbol{x}_1, \tilde{\boldsymbol{x}}_t, \boldsymbol{y}}\left[\text{BCE}\left(b_\ell, c_\ell(\boldsymbol{y})\right)\right]=\mathbb{E}_{\boldsymbol{y}}\underbrace{\left[\mathbb{E}_{\boldsymbol{x}_1, \tilde{\boldsymbol{x}}_t}\big[\text{BCE}(b_\ell, c_\ell(\boldsymbol{y}))\big|\boldsymbol{y}\big]\right]}_{:=D_\ell(\boldsymbol{y}, c_\ell)}
\end{align}
Let us define the conditional random variable:
\begin{align}
    q_\ell(\boldsymbol{y}):=p(b_\ell=1|\boldsymbol{y})=p(\boldsymbol{y}^\ell=\boldsymbol{x}_1^{s_t[\ell]}|\boldsymbol{y})=\mathbb{E}_{\boldsymbol{x}_1, \tilde{\boldsymbol{x}}_t}\left[b_\ell |\boldsymbol{y}\right]
\end{align}
Then, we can write the BCE loss $D_\ell(\boldsymbol{y}, c)$ as:
\begin{align}
    D_\ell(\boldsymbol{y}, c_\ell)=-q_\ell(\boldsymbol{y})\log c_\ell-(1-q_\ell(\boldsymbol{y}))\log(1-c_\ell)
\end{align}
Differentiating with respect to the parameters $\phi$ denoted as $c_\ell$ and setting the derivative equal to $0$, we get that the minimizer is:
\begin{align}
    \frac{d}{dc_\ell}D_\ell(\boldsymbol{y}, c_\ell)=-\frac{q_\ell(\boldsymbol{y})}{c_\ell}+\frac{1-q_\ell(\boldsymbol{y})}{1-c_\ell}\implies \frac{q_\ell(\boldsymbol{y})}{c^\star_\ell}=\frac{1-q_\ell(\boldsymbol{y})}{1-c^\star_\ell}\implies c_\ell^\star=q_\ell(\boldsymbol{y})
\end{align}
Taking the second derivative with respect to $c_\ell$, we show that this minimizer is unique:
\begin{align}
    \frac{d^2}{dc_\ell^2}D_\ell(\boldsymbol{y}, c_\ell)=\frac{q_\ell(\boldsymbol{y})}{c_\ell^2}+\frac{1-q_\ell(\boldsymbol{y})}{(1-c_\ell)^2}\geq 0
\end{align}
Since we defined the true unmasking quality as $\mu_\star^\ell(\boldsymbol{y}):=p(\boldsymbol{y}^\ell=\boldsymbol{x}_1^{s_t[\ell]}|\boldsymbol{y})$, the minimizer $c^\star$ exactly matches the true unmasking posterior:
\begin{align}
    c_\ell^\star=q_\ell(\boldsymbol{y})=\mu_\star^\ell(\boldsymbol{y})=\underset{c_\ell\in [0,1]}{\arg\min}\left\{D_\ell(\boldsymbol{y}, c_\ell)\right\}
\end{align}
Since for all pairs $(\boldsymbol{y}, \ell)$, we have shown that the BCE loss $D_\ell(\boldsymbol{y}, c_\ell)$ is minimized at $\mu_\phi^\ell(\boldsymbol{y})=\mu_\star^\ell(\boldsymbol{y})$, so the function that minimizes the full objective $\mathcal{L}_{\text{UQL}}(\phi):=\mathbb{E}_{\boldsymbol{y}}[D_\ell(\boldsymbol{y}, c_\ell)]$ must be the true function $\mu_\star$. \hfill $\square$

\subsection{Proof of Proposition \ref{prop:insertionloss}}
\label{app-prop:insertionloss}
\begin{tcolorbox}[sharp corners, colback=mybg, boxrule=0pt]
\insertionloss*
\end{tcolorbox}

\textit{Proof.} First, we recall the \textit{Insertion Quality Loss} $\mathcal{L}_{\text{IQL}}$ defined in (\ref{loss:iql}) as:
\begin{align}
    \mathcal{L}_{\text{IQL}}(\phi):=\mathbb{E}_{\boldsymbol{x}_1, \tilde{\boldsymbol{x}}_t, \boldsymbol{y}}\left[\sum_{i\in \mathcal{I}}\text{BCE}\big(\nu_\star^i(\boldsymbol{y}), \nu_\phi^i(\boldsymbol{y})\big)\right]
\end{align}
where the BCE is defined as: 
\begin{align}
    \text{BCE}(b, c)=-b\log c-(1-b) \log (1-c), \quad b,c\in [0,1]
\end{align} 
Since the sum $\sum_{i\in \mathcal{I}}\text{BCE}\big(\nu_\star^i(\boldsymbol{y}), \nu_\phi^i(\boldsymbol{y})\big)$ is inside the expectation, and the loss is calculated independently for each position $i$ given $\nu_\phi^i(\boldsymbol{y})$, minimizing $\mathcal{L}_{\text{IQL}}(\phi)$ is equivalent to minimizing the following for every inserted position $i$:
\begin{align}
    \mathcal{L}_{\text{IQL}}(\phi; i):=\mathbb{E}_{\boldsymbol{x}_1, \tilde{\boldsymbol{x}}_t, \boldsymbol{y}}\left[\text{BCE}\big(\nu_\star^i(\boldsymbol{y}), \nu_\phi^i(\boldsymbol{y})\big)\right]
\end{align}
Applying the law of total expectation, we have:
\begin{align}
    \mathbb{E}_{\boldsymbol{x}_1, \tilde{\boldsymbol{x}}_t, \boldsymbol{y}}\left[\text{BCE}\big(\nu_\star^i(\boldsymbol{y}), \nu_\phi^i(\boldsymbol{y})\big)\right]&=\mathbb{E}_{\boldsymbol{y}}\underbrace{\left[\mathbb{E}_{\boldsymbol{x}_1, \tilde{\boldsymbol{x}}_t}\left[\text{BCE}\big(\nu_\star^i(\boldsymbol{y}), \nu_\phi^i(\boldsymbol{y})\big)\big|\boldsymbol{y}\right]\right]}_{:=B_i(\boldsymbol{y}, c)}
\end{align}
where we define $c_i:=\nu_\phi^i(\boldsymbol{y})\in [0,1]$ as the scalar predicted insertion quality at position $i$ and $B_i(\boldsymbol{y}, c_i)$ as the prediction error given the input sequence $\boldsymbol{y}$ and predicted quality $c_i$. Then, it suffices to prove that for all $(\boldsymbol{y}, i)$, the prediction error is \textit{uniquely minimized} at $c_i=\nu_\star^i(\boldsymbol{y})$.

Fixing a pair $(\boldsymbol{y}, i)$, we can define: 
\begin{align}
    b_i:=\nu_\star^i(\boldsymbol{y})\in [0,1], \quad c_i:=\nu_\phi^i(\boldsymbol{y})\in [0,1]
\end{align}
which yields the prediction error objective:
\begin{align}
    B_i(\boldsymbol{y}, c_i)=-b_i\log c_i-(1-b_i) \log (1-c_i)
\end{align}
Differentiating with respect to $c_i$ (parameterized model), we have:
\begin{align}
    \frac{d}{dc_i}B_i(\boldsymbol{y}, c_i)=-\frac{b_i}{c_i}+\frac{1-b_i}{1-c_i}
\end{align}
and setting the derivative to zero, we get the minimizer:
\begin{align}
    -\frac{b_i}{c^\star_i}+\frac{1-b_i}{1-c^\star_i}=0\implies \frac{b_i}{c^\star_i}=\frac{1-b_i}{1-c^\star_i}\implies c^\star_i=b_i
\end{align}
To prove uniqueness, we take the second derivative:
\begin{align}
    \frac{d^2}{dc_i^2}B_i(\boldsymbol{y}, c_i)=\frac{b_i}{c_i^2}+\frac{1-b_i}{(1-c_i)^2}\geq 0
\end{align}
So, the objective function is convex in $c_i$ and the minimizer is unique. Therefore, we conclude that:
\begin{align}
    b_i=\nu_\star^i(\boldsymbol{y})=\underset{c_i\in [0,1]}{\arg\min}\left\{B_i(\boldsymbol{y}, c_i)\right\}
\end{align}
We have shown that for all $(\boldsymbol{y}, i)$, the prediction error is uniquely minimized at $\nu_\phi^i(\boldsymbol{y})=\nu_\star^i(\boldsymbol{y})$, the only function $\nu_\phi$ that minimizes the outer expectation $\mathbb{E}_{\boldsymbol{y}}[B_i(\boldsymbol{y}, c_i)]$ is the true function $\nu_\star$. \hfill $\square$

\subsection{Proof of Proposition \ref{prop:optimalunmasking}}
\label{app-prop:optimalunmasking}
\begin{restatable}[Optimal Parallel Unmasking as Maximizing the Unmasking Quality]{proposition}{optimalunmasking2}
Assume that in a parallel unmasking step on indices $\mathcal{M}_t=\{\ell_k\}_{k=1}^K$, the unmasked tokens are conditionally independent given the unchanged context $\bar{\boldsymbol{x}}_s$ and the current state $\boldsymbol{x}_s$, i.e.
$p(\boldsymbol{x}_t^{\mathcal{M}_t} | \bar{\boldsymbol{x}}_s,\boldsymbol{x}_s)=\prod_{k=1}^K p(\boldsymbol{x}_t^{\ell_k}| \bar{\boldsymbol{x}}_s,\boldsymbol{x}_s)$.
If we define the unmasking quality $\mu_\star^{\ell_k}(\boldsymbol{x}_t):=p(\boldsymbol{x}_t^{\ell_k}=\boldsymbol{x}_1^{\ell_k}| \boldsymbol{x}_t^{\ell_k\leftarrow \boldsymbol{M}})$ and note that $p(\boldsymbol{x}_t^{\ell_k}| \bar{\boldsymbol{x}}_s,\boldsymbol{x}_s)=p(\boldsymbol{x}_t^{\ell_k}|\boldsymbol{x}_t^{\ell_k\leftarrow \boldsymbol{M}})$, then
\begin{align}
    p(\boldsymbol{x}_t^{\mathcal{M}_t}| \bar{\boldsymbol{x}}_s,\boldsymbol{x}_s)=\prod_{k=1}^K \mu_\star^{\ell_k}(\boldsymbol{x}_t)
\end{align}
Furthermore, the compounding parallelization error can be written as the KL divergence
\begin{align}
    \mathcal{E}^{\mathrm{CPE}}_{\mathrm{unmask}}(s\to t| \boldsymbol{x}_s)
:=D_{\mathrm{KL}}\!\left(p(\boldsymbol{x}_t^{\mathcal{M}_t}| \bar{\boldsymbol{x}}_s,\boldsymbol{x}_s)\,\middle\|\,\prod_{k=1}^K p(\boldsymbol{x}_t^{\ell_k}| \bar{\boldsymbol{x}}_s,\boldsymbol{x}_s)\right)\geq 0
\end{align}
and admits the decomposition
\begin{align}
\mathcal{E}^{\mathrm{CPE}}_{\mathrm{unmask}}(s\to t| \boldsymbol{x}_s)=\mathbb{E}_{p_t}\big[\log p(\boldsymbol{x}_t^{\mathcal{M}_t}| \bar{\boldsymbol{x}}_s,\boldsymbol{x}_s)\big]
-\sum_{k=1}^K \mathbb{E}_{p_t}\big[\log \mu_\star^{\ell_k}(\boldsymbol{x}_t)\big]
\end{align}
with equality to $0$ if and only if the KL divergence is $0$ (i.e., optimal parallel decoding).
\end{restatable}

\textit{Proof.} First, we recall the definition of token quality. Given a partially masked and deleted sequence $\boldsymbol{x}_t$, the unmasking quality of each token $\boldsymbol{x}_t^\ell$ is the probability that given the sequence masked at position $\ell$, denoted $\tilde{\boldsymbol{x}}_s:=\boldsymbol{x}_t^{\ell \gets \boldsymbol{M}}$, we sample $\tilde{\boldsymbol{x}}_s^\ell=\boldsymbol{x}_t^\ell$ via the unmasking posterior $f_\theta(\tilde{\boldsymbol{x}}_s, s)$. This is written as:
\begin{align}
    \mu_\star^\ell(\tilde{\boldsymbol{x}}_s)&=p(\tilde{\boldsymbol{x}}_s^\ell=\boldsymbol{x}_t^\ell|\boldsymbol{x}_t^{\ell \gets \boldsymbol{M}})
\end{align}
Consider a unmasking step $\boldsymbol{x}_s\to \boldsymbol{x}_t$ where a subset of $K$ tokens with indices $\mathcal{M}_t:=\{\ell_k \}_{k=1}^K $ are unmasked where $\bar{\boldsymbol{x}}_s$ is the subset of the sequence that remains unchanged from $s\to t$. Let the joint conditional and the product of marginals conditioned on $\boldsymbol{x}_s$ be denoted as:
\begin{align}
    p_{\mathcal{M}_t}(\cdot):=p(\boldsymbol{x}_t^{\mathcal{M}_t}|\bar{\boldsymbol{x}}_s, \boldsymbol{x}_s), \quad q_{\mathcal{M}_t}(\cdot):=\prod_{k=1 }^Kp(\boldsymbol{x}^{\ell_k}_t|\bar{\boldsymbol{x}}_s, \boldsymbol{x}_s)
\end{align}
Then, the compounding parallelization error (CPE) of the step is given by:
\begin{align}
    \mathcal{E}^{\text{CPE}}_{\text{unmask}}(s\to t|\boldsymbol{x}_s)=D_{\text{KL}}\left(p_{\mathcal{M}_t}\|q_{\mathcal{M}_t}\right)\geq 0
\end{align}
where $\mathcal{E}^{\text{CPE}}_{\text{unmask}}(s\to t|\boldsymbol{x}_s)=0$ if and only if $p_{\mathcal{M}_t}=q_{\mathcal{M}_t}$ and all the unmasked tokens are \textit{conditionally independent} such that the product of the marginal probabilities and the joint probability is equal given $(\bar{\boldsymbol{x}}_s, \boldsymbol{x}_s)$. Expanding the KL divergence yields:
\begin{align}
    \mathcal{E}^{\text{CPE}}_{\text{unmask}}(s\to t|\boldsymbol{x}_s)&=\mathbb{E}_{\boldsymbol{x}_t^{\mathcal{M}_t}\sim p_{\mathcal{M}_t}}\left[\log p(\boldsymbol{x}_t^{\mathcal{M}_t}|\bar{\boldsymbol{x}}_s, \boldsymbol{x}_s)-\sum_{k=1}^K\log p(\boldsymbol{x}^{\ell_k}_t|\bar{\boldsymbol{x}}_s, \boldsymbol{x}_s)\right]\geq 0\label{eq:proof-cpe}
\end{align}

In the case where all unmasked tokens are conditionally independent, we have:
\begin{align}
    p(\boldsymbol{x}_t^{\mathcal{M}_t}|\bar{\boldsymbol{x}}_s, \boldsymbol{x}_s)=\prod_{k=1 }^Kp(\boldsymbol{x}^{\ell_k}_t|\bar{\boldsymbol{x}}_s, \boldsymbol{x}_s)\label{eq:proof-6}
\end{align}
In addition, under conditional independence of $\mathcal{M}_t$, the probability of a particular unmasked token $\boldsymbol{x}_t^{\ell_k}$ should be the same given the context $\boldsymbol{x}_s$ and given the context with only $\ell_k$ masked $\boldsymbol{x}_t^{\ell_k \gets \boldsymbol{M}}$, which aligns with the definition of the true unmasking quality. Therefore,
\begin{align}
    p(\boldsymbol{x}_t^{\ell_k}|\bar{\boldsymbol{x}}_s, \boldsymbol{x}_s)=p(\boldsymbol{x}_t^{\ell_k}|\boldsymbol{x}_t^{\ell_k\gets \boldsymbol{M}})=:\mu_\star^\ell(\boldsymbol{x}_t)\label{eq:proof-5}
\end{align}
Substituting (\ref{eq:proof-5}) into the optimal parallel unmasking condition in (\ref{eq:proof-6}), we have:
\begin{align}
    p(\boldsymbol{x}_t^{\mathcal{M}_t}|\bar{\boldsymbol{x}}_s, \boldsymbol{x}_s)=\prod_{k=1 }^Kp(\boldsymbol{x}^{\ell_k}_t|\bar{\boldsymbol{x}}_s, \boldsymbol{x}_s)=\prod_{k=1}^K\mu_\star^\ell(\boldsymbol{x}_t)
\end{align}
which means that maximizing the product of token qualities maximizes the probability of optimal parallel unmasking. We can also derive a tractable Monte Carlo estimator for the CPE in (\ref{eq:proof-cpe}) as:
\begin{small}
\begin{align}
    \widehat{\mathcal{E}}^{\text{CPE}}_{\text{unmask}}(s\to t|\boldsymbol{x}_s)&=\log p(\boldsymbol{x}_t^{\mathcal{M}_t}|\bar{\boldsymbol{x}}_s, \boldsymbol{x}_s)-\sum_{k=1}^K\log p(\boldsymbol{x}^{\ell_k}_t|\bar{\boldsymbol{x}}_s, \boldsymbol{x}_s), \quad \mathcal{E}^{\text{CPE}}_{\text{unmask}}=\mathbb{E}_{p_t}\left[\widehat{\mathcal{E}}^{\text{CPE}}_{\text{unmask}}\right]\geq 0
\end{align}
\end{small}
which concludes our proof. \hfill $\square$ 

\subsection{Proof of Proposition \ref{prop:optimalinsertion}}
\label{app-prop:optimalinsertion}
\begin{tcolorbox}[sharp corners, colback=mybg, boxrule=0pt]
    \optimalinsertion*
\end{tcolorbox}

\textit{Proof.} First, we recall the definition of insertion quality as the sum of the probabilities of all possible tokens that exist within the gap in the ground truth sequence, defined as:
\begin{align}
    \nu_\star(\boldsymbol{y}):=\underset{s_t[\ell-1]< i < s_t[\ell]}{\max}\left\{p(\boldsymbol{y}^{\ell}=\boldsymbol{x}_1^{i}|\boldsymbol{y})\right\}
\end{align}
Consider a target sequence $\boldsymbol{x}_1\sim p_{\text{target}}$. After an insertion step $(\mathcal{I}_t, \boldsymbol{x}_t) \sim g_\theta(\boldsymbol{x}_s,s)$, we unmask the masked tokens to reconstruct $\boldsymbol{x}_1$. For this proof, we assume that the unmasking posterior factorizes across all masked tokens, which gives the equality: 
\begin{align}
     p(\boldsymbol{x}_t^{\mathcal{I}_t}=\boldsymbol{x}_1^{\mathcal{I}_t}|\boldsymbol{x}_t)&=\prod_{i \in \mathcal{I}} p(\boldsymbol{x}_t^i=\boldsymbol{x}_1^i|\boldsymbol{x}_t)\\
     \log p(\boldsymbol{x}_t^{\mathcal{I}_t}=\boldsymbol{x}_1^{\mathcal{I}_t}|\boldsymbol{x}_t)&=\sum_{i \in \mathcal{I}}\log  p(\boldsymbol{x}_t^i=\boldsymbol{x}_1^i|\boldsymbol{x}_t)
\end{align}
Since we define the \textbf{insertion quality} $\nu_\star^i$ at an inserted mask position $i$ as the largest probability mass assigned by the unmasking posterior to any of the ground truth tokens at the gap $\mathcal{G}_i:=[s_t[i-1], s_t[i]]$, we can write:
\begin{align}
    p(\boldsymbol{x}^i_t=\boldsymbol{x}^{\mathcal{G}_i}_1|\boldsymbol{x}_t)\leq \nu_\star^i(\boldsymbol{x}_t)\implies \log p(\boldsymbol{x}^i_t=\boldsymbol{x}^{\mathcal{G}_i}_1|\boldsymbol{x}_t)\leq \log \nu_\star^i(\boldsymbol{x}_t)
\end{align}
Assuming conditional independence across all inserted masks given $\boldsymbol{x}_t$,we get:
\begin{align}
     p(\boldsymbol{x}_t^{\mathcal{I}_t}=\boldsymbol{x}_1^{\mathcal{I}_t}|\boldsymbol{x}_t)=\prod_{i\in \mathcal{I}} p(\boldsymbol{x}^i_t=\boldsymbol{x}^{\mathcal{G}_i}_1|\boldsymbol{x}_t)\leq \prod_{i\in \mathcal{I}}\nu_\star^i(\boldsymbol{x}_t)
\end{align}
where $\boldsymbol{x}_t^{\mathcal{I}_t}=\boldsymbol{x}_1^{\mathcal{I}_t}$ means that target tokens are aligned to inserted slots. Taking the expectation over $(\mathcal{I}_t, \boldsymbol{x}_t)\sim p(\mathcal{I}_t, \boldsymbol{x}_t|\boldsymbol{x}_s)$, we have:
\begin{align}
    \mathbb{E}_{p(\mathcal{I}_t, \boldsymbol{x}_t|\boldsymbol{x}_s)}\left[p(\boldsymbol{x}_t^{\mathcal{I}_t}=\boldsymbol{x}_1^{\mathcal{I}_t}|\boldsymbol{x}_t)\right]\leq\mathbb{E}_{p(\mathcal{I}_t, \boldsymbol{x}_t|\boldsymbol{x}_s)}\left[\prod_{i\in \mathcal{I}}\nu_\star^i(\boldsymbol{x}_t)\right] 
\end{align}

Therefore, maximizing $\prod_{i\in \mathcal{I}}\nu_\star^i(\boldsymbol{x}_t)$ maximizes a tractable upper bound on the probability that all inserted slots are reconstructed correctly, providing a principled surrogate objective for insertion scheduling. \hfill $\square$

\subsection{Compounding Parallelization Error for Insertion and Unmasking}

First, we establish the following lemma that decomposes the KL divergence between joint distributions.
\begin{lemma}[KL Divergence Chain Rule]
\label{lemma:kl-chain}
    The KL divergence between the joint distributions $\mathbb{P}_{s,t}(\boldsymbol{X}_s,\boldsymbol{X}_t)$ and $\mathbb{P}'_{s,t}(\boldsymbol{X}_s,\boldsymbol{X}_t)$ on the joint space $\mathcal{X}\times \mathcal{X}$ where $\mathcal{X}$ is a finite discrete state space, satisfies:
    \begin{align}
        D_{\text{KL}}(\mathbb{P}_{s,t}\|\mathbb{P}'_{s,t})&=D_{\text{KL}}(\mathbb{P}_s\|\mathbb{P}'_s)+\mathbb{E}_{\boldsymbol{x}_s\sim \mathbb{P}_s}\left[D_{\text{KL}}(\mathbb{P}_{t|s}(\cdot|\boldsymbol{x}_s)\|\mathbb{P}'_{t|s}(\cdot|\boldsymbol{x}_s)\right]
    \end{align}
    where $\mathbb{P}_s, \mathbb{P}'_s$ denote marginal distributions and $\mathbb{P}_{t|s}, \mathbb{P}'_{t|s}$ denote conditional distributions.
\end{lemma}

\textit{Proof.} By expanding the definition of KL divergence, we have:
\begin{align}
    D_{\text{KL}}(\mathbb{P}_{s,t}\|\mathbb{P}'_{s,t})&=\sum_{\boldsymbol{x}_s, \boldsymbol{x}_t}\mathbb{P}_{s,t}(\boldsymbol{x}_s, \boldsymbol{x}_t)\log \frac{\mathbb{P}_{s,t}(\boldsymbol{x}_s, \boldsymbol{x}_t)}{\mathbb{P}'_{s,t}(\boldsymbol{x}_s, \boldsymbol{x}_t)}\nonumber\\
    &=\sum_{\boldsymbol{x}_s, \boldsymbol{x}_t}\mathbb{P}_{t|s}(\boldsymbol{x}_t|\boldsymbol{x}_s)\mathbb{P}_{s}(\boldsymbol{x}_s)\log \frac{\mathbb{P}_{t|s}(\boldsymbol{x}_t|\boldsymbol{x}_s)\mathbb{P}_{s}(\boldsymbol{x}_s)}{\mathbb{P}'_{t|s}(\boldsymbol{x}_t|\boldsymbol{x}_s)\mathbb{P}'_{s}(\boldsymbol{x}_s)}\nonumber\\
    &=\underbrace{\sum_{\boldsymbol{x}_s, \boldsymbol{x}_t}\mathbb{P}_{t|s}(\boldsymbol{x}_t|\boldsymbol{x}_s)}_{=1}\mathbb{P}_{s}(\boldsymbol{x}_s)\log \frac{\mathbb{P}_{s}(\boldsymbol{x}_s)}{\mathbb{P}'_{s}(\boldsymbol{x}_s)}+\sum_{\boldsymbol{x}_s, \boldsymbol{x}_t}\mathbb{P}_{t|s}(\boldsymbol{x}_t|\boldsymbol{x}_s)\mathbb{P}_{s}(\boldsymbol{x}_s)\log \frac{\mathbb{P}_{t|s}(\boldsymbol{x}_t|\boldsymbol{x}_s)}{\mathbb{P}'_{t|s}(\boldsymbol{x}_t|\boldsymbol{x}_s)}\nonumber\\
    &=\underbrace{\sum_{\boldsymbol{x}_s}\mathbb{P}_s(\boldsymbol{x}_s)\log \frac{\mathbb{P}_s(\boldsymbol{x}_s)}{\mathbb{P}'_s(\boldsymbol{x}_s)}}_{D_{\text{KL}}(\mathbb{P}_{s}\|\mathbb{P}_{s})}+\sum_{\boldsymbol{x}_s}\mathbb{P}_s(\boldsymbol{x}_s)\underbrace{\sum_{\boldsymbol{x}_s, \boldsymbol{x}_t}\mathbb{P}_{t|s}(\boldsymbol{x}_t|\boldsymbol{x}_s)\log \frac{\mathbb{P}_{t|s}(\boldsymbol{x}_t|\boldsymbol{x}_s)}{\mathbb{P}'_{t|s}(\boldsymbol{x}_t|\boldsymbol{x}_s)}}_{D_{\text{KL}}(\mathbb{P}_{t|s}\|\mathbb{P}_{t|s})}\nonumber\\
    &=D_{\text{KL}}(\mathbb{P}_{s}\|\mathbb{P}_{s})+\mathbb{E}_{\boldsymbol{x}_s\sim \mathbb{P}_s}[D_{\text{KL}}(\mathbb{P}_{t|s}\|\mathbb{P}_{t|s})]
\end{align}
which concludes the proof. \hfill $\square$

\begin{proposition}[CPE as Upper Bound on KL Divergence]
    Given a sampling schedule $\{T\to \dots\to t_{K-1}\to 1\}$, the KL divergence between the sampled distribution $\mathbb{P}^v_0$ and true distribution $\mathbb{P}^\star_1$ is upper bounded by the total CPE:
    \begin{align}
        D_{\text{KL}}(\mathbb{P}^\star_1\|\mathbb{P}^v_1)&\leq \sum_{k=0}^{K-1}\bigg[\mathcal{E}^{\text{ins}}_{\text{CPE}}(t_k\to t_{k+1})+\mathcal{E}^{\text{unmsk}}_{\text{CPE}}(t_k\to t_{k+1})\bigg]
    \end{align}
\end{proposition}

\textit{Proof.} First, we recall our definition for the CPE for insertion as:
\begin{align}
    \mathcal{E}_{\text{CPE}}(s\to t):=\mathbb{E}_{\boldsymbol{x}_s\sim \mathbb{P}_s}\bigg[D_{\text{KL}}\bigg(p(\mathcal{I}_t, \boldsymbol{X}_t|\boldsymbol{X}_s)\bigg\|p(\mathcal{I}_t|\boldsymbol{X}_s)p(\boldsymbol{X}_t|\boldsymbol{X}_s)\bigg)\bigg]
\end{align}
and the CPE for unmasking as:
\begin{align}
    \mathcal{E}_{\text{CPE}}^{\text{unmsk}}(s\to t)&=\mathbb{E}_{\boldsymbol{x}_s\sim \mathbb{P}_s}\bigg[D_{\text{KL}}\bigg(p(\boldsymbol{X}_t^{\ell_1}, \dots, \boldsymbol{X}_t^{\ell_n}|\boldsymbol{X}_s)\bigg\|\prod_{i=1}^np(\boldsymbol{X}_t^{\ell_i}|\boldsymbol{X}_s)\bigg)\bigg]
\end{align}
Denoting the joint distribution under the true path measure $\mathbb{P}^\star$ as $\mathbb{P}^\star_{t_{k+1}|t_k}(\cdot |\boldsymbol{x}_{t_k}):=p(\boldsymbol{X}_{t_{k+1}}^{\ell_1}, \dots, \boldsymbol{X}_{t_{k+1}}^{\ell_n}|\boldsymbol{X}_{t_k})$ and the parameterized marginal distribution as $\mathbb{P}^v_{t_{k+1}|t_k}(\cdot|\boldsymbol{x}_{t_k}):=\prod_{i=1}^np(\boldsymbol{X}_{t_{k+1}}^{\ell_i}|\boldsymbol{X}_{t_k})$, we can write the terminal distribution generated by both path measures as:
\begin{align}
    \mathbb{P}^\star_0 =\mathbb{P}^\star_{t_{K}|t_{K-1}}\dots  \mathbb{P}^\star_{t_1|t_0}\mathbb{P}^\star_{t_0}, \quad \mathbb{P}^v_0 =\mathbb{P}^v_{t_{K}|t_{K-1}}\dots  \mathbb{P}^v_{t_1|t_0}\mathbb{P}^v_{t_0}
\end{align}
Now, using Lemma \ref{lemma:kl-chain}, we can write the KL divergence between the marginals at time $t$ where $s>t$ as:
\begin{align}
    D_{\text{KL}}(\mathbb{P}^\star_t\|\mathbb{P}^v_t)&\leq D_{\text{KL}}(\mathbb{P}^\star_{s,t}\|\mathbb{P}^v_{s,t})\nonumber\\
    &=D_{\text{KL}}(\mathbb{P}^\star_s\|\mathbb{P}^v_s)+\mathbb{E}_{\boldsymbol{x}_s\sim \mathbb{P}_s}\left[D_{\text{KL}}(\mathbb{P}_{t|s}(\cdot|\boldsymbol{x}_s)\|\mathbb{P}'_{t|s}(\cdot|\boldsymbol{x}_s))\right]\nonumber\\
    &=D_{\text{KL}}(\mathbb{P}^\star_s\|\mathbb{P}^v_s)+\mathbb{E}_{\boldsymbol{x}_s\sim \mathbb{P}_s}\left[\mathcal{E}_{\text{CPE}}^{\text{insert}}(s\to t|\boldsymbol{x}_s)+\mathcal{E}_{\text{CPE}}^{\text{unmask}}(s\to t|\boldsymbol{x}_s)\right]\nonumber\\
    &=D_{\text{KL}}(\mathbb{P}^\star_s\|\mathbb{P}^v_s)+\mathcal{E}_{\text{CPE}}^{\text{insert}}(s\to t)+\mathcal{E}_{\text{CPE}}^{\text{unmask}}(s\to t)
\end{align}
where the equality holds if and only if $D_{\text{KL}}(\mathbb{P}^\star_{t|s}\|\mathbb{P}_{t|s}^v)=0$. Then, applying this inequality over all $K$ time steps from $t_0\to \dots \to t_{K-1}\to t_K=1$, we have:
\begin{align}
    D_{\text{KL}}(\mathbb{P}^\star_0\|\mathbb{P}^v_0)&=D_{\text{KL}}(\mathbb{P}^\star_{t_K}\|\mathbb{P}_{t_K}^v)\nonumber\\
    &\leq D_{\text{KL}}(\mathbb{P}^\star_{t_{K-1}}\|\mathbb{P}^v_{t_{K-1}})+\mathcal{E}_{\text{CPE}}^{\text{insert}}(t_{K-1}\to t_K)+\mathcal{E}_{\text{CPE}}^{\text{unmask}}(t_{K-1}\to t_K)\nonumber\\
    &\vdots\nonumber\\
    &\leq  \underbrace{D_{\text{KL}}(\mathbb{P}^\star_{t_0}\|\mathbb{P}^v_{t_0})}_{=0}+\sum_{k=0}^{K-1}\left[\mathcal{E}_{\text{CPE}}^{\text{insert}}(t_k\to t_{k+1})+\mathcal{E}_{\text{CPE}}^{\text{unmask}}(t_{k}\to t_{k+1})\right]
\end{align}
Since the distribution at time $t_0$ is the fully empty sequence in any-length MDMs (i.e. $\mathbb{P}^\star_{t_0}=\mathbb{P}^v_{t_0}=\pi_{t_0}$), we have finished the proof. This means that the KL divergence of the generated distribution and the true distribution is zero if and only if the compounding parallelization error over all time steps is zero, i.e. $\mathcal{E}_{\text{CPE}}^{\text{insert}}(t_k\to t_{k+1})=0$ and $\mathcal{E}_{\text{CPE}}^{\text{unmask}}(t_{k}\to t_{k+1})=0$ for all $k\in \{0, \dots, K-1\}$.\hfill $\square$

\subsection{Radon-Nikodym Derivative Between Joint CTMCs}
\label{app:rnd-joint}
\begin{proposition}[Radon-Nikodym Derivative of Joint Any-Length CTMCs]\label{app-prop:RND}
    Consider two joint any-length path measures $\mathbb{P}$ and $\mathbb{P}'$ defined by the unmasking rate matrices $\boldsymbol{Q}$ and $\boldsymbol{Q}'$ and the insertion rate matrices $\boldsymbol{R}$ and $\boldsymbol{R}'$. Then, the Radon-Nikodym derivative over the trajectory $\boldsymbol{X}_{0:1}=(\boldsymbol{X}_t)_{t\in [0,1]}$ is defined as:
    \begin{small}
    \begin{align}
        \log\frac{\mathrm{d}\mathbb{P}'}{\mathrm{d}\mathbb{P}}(\boldsymbol{X}_{0:1})&=\log\frac{\mathrm{d}\pi_0'}{\mathrm{d}\pi_0}(\boldsymbol{X}_0)\nonumber\\
        &+\sum_{t_u:\boldsymbol{X}_{t_u}\neq  \boldsymbol{X}_{t_u-}}\log \frac{\boldsymbol{Q}'_{t_u}(\boldsymbol{X}_{t_u-}, \boldsymbol{X}_{t_u})}{\boldsymbol{Q}_{t_u}(\boldsymbol{X}_{t_u-}, \boldsymbol{X}_{t_u})}+\int_0^1\sum_{z\neq \boldsymbol{X}_{t}}(\boldsymbol{Q}_{t}-\boldsymbol{Q}'_{t})(\boldsymbol{X}_{t}, z)dt\nonumber\\
        &+\sum_{t_i: \boldsymbol{X}_{t_i}\neq \boldsymbol{X}_{t_i-}}\log \frac{\boldsymbol{R}'_{t_i}(\boldsymbol{X}_{t_i-}, \boldsymbol{X}_{t_i})}{\boldsymbol{R}_{t_i}(\boldsymbol{X}_{t_i-}, \boldsymbol{X}_{t_i})}+\int_0^1\sum_{y\neq \boldsymbol{X}_{t}}(\boldsymbol{R}_{t}-\boldsymbol{R}'_{t})(\boldsymbol{X}_{t}, z)dt
    \end{align}
    \end{small}
    where $t_i\in [0,1]$ denotes the times of insertion events and $t_u\in [0,1]$ denotes the times of unmasking events.
\end{proposition}
\textit{Proof.} The discrete time log RND between the CTMC path measures $\mathbb{P}$ and $\mathbb{P}'$ is defined as:
\begin{align}
    \log \frac{\mathrm{d}\mathbb{P}'}{\mathrm{d}\mathbb{P}}(\boldsymbol{X}_{0:1})&=\log\frac{\mathrm{d}\pi_0'}{\mathrm{d}\pi_0}(\boldsymbol{X}_0)+\sum_{n=0}^{N-1}\log \frac{\mathrm{d}\mathbb{P}'(\boldsymbol{X}_{t_{n+1}}|\boldsymbol{X}_{t_n})}{\mathrm{d}\mathbb{P}(\boldsymbol{X}_{t_{n+1}}|\boldsymbol{X}_{t_n})}+\mathcal{O}(\Delta t)\label{eq:al-rnd-3}
\end{align}
where $\mathbb{P}_0=\pi_0$ and $\mathbb{P}_0'=\pi_0'$ are the initial distributions. Now, we have both unmasking and insertion rates given by $\boldsymbol{Q}_t(\boldsymbol{x},\boldsymbol{y})$, which denotes the rate of unmasking from state $\boldsymbol{x}\to \boldsymbol{y}$, and $\boldsymbol{R}_t(\boldsymbol{x}, \boldsymbol{y})$, which denotes the rate of insertion from state $\boldsymbol{x}\to \boldsymbol{y}$. Therefore, the total probability of a single jump under the joint path measure $\mathbb{P}$ can be decomposed into the probability of remaining in the same state ($\boldsymbol{y}=\boldsymbol{x}$) and the probability of jumping to a different state ($\boldsymbol{y}\neq \boldsymbol{x}$):
\begin{align}
    \mathbb{P}(\boldsymbol{X}_{t_{n+1}}=\boldsymbol{y}|\boldsymbol{X}_{t_n}=\boldsymbol{x})=\begin{cases}
        1-\Delta t\sum_{\boldsymbol{z}\neq \boldsymbol{x}}\left(\boldsymbol{Q}_t(\boldsymbol{x}, \boldsymbol{y})+\boldsymbol{R}_t(\boldsymbol{x}, \boldsymbol{y})\right)+\mathcal{O}(\Delta t^2)&\boldsymbol{y}= \boldsymbol{x}\\
        \Delta t (\boldsymbol{Q}_t(\boldsymbol{x}, \boldsymbol{y})+\boldsymbol{R}_t(\boldsymbol{x}, \boldsymbol{y}))+\mathcal{O}(\Delta t^2)&\boldsymbol{y}\neq \boldsymbol{x}
    \end{cases}
\end{align}
For both cases, we will derive the log ratio of the two CTMCs. When the state remains the same over the interval $[t_n, t_{n+1}]$, the log-ratio expands to:
\begin{align}
    \log \frac{\mathrm{d}\mathbb{P}'(\boldsymbol{X}_{t_{n+1}}|\boldsymbol{X}_{t_n})}{\mathrm{d}\mathbb{P}(\boldsymbol{X}_{t_{n+1}}|\boldsymbol{X}_{t_n})}&=\log \frac{1-\Delta t\sum_{\boldsymbol{z}\neq \boldsymbol{x}}\left(\boldsymbol{Q}'_{t_n}(\boldsymbol{X}_{t_n}, \boldsymbol{X}_{t_{n+1}})+\boldsymbol{R}'_{t_n}(\boldsymbol{X}_{t_n}, \boldsymbol{X}_{t_{n+1}})\right)+\mathcal{O}(\Delta t^2)}{1-\Delta t\sum_{\boldsymbol{z}\neq \boldsymbol{x}}\left(\boldsymbol{Q}_{t_n}(\boldsymbol{X}_{t_n}, \boldsymbol{X}_{t_{n+1}})+\boldsymbol{R}_{t_n}(\boldsymbol{X}_{t_n}, \boldsymbol{X}_{t_{n+1}})\right)+\mathcal{O}(\Delta t^2)}\nonumber\\
    &=\Delta t\sum_{z\neq \boldsymbol{X}_{t_n}}\big(\boldsymbol{Q}_{t_n}(\boldsymbol{X}_{t_n}, \boldsymbol{X}_{t_{n+1}})+\boldsymbol{R}_{t_n}(\boldsymbol{X}_{t_n}, \boldsymbol{X}_{t_{n+1}})\nonumber\\
    &-\boldsymbol{Q}'_{t_n}(\boldsymbol{X}_{t_n}, \boldsymbol{X}_{t_{n+1}})-\boldsymbol{R}'_{t_n}(\boldsymbol{X}_{t_n}, \boldsymbol{X}_{t_{n+1}})\big)+\mathcal{O}(\Delta t^2)\nonumber\\
    &=\Delta t\sum_{z\neq \boldsymbol{X}_{t_n}}\left(\boldsymbol{Q}_{t_n}(\boldsymbol{X}_{t_n}, z)-\boldsymbol{Q}'_{t_n}(\boldsymbol{X}_{t_n}, z)\right)\nonumber\\
    &+\Delta t\sum_{z\neq \boldsymbol{X}_{t_n}}\left(\boldsymbol{R}_{t_n}(\boldsymbol{X}_{t_n}, z)-\boldsymbol{R}'_{t_n}(\boldsymbol{X}_{t_n}, z)\right)+\mathcal{O}(\Delta t^2)\label{eq:al-rnd-1}
\end{align}
When the state changes over $[t_n, t_{n+1}]$, the log-ratio expands to:
\begin{align}
    \log \frac{\mathrm{d}\mathbb{P}'(\boldsymbol{X}_{t_{n+1}}|\boldsymbol{X}_{t_n})}{\mathrm{d}\mathbb{P}(\boldsymbol{X}_{t_{n+1}}|\boldsymbol{X}_{t_n})}&=\log \frac{\Delta t (\boldsymbol{Q}'_{t_n}(\boldsymbol{X}_{t_n}, \boldsymbol{X}_{t_{n+1}})+\boldsymbol{R}'_{t_n}(\boldsymbol{X}_{t_n}, \boldsymbol{X}_{t_{n+1}}))+\mathcal{O}(\Delta t^2)}{\Delta t (\boldsymbol{Q}_{t_n}(\boldsymbol{X}_{t_n}, \boldsymbol{X}_{t_{n+1}})+\boldsymbol{R}_{t_n}(\boldsymbol{X}_{t_n}, \boldsymbol{X}_{t_{n+1}}))+\mathcal{O}(\Delta t^2)}\nonumber\\
    &=\log \frac{\boldsymbol{Q}'_{t_n}(\boldsymbol{X}_{t_n}, \boldsymbol{X}_{t_{n+1}})}{\boldsymbol{Q}_{t_n}(\boldsymbol{X}_{t_n}, \boldsymbol{X}_{t_{n+1}})}+\log \frac{\boldsymbol{R}'_{t_n}(\boldsymbol{X}_{t_n}, \boldsymbol{X}_{t_{n+1}})}{\boldsymbol{R}_{t_n}(\boldsymbol{X}_{t_n}, \boldsymbol{X}_{t_{n+1}})}+\mathcal{O}(\Delta t)\label{eq:al-rnd-2}
\end{align}
Substituting (\ref{eq:al-rnd-1}) and (\ref{eq:al-rnd-2}) into (\ref{eq:al-rnd-3}) and taking the limit as $\Delta t\to 0$, we get:
\begin{small}
\begin{align}
    \log \frac{\mathrm{d}\mathbb{P}'}{\mathrm{d}\mathbb{P}}&(\boldsymbol{X}_{0:1})=\lim_{\Delta t\to 0}\bigg\{\log\frac{\mathrm{d}\pi_0'}{\mathrm{d}\pi_0}(\boldsymbol{X}_0)+\sum_{n=0}^{N-1}\log \frac{\boldsymbol{Q}'_{t_n}(\boldsymbol{X}_{t_n}, \boldsymbol{X}_{t_{n+1}})}{\boldsymbol{Q}_{t_n}(\boldsymbol{X}_{t_n}, \boldsymbol{X}_{t_{n+1}})}+\sum_{n=0}^{N-1}\log \frac{\boldsymbol{R}'_{t_n}(\boldsymbol{X}_{t_n}, \boldsymbol{X}_{t_{n+1}})}{\boldsymbol{R}_{t_n}(\boldsymbol{X}_{t_n}, \boldsymbol{X}_{t_{n+1}})}\nonumber\\
    &+\Delta t\sum_{z\neq \boldsymbol{X}_{t_n}}\left(\boldsymbol{Q}_{t_n}(\boldsymbol{X}_{t_n}, \boldsymbol{X}_{t_{n+1}})-\boldsymbol{Q}'_{t_n}(\boldsymbol{X}_{t_n}, \boldsymbol{X}_{t_{n+1}})\right)\nonumber\\
    &+\Delta t\sum_{z\neq \boldsymbol{X}_{t_n}}\left(\boldsymbol{R}_{t_n}(\boldsymbol{X}_{t_n}, \boldsymbol{X}_{t_{n+1}})-\boldsymbol{R}'_{t_n}(\boldsymbol{X}_{t_n}, \boldsymbol{X}_{t_{n+1}})\right)+\mathcal{O}(\Delta t)\bigg\}\nonumber\\
    &=\log\frac{\mathrm{d}\pi_0'}{\mathrm{d}\pi_0}(\boldsymbol{X}_0)+\sum_{t_u:\boldsymbol{X}_{t_u}\neq  \boldsymbol{X}_{t_u-}}\log \frac{\boldsymbol{Q}'_{t_u}(\boldsymbol{X}_{t_u-}, \boldsymbol{X}_{t_u})}{\boldsymbol{Q}_{t_u}(\boldsymbol{X}_{t_u-}, \boldsymbol{X}_{t_u})}+\int_0^1\sum_{z\neq \boldsymbol{X}_{t}}(\boldsymbol{Q}_{t}-\boldsymbol{Q}'_{t})(\boldsymbol{X}_{t}, z)dt\nonumber\\
    &+\sum_{t_i: \boldsymbol{X}_{t_i}\neq \boldsymbol{X}_{t_i-}}\log \frac{\boldsymbol{R}'_{t_i}(\boldsymbol{X}_{t_i-}, \boldsymbol{X}_{t_i})}{\boldsymbol{R}_{t_i}(\boldsymbol{X}_{t_i-}, \boldsymbol{X}_{t_i})}+\int_0^1\sum_{y\neq \boldsymbol{X}_{t}}(\boldsymbol{R}_{t}-\boldsymbol{R}'_{t})(\boldsymbol{X}_{t}, z)dt
\end{align}
\end{small}
which concludes our proof. \hfill $\square$

\begin{corollary}[KL Divergence Between Any-Length Joint CTMCs]
\label{app-cor:jointkl}
    The KL divergence between two joint CTMCs $\mathbb{P}, \mathbb{P}'$ with unmasking rates $\boldsymbol{Q}, \boldsymbol{Q}'$ and insertion rates $\boldsymbol{R}, \boldsymbol{R}'$ is given by:
    \begin{small}
    \begin{align}
        D_{\text{KL}}(\mathbb{P}'\|\mathbb{P})&=D_{\text{KL}}(\mathbb{P}_0'\|\mathbb{P}_0)\nonumber\\
        &+\mathbb{E}_{\boldsymbol{X}_{0:1}\sim \mathbb{P}'}\bigg[\sum_{t_u:\boldsymbol{X}_{t_u}\neq  \boldsymbol{X}_{t_u-}}\log \frac{\boldsymbol{Q}'_{t_u}(\boldsymbol{X}_{t_u^-}, \boldsymbol{X}_{t_u})}{\boldsymbol{Q}_{t_u}(\boldsymbol{X}_{t_u^-}, \boldsymbol{X}_{t_u})}+\int_0^1\sum_{z\neq \boldsymbol{X}_{t}}(\boldsymbol{Q}_{t}-\boldsymbol{Q}'_{t})(\boldsymbol{X}_{t}, z)dt\nonumber\\
        &+\sum_{t_i: \boldsymbol{X}_{t_i}\neq \boldsymbol{X}_{t_i-}}\log \frac{\boldsymbol{R}'_{t_i}(\boldsymbol{X}_{t_i^-}, \boldsymbol{X}_{t_i})}{\boldsymbol{R}_{t_i}(\boldsymbol{X}_{t_i^-}, \boldsymbol{X}_{t_i})}+\int_0^1\sum_{y\neq \boldsymbol{X}_{t}}(\boldsymbol{R}_{t}-\boldsymbol{R}'_{t})(\boldsymbol{X}_{t}, y)dt\bigg]
    \end{align}
    \end{small}
\end{corollary}

\textit{Proof.} The KL divergence is defined as the expectation of the log RND derived in Lemma \ref{app-prop:RND} as:
\begin{small}
\begin{align}
    D_{\text{KL}}(\mathbb{P}'\|\mathbb{P})&=\mathbb{E}_{\boldsymbol{X}_{0:1}\sim \mathbb{P}'}\left[\log \frac{\mathrm{d}\mathbb{P}'}{\mathrm{d}\mathbb{P}}(\boldsymbol{X}_{0:1})\right]\nonumber\\
    &=\mathbb{E}_{\boldsymbol{X}_{0:1}\sim \mathbb{P}'}\left[\log\frac{\mathrm{d}\pi_0'}{\mathrm{d}\pi_0}(\boldsymbol{X}_0)\right]\nonumber\\
        &+\mathbb{E}_{\boldsymbol{X}_{0:1}\sim \mathbb{P}'}\left[\sum_{t_u:\boldsymbol{X}_{t_u}\neq  \boldsymbol{X}_{t_u-}}\log \frac{\boldsymbol{Q}'_{t_u}(\boldsymbol{X}_{t_u-}, \boldsymbol{X}_{t_u})}{\boldsymbol{Q}_{t_u}(\boldsymbol{X}_{t_u-}, \boldsymbol{X}_{t_u})}+\int_0^1\sum_{z\neq \boldsymbol{X}_{t}}(\boldsymbol{Q}_{t}-\boldsymbol{Q}'_{t})(\boldsymbol{X}_{t}, z)dt\right]\nonumber\\
        &+\mathbb{E}_{\boldsymbol{X}_{0:1}\sim \mathbb{P}'}\left[\sum_{t_i: \boldsymbol{X}_{t_i}\neq \boldsymbol{X}_{t_i-}}\log \frac{\boldsymbol{R}'_{t_i}(\boldsymbol{X}_{t_i-}, \boldsymbol{X}_{t_i})}{\boldsymbol{R}_{t_i}(\boldsymbol{X}_{t_i-}, \boldsymbol{X}_{t_i})}+\int_0^1\sum_{y\neq \boldsymbol{X}_{t}}(\boldsymbol{R}_{t}-\boldsymbol{R}'_{t})(\boldsymbol{X}_{t}, z)dt\right]
\end{align}
\end{small}
We can decompose the first term as:
\begin{align}
    \mathbb{E}_{\boldsymbol{X}_{0:1}\sim \mathbb{P}'}\bigg[\log\frac{\mathrm{d}\pi_0'}{\mathrm{d}\pi_0}(\boldsymbol{X}_0)\bigg]&=\mathbb{E}_{\boldsymbol{X}_0\sim \pi_0'}\bigg[\log\frac{\mathrm{d}\pi_0'}{\mathrm{d}\pi_0}(\boldsymbol{X}_0)\bigg]=D_{\text{KL}}(\pi'_0\|\pi_0)
\end{align}

The second term can be decomposed as:
\begin{small}
\begin{align}
    &\mathbb{E}_{\boldsymbol{X}_{0:1}\sim \mathbb{P}'}\left[\sum_{t_u:\boldsymbol{X}_{t_u}\neq  \boldsymbol{X}_{t_u-}}\log \frac{\boldsymbol{Q}'_{t_u}(\boldsymbol{X}_{t_u-}, \boldsymbol{X}_{t_u})}{\boldsymbol{Q}_{t_u}(\boldsymbol{X}_{t_u-}, \boldsymbol{X}_{t_u})}+\int_0^1\sum_{z\neq \boldsymbol{X}_{t}}(\boldsymbol{Q}_{t}-\boldsymbol{Q}'_{t})(\boldsymbol{X}_{t}, z)dt\right]\label{eq:step1}\\
    &=\mathbb{E}_{\boldsymbol{X}_{0:1}\sim \mathbb{P}'}\left[\sum_{t_u:\boldsymbol{X}_{t_u}\neq  \boldsymbol{X}_{t_u-}}\log \frac{\boldsymbol{Q}'_{t_u}(\boldsymbol{X}_{t_u-}, \boldsymbol{X}_{t_u})}{\boldsymbol{Q}_{t_u}(\boldsymbol{X}_{t_u-}, \boldsymbol{X}_{t_u})}\right]+\mathbb{E}_{\boldsymbol{X}_{0:1}\sim \mathbb{P}'}\left[\int_0^1\sum_{z\neq \boldsymbol{X}_{t}}(\boldsymbol{Q}_{t}-\boldsymbol{Q}'_{t})(\boldsymbol{X}_{t}, z)dt\right]
\end{align}
\end{small}
For the first component, we can consider the time discretization and take the limit for $\Delta t\to 0$ to get:
\begin{align}
    \mathbb{E}_{\boldsymbol{X}_{0:1}\sim \mathbb{P}'}&\bigg[\sum_{k=0}^{K-1}\boldsymbol{1}[\boldsymbol{X}_{t_{k+1}}\neq \boldsymbol{X}_{t_k}]\log \frac{\boldsymbol{Q}'_{t_u}(\boldsymbol{X}_{t_k}, \boldsymbol{X}_{t_{k+1}})}{\boldsymbol{Q}_{t_u}(\boldsymbol{X}_{t_k}, \boldsymbol{X}_{t_{k+1}})}\bigg]\nonumber\\
    &=\sum_{k=0}^{K-1}\mathbb{E}_{\mathbb{P}'(\boldsymbol{X}_{t_k}), \mathbb{P}'(\boldsymbol{X}_{t_{k+1}}|\boldsymbol{X}_{t_k})}\bigg[\boldsymbol{1}[\boldsymbol{X}_{t_{k+1}}\neq \boldsymbol{X}_{t_k}]\log \frac{\boldsymbol{Q}'_{t_u}(\boldsymbol{X}_{t_k}, \boldsymbol{X}_{t_{k+1}})}{\boldsymbol{Q}_{t_u}(\boldsymbol{X}_{t_k}, \boldsymbol{X}_{t_{k+1}})}\bigg]\nonumber\\
    &=\sum_{k=0}^{K-1}\mathbb{E}_{\mathbb{P}'(\boldsymbol{X}_{t_k})}\sum_{z\neq \boldsymbol{X}_{t_k}}\left[\mathbb{P}'(z|\boldsymbol{X}_{t_k})\log \frac{\boldsymbol{Q}'_{t_u}(\boldsymbol{X}_{t_k}, z)}{\boldsymbol{Q}_{t_u}(\boldsymbol{X}_{t_k}, z)}\right]\nonumber\\
    &=\sum_{k=0}^{K-1}\mathbb{E}_{\mathbb{P}'(\boldsymbol{X}_{t_k})}\sum_{z\neq \boldsymbol{X}_{t_k}}\left[\Delta t \boldsymbol{Q}_{t_k}'(\boldsymbol{X}_{t_{k}},z)\log \frac{\boldsymbol{Q}'_{t_k}(\boldsymbol{X}_{t_k}, z)}{\boldsymbol{Q}_{t_k}(\boldsymbol{X}_{t_k}, z)}+\mathcal{O}(\Delta t^2)\right]\nonumber\\
    &\underset{\Delta t \to 0}{=}\mathbb{E}_{\boldsymbol{X}_{0:1}\sim\mathbb{P}'}\left[\int_0^1\sum_{z\neq \boldsymbol{X}_{t}}\boldsymbol{Q}'_{t}\log \frac{\boldsymbol{Q}'_{t}}{\boldsymbol{Q}_{t}}(\boldsymbol{X}_t,z)dt\right]
\end{align}
Then, summing with the second integral, we get: 
\begin{align}
    \mathbb{E}_{\boldsymbol{X}_{0:1}\sim\mathbb{P}'}&\left[\int_0^1\sum_{z\neq \boldsymbol{X}_{t}}\boldsymbol{Q}'_{t}\log \frac{\boldsymbol{Q}'_{t}}{\boldsymbol{Q}_{t}}(\boldsymbol{X}_t,z)dt\right]+\mathbb{E}_{\boldsymbol{X}_{0:1}\sim \mathbb{P}'}\bigg[\int_0^1\sum_{z\neq \boldsymbol{X}_{t}}(\boldsymbol{Q}_{t}-\boldsymbol{Q}'_{t})(\boldsymbol{X}_{t}, z)dt\bigg]\nonumber\\
    &=\mathbb{E}_{\boldsymbol{X}_{0:1}\sim\mathbb{P}'}\int_0^1\sum_{z\neq \boldsymbol{X}_{t}}\bigg[\left(\boldsymbol{Q}'_{t}\log \frac{\boldsymbol{Q}'_{t}}{\boldsymbol{Q}_{t}}+\boldsymbol{Q}_{t}-\boldsymbol{Q}'_{t}\right)(\boldsymbol{X}_{t}, z)\bigg]dt\label{eq:steplast}
\end{align}
Applying the same procedure from lines (\ref{eq:step1}) to (\ref{eq:steplast}) to the insertion rates, we get:
\begin{align}
    \mathbb{E}_{\boldsymbol{X}_{0:1}\sim \mathbb{P}'}&\left[\sum_{t_i: \boldsymbol{X}_{t_i}\neq \boldsymbol{X}_{t_i-}}\log \frac{\boldsymbol{R}'_{t_i}(\boldsymbol{X}_{t_i-}, \boldsymbol{X}_{t_i})}{\boldsymbol{R}_{t_i}(\boldsymbol{X}_{t_i-}, \boldsymbol{X}_{t_i})}+\int_0^1\sum_{y\neq \boldsymbol{X}_{t}}(\boldsymbol{R}_{t}-\boldsymbol{R}'_{t})(\boldsymbol{X}_{t}, z)dt\right]\nonumber\\
    &=\mathbb{E}_{\boldsymbol{X}_{0:1}\sim\mathbb{P}'}\int_0^1\sum_{z\neq \boldsymbol{X}_{t}}\left[\left(\boldsymbol{R}'_{t}\log \frac{\boldsymbol{R}'_{t}}{\boldsymbol{R}_{t}}+\boldsymbol{R}_{t}-\boldsymbol{R}'_{t}\right)(\boldsymbol{X}_{t}, z)\right]dt
\end{align}
Putting these terms together, we get that the KL divergence between the joint CTMCs $\mathbb{P}$ and $\mathbb{P}'$ is given by:
\begin{small}
\begin{align}
    D_{\text{KL}}(\mathbb{P}'\|\mathbb{P})&=D_{\text{KL}}(\mathbb{P}_0'\|\mathbb{P}_0)\nonumber\\
    &+\mathbb{E}_{\boldsymbol{X}_{0:1}\sim \mathbb{P}'}\bigg[\sum_{t_u:\boldsymbol{X}_{t_u}\neq  \boldsymbol{X}_{t_u-}}\log \frac{\boldsymbol{Q}'_{t_u}(\boldsymbol{X}_{t_u^-}, \boldsymbol{X}_{t_u})}{\boldsymbol{Q}_{t_u}(\boldsymbol{X}_{t_u^-}, \boldsymbol{X}_{t_u})}+\int_0^1\sum_{z\neq \boldsymbol{X}_{t}}(\boldsymbol{Q}_{t}-\boldsymbol{Q}'_{t})(\boldsymbol{X}_{t}, z)dt\nonumber\\
    &+\sum_{t_i: \boldsymbol{X}_{t_i}\neq \boldsymbol{X}_{t_i-}}\log \frac{\boldsymbol{R}'_{t_i}(\boldsymbol{X}_{t_i^-}, \boldsymbol{X}_{t_i})}{\boldsymbol{R}_{t_i}(\boldsymbol{X}_{t_i^-}, \boldsymbol{X}_{t_i})}+\int_0^1\sum_{y\neq \boldsymbol{X}_{t}}(\boldsymbol{R}_{t}-\boldsymbol{R}'_{t})(\boldsymbol{X}_{t}, y)dt\bigg]
\end{align}
\end{small}
which concludes our proof. \hfill $\square$

\subsection{Adaptive Joint Decoding Loss}
\label{app:ajd-proofs}
In this section, we provide theoretical justifications for our \textbf{Adaptive Joint Decoding (AJD)} loss, which can be used to jointly fine-tune the insertion and unmasking policies, as well as an adaptive inference schedule, to optimize sampling from the reward-tilted distribution. Throughout the proofs, we denote the full any-length generator as $\boldsymbol{A}_t:=\boldsymbol{Q}_t+\boldsymbol{R}_t$ where $\boldsymbol{Q}_t$ denotes the unmasking rate and $\boldsymbol{R}_t$ denotes the insertion rate. 

We first derive the form of the optimal any-length generator $\boldsymbol{A}^\star$, which follows seamlessly from the theory of fixed-length MDM generators \citep{zhu2025mdns, tang2025tr2} but is defined on a larger state space of any-length sequences $\boldsymbol{x}\in \mathcal{X}$. Then, we show that our Adaptive Joint Decoding loss provably converges to the optimal generator.

For an intermediate state $\boldsymbol{x}\in \mathcal{X}$, we define the \textit{total cost} $J_t(\boldsymbol{x}, v)$ of insertions and unmasking steps required to reach a final state $\boldsymbol{X}_1$ under a tilted path measure $\mathbb{P}^v$. Given a terminal reward $r(\boldsymbol{X}_1)$, we define the cost-minimization objective as:
\begin{align}
    J_t(\boldsymbol{x}, v)=\underset{\boldsymbol{X}_{0:1}\sim \mathbb{P}^v}{\mathbb{E}}\left[\int_0^1\sum_{y\neq \boldsymbol{X}_s}C_s(\boldsymbol{X}_s, y)ds-r(\boldsymbol{X}_1)\bigg|\boldsymbol{X}_t=\boldsymbol{x}\right]
\end{align}
where the cost is defined as the KL divergence between generators derived in (\ref{app-cor:jointkl}) as $C_s(x, y):=\left(\boldsymbol{A}^v_t\log \frac{\boldsymbol{A}^v}{\boldsymbol{A}^0}-\boldsymbol{A}^v+\boldsymbol{A}^0\right)(x, y)$. We can minimize $J_t(\boldsymbol{x}, v)$ by maximizing a \textbf{value function} defined as the \textit{negative optimal cost-to-go} $V_t(\boldsymbol{x}):=-J^\star_t(\boldsymbol{x})=-\inf_vJ_t(\boldsymbol{x}, v)$. Using the recursive nature of the cost-to-go, where the total cost is equal to the immediate and future cost, we can decompose the value function using Bellman's principle as:
\begin{small}
\begin{align}
    -V_t(\boldsymbol{x})=J^\star_t(\boldsymbol{x})&=\inf_v\underset{\boldsymbol{X}_{0:1}\sim \mathbb{P}^v}{\mathbb{E}}\left[\left(\int_t^{t+\Delta t}+\int_{t+\Delta t}^1\right)\sum_{y\neq \boldsymbol{X}_s}C_s(\boldsymbol{X}_s, y)ds-r(\boldsymbol{X}_1)\bigg|\boldsymbol{X}_t=\boldsymbol{x}\right]\nonumber\\
    &=\left[\Delta t\inf_v\sum_{y\neq x}C_s(\boldsymbol{x}, \boldsymbol{y})+\mathcal{O}(\Delta t^2)\right]+\inf_v\underset{\boldsymbol{X}_{0:1}\sim \mathbb{P}^v}{\mathbb{E}}[-V_{t+\Delta t}(\boldsymbol{X}_{t+\Delta t})|\boldsymbol{X}_{t}=\boldsymbol{x}]\label{eq:expanded-value}
\end{align}
\end{small}
Now, we derive the full form of the optimal any-length generator $\boldsymbol{A}^\star$ using our definition of the value function.

\begin{lemma}[Optimal Generator]\label{lemma:optimal-generator}
    Given a reference any-length generator $\boldsymbol{A}^0:=\boldsymbol{Q}^0+\boldsymbol{R}^0$ a value function $V_t$, the optimal generator takes the form:
    \begin{align}
        \boldsymbol{A}^\star_t(\boldsymbol{x}, \boldsymbol{y})&=\boldsymbol{A}^0_t(\boldsymbol{x}, \boldsymbol{y})\exp(V_t(\boldsymbol{y})-V_t(\boldsymbol{x}))\nonumber\\
        &=\boldsymbol{Q}^0_t(\boldsymbol{x}, \boldsymbol{y}_u)\exp(V_t(\boldsymbol{y}_u)-V_t(\boldsymbol{x}))+\boldsymbol{R}^0_t(\boldsymbol{x}, \boldsymbol{y}^i)\exp(V_t(\boldsymbol{y}^i)-V_t(\boldsymbol{x}))
    \end{align}
    where $\boldsymbol{y}_u$ denotes the state after an unmasking transition and $\boldsymbol{y}_i$ denotes the state after an insertion transition
\end{lemma}
\textit{Proof.} Starting with the expanded value function in (\ref{eq:expanded-value}), we expand the recursive term by defining the next state $\boldsymbol{X}_{t+\Delta t}:=\boldsymbol{y}$ and applying the CTMC transition probability as:
\begin{align}
    \inf_v&\underset{\boldsymbol{X}_{0:1}\sim \mathbb{P}^v}{\mathbb{E}}[-V_{t+\Delta t}(\boldsymbol{X}_{t+\Delta t})|\boldsymbol{X}_{t}=\boldsymbol{x}]\nonumber\\
    &=\inf_v\left[-\sum_{\boldsymbol{y}}V_{t+\Delta t}(\boldsymbol{y})\left(\boldsymbol{1}_{x=y}+\Delta t\boldsymbol{A}^v_t(\boldsymbol{x}, \boldsymbol{y})+\mathcal{O}(\Delta t^2)\right)\right]\nonumber\\
    &=\inf_v\left[-V_{t+\Delta t}(\boldsymbol{x})-\Delta t\sum_{x\neq y}V_{t+\Delta t}(\boldsymbol{y})\boldsymbol{A}^v_t(\boldsymbol{x}, \boldsymbol{y})+\Delta t\sum_{x\neq y}V_{t+\Delta t}(\boldsymbol{x})\boldsymbol{A}^v_t(\boldsymbol{x}, \boldsymbol{y})+\mathcal{O}(\Delta t^2)\right]\nonumber\\
    &=-V_{t+\Delta t}(\boldsymbol{x})+\Delta t\inf_u\left[\sum_{x\neq y}\boldsymbol{A}^v_t(\boldsymbol{x}, \boldsymbol{y})\left(V_{t+\Delta t}(\boldsymbol{x})-V_{t+\Delta t}(\boldsymbol{y})\right)\right]+\mathcal{O}(\Delta t^2)\label{eq:proof-1}
]\end{align}
Substituting (\ref{eq:proof-1}) back into the Bellman recursion in (\ref{eq:expanded-value}), we get:
\begin{small}
\begin{align}
    -V_t(\boldsymbol{x})&=\left[\Delta t\inf_v\sum_{y\neq x}C_s(\boldsymbol{x}, \boldsymbol{y})+\mathcal{O}(\Delta t^2)\right]\nonumber\\
    &-V_{t+\Delta t}(\boldsymbol{x})+\Delta t\inf_u\left[\sum_{x\neq y}\boldsymbol{A}^v_t(\boldsymbol{x}, \boldsymbol{y})\left(V_{t+\Delta t}(\boldsymbol{x})-V_{t+\Delta t}(\boldsymbol{y})\right)\right]+\mathcal{O}(\Delta t^2)\nonumber\\
    V_{t+\Delta t}(\boldsymbol{x})-V_t(\boldsymbol{x})&=\Delta t\inf_v\sum_{y\neq x}\left[C_s(\boldsymbol{x}, \boldsymbol{y})+\boldsymbol{A}^v_t(\boldsymbol{x}, \boldsymbol{y})\left(V_{t+\Delta t}(\boldsymbol{x})-V_{t+\Delta t}(\boldsymbol{y})\right)\right]
\end{align}
\end{small}
Then, dividing by $\Delta t$, taking the limit $\Delta t\to 0$, and substituting $C_s(x, y)=\left(\boldsymbol{A}^v_t\log \frac{\boldsymbol{A}^v}{\boldsymbol{A}^0}-\boldsymbol{A}^v+\boldsymbol{A}^0\right)(x, y)$, we get:
\begin{align}
    \partial_tV_t(\boldsymbol{x})&=\lim_{\Delta t\to 0}\frac{V_{t+\Delta t}(\boldsymbol{x})-V_t(\boldsymbol{x})}{\Delta t}=\inf_v\sum_{y\neq x}\left[C_s(\boldsymbol{x}, \boldsymbol{y})+\boldsymbol{A}^v_t(\boldsymbol{x}, \boldsymbol{y})\left(V_t(\boldsymbol{x})-V_t(\boldsymbol{y})\right)\right]\\
    &=\inf_v\sum_{y\neq x}\left[\left(\boldsymbol{A}^v_t\log \frac{\boldsymbol{A}^v}{\boldsymbol{A}^0}-\boldsymbol{A}^v+\boldsymbol{A}^0\right)(\boldsymbol{x}, \boldsymbol{y})+\boldsymbol{A}^v_t(\boldsymbol{x}, \boldsymbol{y})\left(V_t(\boldsymbol{x})-V_t(\boldsymbol{y})\right)\right]\label{eq:partial-value}
\end{align}
The infimum is achieved by minimizing the following for every pair $(\boldsymbol{x}, \boldsymbol{y})$:
\begin{align}
    f(\boldsymbol{A}^v)&=\boldsymbol{A}^v_t\log \frac{\boldsymbol{A}^v}{\boldsymbol{A}^0}-\boldsymbol{A}^v+\boldsymbol{A}^0+\boldsymbol{A}^v_t\left(V_t(\boldsymbol{x})-V_t(\boldsymbol{y})\right)\\
    f'(\boldsymbol{A}^v)&=\log \frac{\boldsymbol{A}^v}{\boldsymbol{A}^0}+\left(V_t(\boldsymbol{x})-V_t(\boldsymbol{y})\right)
\end{align}
The minimizer $\boldsymbol{A}^\star$ satisfying $f'(\boldsymbol{A}^v)=0$ is defined as:
\begin{align}
    \log \frac{\boldsymbol{A}^\star}{\boldsymbol{A}^0}=V_t(\boldsymbol{y})-V_t(\boldsymbol{x})\implies \boldsymbol{A}^\star_t(\boldsymbol{x}, \boldsymbol{y})=\boldsymbol{A}^0_t(\boldsymbol{x}, \boldsymbol{y})\exp\left(V_t(\boldsymbol{y})-V_t(\boldsymbol{x})\right)
\end{align}
Since the total generator factorizes into the sum of insertion and unmasking rates, and rates do not overlap as defined in App \ref{app:observations}, we have: 
\begin{small}
\begin{align}
    \boldsymbol{A}^\star_t(\boldsymbol{x}, \boldsymbol{y})&=\boldsymbol{Q}^\star_t(\boldsymbol{x}, \boldsymbol{y})+\boldsymbol{R}^\star_t(\boldsymbol{x}, \boldsymbol{y})\quad \text{s.t.}\quad \begin{cases}
        \boldsymbol{Q}^\star_t(\boldsymbol{x}, \boldsymbol{y})=\boldsymbol{Q}^0_t(\boldsymbol{x}, \boldsymbol{y})\exp\left(V_t(\boldsymbol{y})-V_t(\boldsymbol{x})\right)\\
        \boldsymbol{R}^\star_t(\boldsymbol{x}, \boldsymbol{y})=\boldsymbol{R}^0_t(\boldsymbol{x}, \boldsymbol{y})\exp\left(V_t(\boldsymbol{y})-V_t(\boldsymbol{x})\right)
    \end{cases}
\end{align}
\end{small}
which proves our result. \hfill $\square$

Now, we will derive the form of the optimal path measure $\mathbb{P}^\star$ corresponding to the optimal generator $\boldsymbol{A}^\star$.

\begin{lemma}[Optimal Path Measure]\label{lemma:optimal-path-measure}
    The optimal path measure $\mathbb{P}^\star$ corresponding to the value function $V_t(\boldsymbol{x})$ is defined as:
    \begin{align}
        \mathbb{P}^\star_t(\boldsymbol{x})=\frac{1}{Z}\mathbb{P}^0_t(\boldsymbol{x})\mathrm{e}^{V_t(\boldsymbol{x})}, \quad Z:=\mathbb{E}_{\boldsymbol{x}\sim \mathbb{P}^0_1}\left[\mathrm{e}^{r(\boldsymbol{x})/\alpha}\right]
    \end{align}
    for continuous $t\mapsto\boldsymbol{A}^0_t$. 
\end{lemma}

\textit{Proof.} We define $h_t(\boldsymbol{x}):=\frac{1}{Z}\mathbb{P}^0_t(\boldsymbol{x})\mathrm{e}^{V_t(\boldsymbol{x})}$, which satisfies $h_1=\mathbb{P}^\star_1$ by definition. To show that $\mathbb{P}^\star$ is the path measure generated by $\boldsymbol{A}^\star$, it suffices to show that $h_t(\boldsymbol{x})$ satisfies the Kolmogorov forward equation for $\boldsymbol{A}^\star_t$. First, the Kolmogorov forward equation for the reference generator $\boldsymbol{A}^0_t$ is given by:
\begin{align}
    \partial_t\mathbb{P}^0_t(\boldsymbol{x})=\sum_{y\neq x}\left(\boldsymbol{A}^0_t(\boldsymbol{y}, \boldsymbol{x})\mathbb{P}^0_t(\boldsymbol{y})-\boldsymbol{A}^0_t(\boldsymbol{x}, \boldsymbol{y})\mathbb{P}^0_t(\boldsymbol{x})\right)
\end{align}
Then, from Lemma \ref{lemma:optimal-generator}, we have $\boldsymbol{A}^\star_t(\boldsymbol{x}, \boldsymbol{y})=\boldsymbol{A}^0_t(\boldsymbol{x}, \boldsymbol{y})\mathrm{e}^{V_t(\boldsymbol{y})-V_t(\boldsymbol{x})}$ which can be substituted into (\ref{eq:partial-value}) to get:
\begin{align}
    \partial_tV_t(\boldsymbol{x})&=\sum_{y\neq x}\boldsymbol{A}^0_t(\boldsymbol{x}, \boldsymbol{y})\left(\mathrm{e}^{V_t(\boldsymbol{x})}-\mathrm{e}^{V_t(\boldsymbol{y})}\right)\label{eq:hjb}
\end{align}
Taking the partial derivative of $h_t(\boldsymbol{x})$ and substituting (\ref{eq:hjb}), we get:
\begin{align}
    &\partial_t h_t(\boldsymbol{x})=\frac{1}{Z}\left[\partial _t\mathbb{P}^0_t(\boldsymbol{x})\mathrm{e}^{V_t(\boldsymbol{x})}+\mathbb{P}^0_t\partial_t\mathrm{e}^{V_t(\boldsymbol{x})}\right]\nonumber\\
    &=\frac{1}{Z}\left[\mathrm{e}^{V_t(\boldsymbol{x})}\sum_{y\neq x}\left(\boldsymbol{A}^0_t(\boldsymbol{y}, \boldsymbol{x})\mathbb{P}^0_t(y)-\boldsymbol{A}^0_t(\boldsymbol{x}, y)\mathbb{P}^0_t(\boldsymbol{x})\right)+\mathbb{P}^0_t(\boldsymbol{x})\sum_{y\neq x}\boldsymbol{A}^0_t(\boldsymbol{x}, \boldsymbol{y})\left(\mathrm{e}^{V_t(\boldsymbol{x})}-\mathrm{e}^{V_t(\boldsymbol{y})}\right)\right]\nonumber\\
    &=\sum_{y\neq x}\left(\boldsymbol{A}^0_t(\boldsymbol{y}, \boldsymbol{x})\frac{1}{Z}\mathbb{P}^0_t(\boldsymbol{y})\mathrm{e}^{V_t(\boldsymbol{x})}-\boldsymbol{A}^0_t(\boldsymbol{x}, \boldsymbol{y})\frac{1}{Z}\mathbb{P}^0_t(\boldsymbol{x})\mathrm{e}^{V_t(\boldsymbol{y})}\right)\nonumber\\
    &=\sum_{y\neq x}\left(\boldsymbol{A}^0_t(\boldsymbol{y}, \boldsymbol{x})\mathrm{e}^{V_t(\boldsymbol{x})-V_t(\boldsymbol{y})}h_t(\boldsymbol{y})-\boldsymbol{A}^0_t(\boldsymbol{x}, \boldsymbol{y})\mathrm{e}^{V_t(\boldsymbol{x})-V_t(\boldsymbol{y})}h_t(\boldsymbol{x})\right)\nonumber\\
    &\overset{(\ref{eq:hjb})}{=}\sum_{y\neq x}\left(\boldsymbol{A}^\star_t(\boldsymbol{y}, \boldsymbol{x})h_t(\boldsymbol{y})-\boldsymbol{A}^\star_t(\boldsymbol{x}, \boldsymbol{y})h_t(\boldsymbol{x})\right)\nonumber
\end{align}

The final equality is exactly the Kolmogorov forward equation of the optimal any-length generator $\boldsymbol{A}^\star_t$ for the tilted path measure $\mathbb{P}^\star$. Since the Kolmogorov forward equation yields a unique solution for all $t$, we have shown that $\mathbb{P}_t^\star=\frac{1}{Z}\mathbb{P}^0_t(\boldsymbol{x})\mathrm{e}^{V_t(\boldsymbol{x})}$. \hfill $\square$

Finally, we can derive the RND between the optimal any-length path measure $\mathbb{P}^\star$ with respect to the base measure $\mathbb{P}^0$ in the following Lemma.

\begin{lemma}[Radon-Nikodym Derivative Between Optimal and Base Path Measures]\label{lemma:rnd-optimal}
    The RND between the optimal path measure $\mathbb{P}^\star$ and its generator $\boldsymbol{Q}^\star$ and the reference path measure $\mathbb{P}^0$ and generator $\boldsymbol{Q}^0$ over any trajectory $\boldsymbol{X}_{0:1}$ can be expressed as:
    \begin{align}
        \frac{\mathrm{d}\mathbb{P}^\star}{\mathrm{d}\mathbb{P}^0}(\boldsymbol{X}_{0:1})=\frac{1}{Z}e^{r(\boldsymbol{X}_1)/\alpha}, \quad \text{where}\quad Z:=\mathbb{E}_{\boldsymbol{X}_1\sim\mathbb{P}^0_1}[e^{r(\boldsymbol{X}_1)/\alpha}]
    \end{align}
\end{lemma}

\textit{Proof.} From Prop \ref{app-prop:RND}, the RND between the path measures is defined as:
\begin{small}
\begin{align}
    \log\frac{\mathrm{d}\mathbb{P}^\star}{\mathrm{d}\mathbb{P}^0}(\boldsymbol{X}_{0:1})&=\log\frac{\mathrm{d}\mathbb{P}_0^\star}{\mathrm{d}\mathbb{P}^0_0}(\boldsymbol{X}_0)+\sum_{t:\boldsymbol{X}_t\neq  \boldsymbol{X}_{t-}}\log \frac{\boldsymbol{A}^\star_{t}(\boldsymbol{X}_{t_u-}, \boldsymbol{X}_{t})}{\boldsymbol{A}^0_{t}(\boldsymbol{X}_{t^-}, \boldsymbol{X}_{t})}+\int_0^1\sum_{z\neq \boldsymbol{X}_{t}}(\boldsymbol{A}^0_{t}-\boldsymbol{A}^\star_{t})(\boldsymbol{X}_{t}, z)dt\nonumber
\end{align}
\end{small}
Then, applying Lemmas \ref{lemma:optimal-generator} and \ref{lemma:optimal-path-measure}, we get: 
\begin{small}
\begin{align}
    \log\frac{\mathrm{d}\mathbb{P}^\star}{\mathrm{d}\mathbb{P}^0}(\boldsymbol{X}_{0:1})=&V_0(\boldsymbol{X}_0)-\log Z+\sum_{t:\boldsymbol{X}_t\neq \boldsymbol{X}_{t-}}(V_t(\boldsymbol{X}_t)-V_t(\boldsymbol{X}_{t^-}))\nonumber\\
    &+\int_0^1\sum_{y\neq\boldsymbol{X}_t}\boldsymbol{A}_t^0(\boldsymbol{X}_t, y)\left(1-\mathrm{e}^{V_t(\boldsymbol{y})-V_t(\boldsymbol{X}_t)}\right)dt\label{eq:proof-2}
\end{align}
\end{small}
Now, we derive an expression for $V_0(\boldsymbol{X}_0)$ with respect to $V_1(\boldsymbol{X}_1)=r(\boldsymbol{X}_1)$ by recognizing that the CTMC is a piecewise càdlàg function where each discrete step $t_k\to t_{k+1}$ can be categorized as a time evolution at a fixed state $\boldsymbol{X}_{t_k}$ or a jump from state $\boldsymbol{X}_{t_k}$ to $\boldsymbol{X}_{t_k^-}\neq \boldsymbol{X}_{t_k}$. Given this, we decompose the value difference as:
\begin{small}
\begin{align}
    V_1(\boldsymbol{X}_1)-V_0(\boldsymbol{X}_0)&=\sum_{k=0}^{K-1}(V_{t_{k+1}}(\boldsymbol{X}_{t_k})-V_{t_k}(\boldsymbol{X}_{t_k}))+\sum_{k=1}^{K-1}(V_{t_k}(\boldsymbol{X}_{t_k})-V_{t_k}(\boldsymbol{X}_{t_{k-1}}))\nonumber\\
    &=\sum_{k=0}^{K-1}\int_{t_k}^{t_{k+1}}\partial_tV_t(\boldsymbol{X}_{t_k})dt+\sum_{t: \boldsymbol{X}_t\neq \boldsymbol{X}_{t-}}(V_t(\boldsymbol{X}_t)-V_t(\boldsymbol{X}_{t^-}))\nonumber\\
    &=\int_0^1\partial_tV_t(\boldsymbol{X}_t)dt +\sum_{t: \boldsymbol{X}_{t}\neq \boldsymbol{X}_{t-}}(V_t(\boldsymbol{X}_t)-V_t(\boldsymbol{X}_{t^-}))\\
    V_0(\boldsymbol{X}_0)=V_1(\boldsymbol{X}_1)&-\int_0^1\partial_tV_t(\boldsymbol{X}_t)dt -\sum_{t: \boldsymbol{X}_{t}\neq \boldsymbol{X}_{t-}}(V_t(\boldsymbol{X}_t)-V_t(\boldsymbol{X}_{t^-}))\nonumber\\
    \overset{(\ref{lemma:optimal-generator})}{=}V_1(\boldsymbol{X}_1)&-\int_0^1\sum_{y\neq \boldsymbol{X}_t}\boldsymbol{A}^0_t(\boldsymbol{X}_t,\boldsymbol{y})\left(1-\mathrm{e}^{V_t(\boldsymbol{y})-V_t(\boldsymbol{X}_t)}\right)dt -\sum_{t: \boldsymbol{X}_{t}\neq \boldsymbol{X}_{t-}}(V_t(\boldsymbol{X}_t)-V_t(\boldsymbol{X}_{t^-}))
\end{align}
\end{small}
where the final equality follows from Lemma \ref{lemma:optimal-generator}. Substituting this expression for $V_0(\boldsymbol{X}_0)$ into (\ref{eq:proof-2}), terms cancel and the RND reduces to:
\begin{align}
    \log\frac{\mathrm{d}\mathbb{P}^\star}{\mathrm{d}\mathbb{P}^0}(\boldsymbol{X}_{0:1})=V_1(\boldsymbol{X}_1)-\log Z\implies \frac{\mathrm{d}\mathbb{P}^\star}{\mathrm{d}\mathbb{P}^0}(\boldsymbol{X}_{0:1})=\frac{1}{Z}\mathrm{e}^{V_1(\boldsymbol{X}_1)}
\end{align}
where $V_1(\boldsymbol{X}_1)=r(\boldsymbol{X}_1)/\alpha$, we conclude our proof. \hfill $\square$

Finally, we are ready to derive the RND between the optimal path measure and the parameterized path measure.

\label{app-prop:rndparam}
\begin{tcolorbox}[sharp corners, colback=mybg, boxrule=0pt]
\rndparam*
\end{tcolorbox}

\textit{Proof.} First, we decompose the RND as:
\begin{align}
    \log\frac{\mathrm{d}\mathbb{P}^\star}{\mathrm{d}\mathbb{P}^v}(\boldsymbol{X}_{0:1})&=\log\frac{\mathrm{d}\mathbb{P}^\star}{\mathrm{d}\mathbb{P}^0}\frac{\mathrm{d}\mathbb{P}^0}{\mathrm{d}\mathbb{P}^v}(\boldsymbol{X}_{0:1})=\log\frac{\mathrm{d}\mathbb{P}^\star}{\mathrm{d}\mathbb{P}^0}(\boldsymbol{X}_{0:1})+\log\frac{\mathrm{d}\mathbb{P}^0}{\mathrm{d}\mathbb{P}^v}(\boldsymbol{X}_{0:1})\label{eq:proof-4}
\end{align}

First, we derive the $\log\frac{\mathrm{d}\mathbb{P}^0}{\mathrm{d}\mathbb{P}^v}(\boldsymbol{X}_{0:1})$ where $\mathbb{P}^0$ is the path measure of the pretrained model $(f^{\text{pre}}, g^{\text{pre}})$ and $\mathbb{P}^v$ is the path measure generated from the tilted models $(f^v, g^v)$. Let $t_i$ denote the times of insertion events and $t_u$ denote the times of unmasking events. Then, the log RND decomposes into:
\begin{small}
\begin{align}
    \log \frac{\mathrm{d}\mathbb{P}^0}{\mathrm{d}\mathbb{P}^v}(\boldsymbol{X}_{0:1})&=\log\frac{\mathrm{d}\mathbb{P}_0^0}{\mathrm{d}\mathbb{P}_0^v}(\boldsymbol{X}_0)+\sum_{t_u:\boldsymbol{X}_{t_u}\neq  \boldsymbol{X}_{t_u-}}\log \frac{\boldsymbol{Q}^0_{t_u}(\boldsymbol{X}_{t_u-}, \boldsymbol{X}_{t_u})}{\boldsymbol{Q}^v_{t_u}(\boldsymbol{X}_{t_u-}, \boldsymbol{X}_{t_u})}+\int_0^1\sum_{\boldsymbol{z}\neq \boldsymbol{X}_{t}}(\boldsymbol{Q}^v_{t}-\boldsymbol{Q}^0_{t})(\boldsymbol{X}_{t}, \boldsymbol{z})dt\nonumber\\
    &+\sum_{t_i: \boldsymbol{X}_{t_i}\neq \boldsymbol{X}_{t_i-}}\log \frac{\boldsymbol{R}^0_{t_i}(\boldsymbol{X}_{t_i-}, \boldsymbol{X}_{t_i})}{\boldsymbol{R}^v_{t_i}(\boldsymbol{X}_{t_i-}, \boldsymbol{X}_{t_i})}+\int_0^1\sum_{\boldsymbol{y}\neq \boldsymbol{X}_{t}}(\boldsymbol{R}^v_{t}-\boldsymbol{R}^0_{t})(\boldsymbol{X}_{t}, \boldsymbol{y})dt
\end{align}
\end{small}

Now, we recall the form of the any-length MDM unmasking rate:
\begin{align}
    \boldsymbol{Q}^v_t(\boldsymbol{x}, \boldsymbol{x}^{\ell \gets \boldsymbol{v}})&=\frac{\dot{\beta}_t}{1-\beta_t}\cdot \mathbb{P}(\boldsymbol{x}_1^{s_t[i]}=\boldsymbol{v}|\boldsymbol{x}_t=\boldsymbol{x}), \quad\boldsymbol{v}\in \mathcal{V}, \boldsymbol{x}^i=\boldsymbol{M}\\
    &=\frac{\dot{\beta}_t}{1-\beta_t}\cdot f^v(\boldsymbol{x},t)[i, \boldsymbol{v}]
\end{align}
where the sum over the vocabulary $\sum_{d\in \mathcal{V}}f_\theta(\boldsymbol{x},t)[i, \boldsymbol{v}]= 1$. For a state $\boldsymbol{x}$ with $|\{\ell: \boldsymbol{x}^\ell=\boldsymbol{M}\}|$ masked positions, we have that the \textbf{exit rate} from state $\boldsymbol{x}$ reduces to: 
\begin{align}
    \sum_{\boldsymbol{y}\neq \boldsymbol{x}}\boldsymbol{Q}_t^v(\boldsymbol{x}, \boldsymbol{y})=\sum_{\ell:\boldsymbol{x}^\ell=\boldsymbol{M}}\sum_d\boldsymbol{Q}_t^v(\boldsymbol{x}, \boldsymbol{x}^{\ell\gets \boldsymbol{v}})=\frac{\dot{\beta}_t}{1-\beta_t}\sum_{\ell:\boldsymbol{x}^\ell=\boldsymbol{M}}1=\frac{\dot{\beta}_t}{1-\beta_t}\big|\{\ell:\boldsymbol{x}^\ell=\boldsymbol{M}\}\big|
\end{align}
Given the noise schedule $\beta_t$ is equal for both the pretrained and fine-tuned models, the number of masked positions along the interpolant is equal. So we can write: 
\begin{align}
    \sum_{\boldsymbol{y}\neq \boldsymbol{x}}\boldsymbol{Q}_t^0(\boldsymbol{x}, \boldsymbol{y})=\frac{\dot{\beta}_t}{1-\beta_t}\big|\{\ell:\boldsymbol{x}^\ell=\boldsymbol{M}\}\big|
\end{align}
and the exit term vanishes:
\begin{align}
    \int_0^1\sum_{y\neq \boldsymbol{X}_{t_u}}(\boldsymbol{Q}_{t_u}^v-\boldsymbol{Q}^0_{t_u})(\boldsymbol{X}_{t_u}, y)\mathrm{d}t=0
\end{align}

The insertion rate is defined as:
\begin{align}
    \boldsymbol{R}^v_t(\boldsymbol{x}, \boldsymbol{x}^{\insertat{\ell} \boldsymbol{M}})&=\frac{\dot{\alpha}_t}{1-\alpha_t}\cdot \mathbb{E}[s_t[\ell]-s_t[\ell-1]-1|\boldsymbol{x}_t=\boldsymbol{x}]\\
    &=\frac{\dot{\alpha}_t}{1-\alpha_t}g^v(\boldsymbol{x},t)[\ell]
\end{align}
Since the insertion rate depends on the model outputs, we have that $\boldsymbol{R}^0_t(\boldsymbol{x}, \boldsymbol{x}^{\insertat{\ell} \boldsymbol{M}})\neq \boldsymbol{R}^v_t(\boldsymbol{x}, \boldsymbol{x}^{\insertat{\ell} \boldsymbol{M}})$. In this case, the exit term is defined as:
\begin{small}
\begin{align}
    \int_0^1\sum_{y\neq \boldsymbol{X}_{t_u}}(\boldsymbol{R}_{t}^v-\boldsymbol{R}^0_{t})(\boldsymbol{X}_{t}, y)dt=\int_0^1\frac{\dot{\alpha}_t}{1-\alpha_t}\left(\sum_\ell g^v(\boldsymbol{X}_{t},t)[\ell]-\sum_\ell g_\theta^{\text{pre}}(\boldsymbol{X}_{t},t)[\ell]\right)dt
\end{align}
\end{small}

Putting these terms together, we get the final form of the log RND between any-length path measures:
\begin{small}
\begin{align}
    &\log\frac{\mathrm{d}\mathbb{P}^0}{\mathrm{d}\mathbb{P}^v}(\boldsymbol{X}_{0:1})=\sum_{t_u:\boldsymbol{X}_{t_u}\neq  \boldsymbol{X}_{t_u-}}\sum_{\ell: \boldsymbol{X}^\ell_{t_u-}\neq \boldsymbol{X}^\ell_{t_u}}\log \frac{f^{\text{pre}}(\boldsymbol{X}_{t_u}, t)[\ell, \boldsymbol{v}]}{f^v(\boldsymbol{X}_{t_u}, t)[\ell, \boldsymbol{v}]}\nonumber\\
    &+\sum_{t_i: \boldsymbol{X}_{t_i}\neq \boldsymbol{X}_{t_i-}}\sum_{\ell: \boldsymbol{X}_{t_i}= \boldsymbol{X}^{\insertat{\ell}\boldsymbol{M}}_{t_i^-}}\log \frac{g_\theta^{\text{pre}}(\boldsymbol{X}_{t_i}, t_i)[\ell]}{g^v(\boldsymbol{X}_{t_i}, t_i)[\ell]}+\int_0^1\frac{\dot{\alpha}_t}{1-\alpha_t}\left(\sum_\ell g^v(\boldsymbol{X}_{t},t)[\ell]-\sum_\ell g^{\text{pre}}(\boldsymbol{X}_{t},t)[\ell]\right)dt\label{eq:proof-3}
\end{align}
\end{small}
Finally, substituting the result from Lemma \ref{lemma:rnd-optimal} and (\ref{eq:proof-3}) into (\ref{eq:proof-4}), we get:
\begin{small}
\begin{align}
    \log\frac{\mathrm{d}\mathbb{P}^\star}{\mathrm{d}\mathbb{P}^v}(\boldsymbol{X}_{0:1})&=\log\frac{\mathrm{d}\mathbb{P}^\star}{\mathrm{d}\mathbb{P}^0}(\boldsymbol{X}_{0:1})+\log\frac{\mathrm{d}\mathbb{P}^0}{\mathrm{d}\mathbb{P}^v}(\boldsymbol{X}_{0:1})\nonumber\\
    &=\frac{r(\boldsymbol{X}_1)}{\alpha}-\log Z+\sum_{t_u:\boldsymbol{X}_{t_u}\neq  \boldsymbol{X}_{t_u-}}\sum_{\ell: \boldsymbol{X}^\ell_{t_u}\neq \boldsymbol{X}^\ell_{t_u^-}}\log \frac{f^{\text{pre}}(\boldsymbol{X}_{t_u}, t)[\ell, \boldsymbol{v}]}{f^v(\boldsymbol{X}_{t_u}, t)[\ell, \boldsymbol{v}]}\nonumber\\
    &+\sum_{t_i: \boldsymbol{X}_{t_i}\neq \boldsymbol{X}_{t_i-}}\sum_{\ell: \boldsymbol{X}_{t_i}= \boldsymbol{X}^{\insertat{\ell}\boldsymbol{M}}_{t_i^-}}\log \frac{g_\theta^{\text{pre}}(\boldsymbol{X}_{t_i}, t_i)[\ell]}{g^v(\boldsymbol{X}_{t_i}, t_i)[\ell]}\nonumber\\
    &+\int_0^1\frac{\dot{\alpha}_t}{1-\alpha_t}\left(\sum_\ell g^v(\boldsymbol{X}_{t},t)[\ell]-\sum_\ell g^{\text{pre}}(\boldsymbol{X}_{t},t)[\ell]\right)dt
\end{align}
\end{small}
which concludes our proof. \hfill $\square$

We use this form of the log RND to define the importance weight $W^v(\boldsymbol{X}_{0:1}):=\log \frac{\mathrm{d}\mathbb{P}^\star}{\mathrm{d}\mathbb{P}^v}(\boldsymbol{X}_{0:1})$ in the AJD loss, which yields the optimal any-length path measure $\mathbb{P}^\star$, as we will prove in Prop \ref{prop:ajd-loss} below.

\begin{proposition}[Adaptive Joint Decoding Loss]\label{prop:ajd-loss}
    The unique minimizer of the \textbf{Adaptive Joint Decoding (AJD)} loss defined as:
    \begin{small}
    \begin{align}
        \mathcal{L}_{\text{AJD}}(\theta, \phi):=\underset{\boldsymbol{X}_{0:1}\sim \mathbb{P}^v}{\mathbb{E}}\left[\frac{1}{Z}\mathrm{e}^{W^v}\big[\mathcal{L}_{\text{unmask}}(\theta;\boldsymbol{X}_1)+\mathcal{L}_{\text{insert}}(\theta;\boldsymbol{X}_1)+\mathcal{L}_{\text{UQL}}(\phi;\boldsymbol{X}_1)+\mathcal{L}_{\text{IQL}}(\phi;\boldsymbol{X}_1)\big]\right]\nonumber
    \end{align}
    \end{small}
    is the optimal unmasking generator $\boldsymbol{Q}^\star$ and insertion generator $\boldsymbol{R}^\star$ of the reward-tilted path measure $\mathbb{P}^\star$
\end{proposition}

\textit{Proof.} First, we recall the form of the cross-entropy loss as:
\begin{align}
    \mathcal{F}_{\text{CE}}(\mathbb{P}^{\theta, \phi}, \mathbb{P}^\star):=D_{\text{KL}}(\mathbb{P}^\star\|\mathbb{P}^{\theta, \phi})=\mathbb{E}_{\mathbb{P}^\star}\left[\log \frac{\mathrm{d}\mathbb{P}^\star}{\mathrm{d}\mathbb{P}^{\theta, \phi}}\right]=\mathbb{E}_{\mathbb{P}^{v}}\bigg[\frac{\mathrm{d}\mathbb{P}^\star}{\mathrm{d}\mathbb{P}^{v}}\log \frac{\mathrm{d}\mathbb{P}^\star}{\mathrm{d}\mathbb{P}^{\theta, \phi }}\bigg]
\end{align}
Then, we substitute $W^v(\boldsymbol{X}_{0:1})=\log \frac{\mathrm{d}\mathbb{P}^\star}{\mathrm{d}\mathbb{P}^v}(\boldsymbol{X}_{0:1})$. Since $\mathbb{P}^\star$ is fixed under $(\theta, \phi)$, minimizing the CE loss is equivlalent to minimizing:
\begin{align}
    \min_{\theta, \phi}\mathbb{E}_{\mathbb{P}^{v}}\bigg[\frac{\mathrm{d}\mathbb{P}^\star}{\mathrm{d}\mathbb{P}^{v}}\log \frac{\mathrm{d}\mathbb{P}^\star}{\mathrm{d}\mathbb{P}^{\theta, \phi }}\bigg]=\min_{\theta, \phi}\mathbb{E}_{\mathbb{P}^{v}}\bigg[\frac{1}{Z}\mathrm{e}^{W^v(\boldsymbol{X}_{0:1})}(-\log \mathrm{d}\mathbb{P}^{\theta, \phi }(\boldsymbol{X}_{0:1}))\bigg]\label{eq:minimize-ce}
\end{align}
Since the path measure $\mathbb{P}^{\theta, \phi}$ is defined by disjoint insertion and unmasking decisions, the probability of a path $\boldsymbol{X}_{0:1}$ under $\mathbb{P}^{\theta, \phi}$ can be written as the product of insertion and unmasking rates:
\begin{align}
    \mathrm{d}\mathbb{P}^{\theta, \phi}(\boldsymbol{X}_{0:1})&=\mathbb{P}_0^{\theta, \phi}(\boldsymbol{X}_0)\prod_{t:\boldsymbol{X}_{t^-}\neq \boldsymbol{X}_{t}}\mathbb{P}^{\theta, \phi}(\boldsymbol{X}_{t^-}| \boldsymbol{X}_{t})\nonumber\\
    &=\mathbb{P}_0^{\theta, \phi}(\boldsymbol{X}_0)\prod_{t_i:\boldsymbol{X}_{t_i^-}\neq \boldsymbol{X}_{t_i}}\mathbb{P}^{\theta, \phi}(\boldsymbol{X}_{t_i^-}| \boldsymbol{X}_{t_i})\prod_{t_u:\boldsymbol{X}_{t_u^-}\neq \boldsymbol{X}_{t_u}}\mathbb{P}^{\theta, \phi}(\boldsymbol{X}_{t_u^-}| \boldsymbol{X}_{t_u})
\end{align}
Taking logs, we have:
\begin{small}
\begin{align}
    -\log \mathbb{P}^{\theta, \phi}(\boldsymbol{X}_{0:1})&=-\log \mathbb{P}_0^{\theta, \phi}(\boldsymbol{X}_0)+\sum_{t_i:\boldsymbol{X}_{t_i}\neq \boldsymbol{X}_{t_i-}}-\log \mathbb{P}^{\theta, \phi}(\boldsymbol{X}_{t_i}| \boldsymbol{X}_{t_i^-})+\sum_{t_u:\boldsymbol{X}_{t_u}\neq \boldsymbol{X}_{t_u-}}-\log \mathbb{P}^{\theta, \phi}(\boldsymbol{X}_{t_u}| \boldsymbol{X}_{t_u^-})\nonumber
\end{align}
\end{small}
Instead of defining the loss with the probability of generating $\boldsymbol{X}_1\sim p_{\text{target}}$ via only the single trajectory $\mathbb{P}^{\theta, \phi}(\boldsymbol{X}_{0:1})$, we can define it with the expectation over many possible trajectories by taking an expectation over intermediate samples from the endpoint-conditioned interpolant $\tilde{\boldsymbol{x}}_t\sim p_t(\cdot|\boldsymbol{X}_1)$ and times $t\sim \mathcal{U}(0,1)$ following \citep{zhu2025mdns}. Then, we have:
\begin{small}
\begin{align}
    -\log \mathbb{P}^{\theta, \phi}(\boldsymbol{X}_{0:1})&=-\log \mathbb{P}_0^{\theta, \phi}(\boldsymbol{X}_0)+\underset{t\sim \mathcal{U}(0,1)}{\mathbb{E}}\underset{\tilde{\boldsymbol{x}}_t\sim p_t(\cdot|\boldsymbol{X}_1)}{\mathbb{E}}\left[-\log \mathbb{P}_{\text{insert}}^{\theta, \phi}(\boldsymbol{X}_{t_i}|\tilde{\boldsymbol{x}}_t, \boldsymbol{X}_1)-\log \mathbb{P}_{\text{unmask}}^{\theta, \phi}(\boldsymbol{X}_{t_u}|\tilde{\boldsymbol{x}}_t, \boldsymbol{X}_1)\right]\nonumber
\end{align}
\end{small}
Since we decompose each step into unmasking via $f_\theta$ and remasking via $\mu_\phi$ and insertion via $g_\theta$ and deletion under $\nu_\phi$ and define the losses as a form of the negative log-likelihood of the step given the \textbf{ground-truth} sequence $\boldsymbol{X}_1\sim p_{\text{target}}$, we can write:
\begin{align}
    -\log \mathbb{P}^{\theta, \phi}(\boldsymbol{X}_{0:1})=\underbrace{\mathcal{L}_{\text{unmask}}(\theta;\boldsymbol{X}_1)}_{(\ref{loss:unmask})}+\underbrace{\mathcal{L}_{\text{insert}}(\theta;\boldsymbol{X}_1)}_{(\ref{loss:insert})}+\underbrace{\mathcal{L}_{\text{UQL}}(\phi;\boldsymbol{X}_1)}_{(\ref{loss:uql})}+\underbrace{\mathcal{L}_{\text{IQL}}(\phi;\boldsymbol{X}_1)}_{(\ref{loss:iql})}+\text{const}
\end{align}
Therefore, the minimizer in (\ref{eq:minimize-ce}) is equivalent to:
\begin{small}
\begin{align}
    \min_{\theta, \phi}&\mathbb{E}_{\mathbb{P}^{v}}\bigg[\frac{\mathrm{d}\mathbb{P}^\star}{\mathrm{d}\mathbb{P}^{v}}\log \frac{\mathrm{d}\mathbb{P}^\star}{\mathrm{d}\mathbb{P}^{\theta, \phi }}\bigg]\nonumber\\
    &=\min_{\theta, \phi}\mathbb{E}_{\mathbb{P}^{v}}\bigg[\frac{1}{Z}\mathrm{e}^{W^v(\boldsymbol{X}_{0:1})}\big[\mathcal{L}_{\text{unmask}}(\theta;\boldsymbol{X}_1)+\mathcal{L}_{\text{insert}}(\theta;\boldsymbol{X}_1)+\mathcal{L}_{\text{UQL}}(\phi;\boldsymbol{X}_1)+\mathcal{L}_{\text{IQL}}(\phi;\boldsymbol{X}_1)\big]\bigg]
\end{align}
\end{small}
and defining $\mathcal{L}_{\text{AJD}}(\theta, \phi)$ as:
\begin{small}
\begin{align}
    \mathcal{L}_{\text{AJD}}(\theta, \phi):=\underset{\boldsymbol{X}_{0:1}\sim \mathbb{P}^v}{\mathbb{E}}\left[\frac{1}{Z}\mathrm{e}^{W^v}\big[\mathcal{L}_{\text{unmask}}(\theta;\boldsymbol{X}_1)+\mathcal{L}_{\text{insert}}(\theta;\boldsymbol{X}_1)+\mathcal{L}_{\text{UQL}}(\phi;\boldsymbol{X}_1)+\mathcal{L}_{\text{IQL}}(\phi;\boldsymbol{X}_1)\big]\right]
\end{align}
\end{small}
concludes our proof. \hfill $\square$

\section{Small Molecule Experiment Details}
\label{app:molecule-exp}
\subsection{Pre-Training the Any-Length Small-Molecule MDM}

\paragraph{Data} 
We represent molecules using SAFE (Sequential Attachment-based Fragment Embedding) strings using the SAFE tokenizer with a vocabulary size of $V=1880$, including the special pad ($\texttt{id}=3$) and mask ($\texttt{id}=4$) tokens \citep{noutahi2024gotta}. All sequences are tokenized with padding and truncation to a maximum length of $L=256$ tokens. 

The model is pretrained on the SAFE-GPT dataset (\texttt{datamol-io/safe-gpt}) \citep{noutahi2024gotta} hosted on HuggingFace, which contains over 1 billion drug-like molecules sourced from ZINC20 and other chemical databases. The dataset is split into $945,455,307$ training sequences, $118,451,032$ validation sequences, and $118,890,444$ test sequences. 

\paragraph{Model Architecture}
The flexible-length model architecture is a 96M-parameter bidirectional transformer based on the Diffusion Transformer (DiT) adapted for discrete masked diffusion over variable-length sequences. The architecture consists of \textbf{(1) token embedding layer} that projects the vocabulary with $V=1880$ to embeddings of dimension $d=768$, \textbf{(2) timestep conditioning} that encodes the scalar timestep $t\in [0,1]$ via sinusoidal positional embeddings (frequency dimension 256), projected through a 2-layer MLP with SiLU activation to a conditioning vector of dimension $d_c=128$, and \textbf{(3) Transformer backbone} with 12 DiT blocks with Rotary Position Embeddings (RoPE, base frequency 10,000) and FlexAttention with block masking for variable-length sequence support. There are two prediction heads: \textbf{(1) token prediction head} that predicts the unmasking posterior in the form of logits over the vocabulary dimension and \textbf{(2) length prediction head} which predicts the per-position expected gap count. The unmasking quality predictor $\mu_\phi$ and insertion quality predictor $\nu_\phi$ are lightweight heads that operate on the final-layer hidden states of the shared backbone, where each applies a LayerNorm followed by a 2-layer MLP (dimensions $d\to d\to 1$) with a GELU activation, producing a per-position scalar that is passed through a sigmoid to yield a quality value in $[0,1]$. Hyperparameters are provided in Table \ref{table:hyperparmeters-pretrain}.

\paragraph{Training Details}
The model is trained for $87,000$ gradient steps using AdamW with a learning rate of $3\times 10^{-4}$ and weight decay of 0.03. The learning rate follows a linear warmup over the first 2,000 steps from $3\times 10^{-10}$ to $3\times 10^{-3}$ followed by cosine annealing to zero over the remaining steps. Training uses a global batch size of $2,048$, implemented with a per-GPU batch size of 64 across 2 GPUs with 16 gradient accumulation steps. An exponential moving average (EMA) of model parameters is maintained with decay rate $0.9999$. 

\begin{table*}[t]
\caption{\textbf{Hyperparameters for pre-training any-length peptide and small molecule masked diffusion models.}}
\label{table:hyperparmeters-pretrain}
\begin{center}
\begin{small}
\resizebox{0.6\linewidth}{!}{
\begin{tabular}{@{}lcc@{}}
\toprule
 \textbf{Hyperparameter} & \textbf{Small Molecule} & \textbf{Peptides} \\
\midrule
Hidden dimension & $768$ & $768$ \\
Attention heads & $12$ & $12$ \\
Transformer blocks & $12$ & $12$ \\
Timestep conditioning dimension & $128$ & $128$ \\
MLP expansion ratio & $4$ & $4$ \\
Dropout & $0.05$ & $0.05$ \\
Positional encoding & Rotary (RoPE) & Rotary (RoPE) \\
Vocabulary size & $1,882$ & $587$ \\
Max sequence length & $256$ & $1,024$ \\
Global batch size & $2,048$ & $1,024$ \\
Optimizer & AdamW & AdamW \\
Gradient accumulation steps & $16$ & $4$ \\
\bottomrule
\end{tabular}
}
\end{small}
\end{center}
\end{table*}

\subsection{Rewards and Evaluation Metrics}
\paragraph{Uniqueness} 
The fraction of valid generated molecules whose canonical SMILES strings are distinct. Higher uniqueness indicates that the model has not collapsed onto a small set of high-reward modes.

\paragraph{QED}
The \textbf{Q}uantitative \textbf{E}stimate of \textbf{D}rug likeness (QED) \citep{bickerton2012quantifying} is a scalar in $[0, 1]$ that aggregates eight physicochemical properties (molecular weight, logP, hydrogen-bond donors and acceptors, polar surface area, rotatable bonds, aromatic rings, and structural alerts) into a single drug-likeness score. We compute QED using the RDKit implementation and use it both as a fine-tuning reward and as an evaluation metric.

\paragraph{Synthetic Accessibility}
The Synthetic Accessibility (SA) score \citep{ertl2009estimation} estimates the ease of synthesizing a molecule on a scale from $1$ (easy) to $10$ (hard), based on fragment contributions and structural complexity penalties. Since the reward is maximized and summed with QED, SA is flipped and rescaled so higher = more synthesizable. The SA reward is transformed as $\text{reward}_{\text{SA}} = \frac{1}{1 + \text{SA}}$. Empirically, we found that optimizing only QED without the SA in the reward also led to a decrease in SA. We report SA via RDKit, with lower values preferred. 

\paragraph{Quality} 
Following \citet{jin2020multi}, we report \emph{quality} as the fraction of generated molecules that are simultaneously valid, unique, drug-like, and synthesizable. A molecule is considered drug-like if its QED~\citep{bickerton2012quantifying} satisfies $\text{QED} \geq 0.6$ and synthesizable if its SA score~\citep{ertl2009estimation} satisfies $\text{SA} \leq 4$. Quality thus collapses validity, uniqueness, drug-likeness, and synthesizability into a single metric that measures whether the model can generate chemically reasonable and distinct molecules.

\paragraph{Diversity}
The average pairwise Tanimoto distance ($1 - \texttt{Tanimoto similarity}$) between Morgan fingerprints (radius $2$ and $2048$ bits) of all valid generated molecules. Higher diversity indicates broader coverage of chemical space.

\subsection{Fine-Tuning Details}
Fine-tuning with A2D2 follows Algorithm \ref{alg:a2d2} with hyperparameters provided in Table \ref{table:hyperparmeters-finetune}. For the A2D2 baseline without quality, we only perform the first loop in Algorithm \ref{alg:a2d2}, which fine-tunes the policy model without the planner and does not alternate to the second loop. For the baselines with only insertion or unmasking quality, we alternate to the planner fine-tuning but using only quality-based unmasking \textit{or} quality-based insertions. During buffer generation, we retain only the sequences that pass the quality threshold and weight them by their log RND added to the QED and SA rewards. Fine-tuning was performed on 1 NVIDIA A100 GPU per run.

\begin{table*}[t]
\caption{\textbf{Hyperparameters for small molecule and peptide fine-tuning.} The same base hyperparameters were used for all A2D2 baselines as well as fixed-length RL fine-tuning results.}
\label{table:hyperparmeters-finetune}
\begin{center}
\begin{small}
\resizebox{0.7\linewidth}{!}{
\begin{tabular}{@{}lcc@{}}
\toprule
 \textbf{Hyperparameter} & \textbf{Small Molecule} & \textbf{Peptides} \\
\midrule
Number of Replicates $R$  & $16$ & $8$ \\
Buffer Size $B$ & $1000$ & $1000$ \\
Buffer Refresh Fraction & $0.3$ & $0.3$ \\
Iterations Between Buffer Generation $N_{\text{resample}}$ & $10$ & $10$ \\
Gradient Steps Per Iteration $N_{\text{update}}$ & $25$ & $10$ \\ 
Alternation Frequency $N_{\text{alt}}$ & $5$ & $5$ \\
Warmup Iterations $N_{\text{warmup}}$ & $20$ & $20$ \\
Number of Sampling Steps $N_{\text{steps}}$ & $256$ & $256$ \\
Training Mini Batch Size & $20$ & $10$ \\
Reward Scaling $\alpha$ & $0.01$ & $0.1$ \\
Max Length $L$ & $256$ & $256$ \\
\bottomrule
\end{tabular}
}
\end{small}
\end{center}
\end{table*}

\section{Peptide Experiment Details}
\label{app:peptide-exp}

\subsection{Pre-Training the Any-Length Peptide MDM}

\paragraph{Data}
The any-length peptide MDM is pretrained on a corpus of approximately 11 million peptide SMILES sequences from \citet{tang2025peptune}, containing $7451$ sequences from the CycPeptMPDB database \citep{li2023cycpeptmpdb}, $825,632$ unique peptides from SmProt \citep{li2021smprot}, and approximately $10$ million modified peptides generated from CycloPs \citep{duffy2011cyclops, feller2025peptide}. We leverage the SMILES Pair Encoding tokenizer with vocabulary size $V=587$ \citep{feller2025peptide}. The data is split with a $0.8/0.2$ train/validation ratio and batched with dynamic batch sizes to handle variable-length token sequences with a per-GPU token count set to $16,000$.

\paragraph{Model Architecture and Training Details}
We leverage the same model architecture and training procedure as in the small-molecule experiment, with specific hyperparameters provided in Table \ref{table:hyperparmeters-pretrain}. Pre-training was performed on 4 NVIDIA A100 GPUs until convergence. The checkpoint at $46$ epochs was used for fine-tuning.

\subsection{Rewards and Evaluation Metrics}
\paragraph{Validity}
The fraction of generated SMILES sequences that can be decoded into a sequence of canonical and non-canonical amino acids using the SMILES2PEPTIDE filter \citep{tang2025peptune}.

\paragraph{Uniqueness} 
The fraction of distinct peptides over all valid peptide sequences.

\paragraph{Diversity}
The similarity of the generated peptide structures, measured as one minus the Tanimoto similarity between the pairwise Morgan fingerprints of generated sequences, computed as:
\begin{align}
    \text{Diversity}=1-\frac{1}{\binom{N_{\text{generated}}}{2}}\sum_{i,j}\frac{\boldsymbol{f}(\boldsymbol{x}_i)\cdot \boldsymbol{f}(\boldsymbol{x}_j)}{|\boldsymbol{f}(\boldsymbol{x}_i)|+ |\boldsymbol{f}(\boldsymbol{x}_j)|-\boldsymbol{f}(\boldsymbol{x}_i)\cdot \boldsymbol{f}(\boldsymbol{x}_j)}
\end{align}

\paragraph{Therapeutic Property Oracles}
We use lightweight property classifiers from PeptiVerse \citep{zhang2026peptiverse}, trained on experimental peptide data, to predict binding affinity to an input protein sequence, solubility, non-hemolysis, non-fouling, and cell permeability. For binding affinity, values $>9$ indicate strong $n\mathrm{M}$ to $p\mathrm{M}$ binders, $7-9$ indicates $n\mathrm{M}$ to $\mu\mathrm{M}$ medium binders, and $< 7$ indicates weak $\mu\mathrm{M}$ binders. For solubility, non-hemolysis, and non-fouling, the values are between $[0,1]$, indicating the probability of the positive class. For cell permeability, the model predicts the logarithm of the effective permeability coefficient $\log P_{\text{exp}}$, where values $\geq -6.0$ were labeled as high permeability, and values $< -6.0$ as weak permeability.

\subsection{Fine-Tuning Details}
Fine-tuning with A2D2 follows Algorithm \ref{alg:a2d2} with hyperparameters provided in Table \ref{table:hyperparmeters-finetune}. Similar to the small molecule experiment, we only perform the first loop in Algorithm \ref{alg:a2d2} for the A2D2 baseline without quality, which fine-tunes the policy model without the planner and does not alternate to the second loop. For the baselines with only insertion or unmasking quality, we alternate to the planner fine-tuning but using only quality-based unmasking \textit{or} quality-based insertions. During buffer generation, we retain only the valid sequences that pass the \texttt{SMILES2PEPTIDE} filter \citep{tang2025peptune}. The multi-objective rewards are scalarized and added to the log RND for each sequence. Fine-tuning was performed on a single NVIDIA A100 GPU per run.

\section{Language Reasoning Experiment Details}
\label{app:language}
\subsection{Pre-Training the Any-Length Language MDM}
Following the pre-training procedure from \citet{kim2025any}, we adapt the LLaDA-8B-Base model \citep{nie2026large} for any-length generation by adding a time-embedding layer with AdaLN \citep{peebles2023scalable} since the insertion expectation requires time dependence. Four AdaLN parameter sets are shared across intermediate bidirectional Transformer layers, and the insertion expectation is modeled with an additional \texttt{softplus} activation head. To adapt the backbone layers for the new any-length unmasking prediction task while maintaining transfer efficiency, we add LoRA adapters \citep{hu2022lora} to each attention layer and the unmasking posterior prediction head. The LoRA parameters are set to rank $r=128$, scaling factor $\alpha=128$, and $\texttt{dropout}=0.1$. The training data consists of an equal mixture of OpenWebText \citep{Gokaslan2019OpenWeb} and Proof-Pile-2 \citep{azerbayev2023llemma}, totalling ~$13.97$M training sequences. Training ran for approximately 3 days on 8 B200 NVIDIA GPUs for $240,000$ optimizer steps and $\approx1.10$ epochs.

\subsection{Reasoning Tasks}
We evaluate on two reasoning tasks where any-length, any-order decoding is particularly well-suited: math word-problem solving and code infilling.

\paragraph{Math reasoning (GSM8K).} GSM8K \citep{cobbe2021training} requires generating a multi-step chain-of-thought before a final numeric answer. Because the length of the reasoning trace varies substantially across problems, fixed-length decoding either truncates long derivations or wastes a large mask budget on padding. Any-length generation instead inserts tokens as the reasoning chain grows, letting the response length adapt to the difficulty of each problem. We prompt the model to emit reasoning and the answer in a fixed XML schema, \texttt{<reasoning>\dots</reasoning><answer>\textbackslash boxed\{\dots\}</answer>}, and extract the boxed value as the predicted answer.

\paragraph{Code infilling (HumanEval-infill)} We use the single-line infilling split of the HumanEval benchmark \citep{chen2021evaluating, bavarian2022efficient}, where a fixed prefix and suffix surround a masked span that the model must complete. The number of tokens needed to fill the gap is unknown a priori, making this a natural setting for insertion-based decoding: the model conditions bidirectionally on both the prefix and suffix and inserts the variable-length completion between them. We delimit the infilling region with \texttt{<infill-boundary>} tokens and score the generated span against the canonical solution.

For both tasks, we compare the A2D2-finetuned any-length model against the pretrained any-length model adapted from LLaMA-8B-Base with instruction fine-tuning under matched inference budgets. Generation uses greedy decoding (temperature $0$) with a generation length of $1024$, and we sweep the number of diffusion sampling steps $N_{\text{steps}} \in \{128, 256, 512, 1024\}$.

\subsection{Instruction Fine-Tuning}
We instruction fine-tune both the any-length MDM (obtained from the transfer-learning stage above) and the fixed-length LLaDA-8B-Base baseline separately for each task, with no cross-task data mixing. For math, we fine-tune on the GSM8K train split, formatting each example with the reasoning/answer XML schema described above and supervising only the response tokens. For code, we fine-tune on the \texttt{educational\_instruct} subset of OpenCoder \texttt{opc-sft-stage2} \citep{huang2025opencoder} and reserve the HumanEval single-line infilling tasks for evaluation. Both models are trained with LoRA adapters of rank $r=128$, scaling factor $\alpha=128$, and $\texttt{dropout}=0.1$ applied to the attention projections (\texttt{q\_proj}, \texttt{k\_proj}, \texttt{v\_proj}) and the feed-forward output projection, using a learning rate of $1\text{e-}4$ with a cosine schedule, a maximum sequence length of $1024$, and variable-length batching. We select the checkpoint with the best tracked validation correctness. Instruction fine-tuning was performed on 8 B200 NVIDIA GPUs.

\subsection{Reinforcement Learning Fine-Tuning with A2D2}
We then fine-tune the instruction-tuned any-length MDM with \textbf{A2D2} following Algorithm \ref{alg:a2d2}. The insertion–unmasking policy is initialized from the IFT checkpoint and fine-tuned with the Adaptive Joint Decoding (AJD) loss, while the unmasking- and insertion-quality predictor heads are trained from scratch using the same architecture as in the small molecule and peptide experiments. Following \citet{zhou2025co}, we alternate between updating the policy and the quality predictors every $N_{\text{alt}}=5$ iterations, and warm up by sampling the replay buffer with the policy alone for $N_{\text{warmup}}=50$ iterations before enabling quality-based sampling. Each iteration draws $12$ generations per prompt via greedy adaptive inference (temperature $0$) at a generation length of $1024$, stores them in the replay buffer, and performs $256$ gradient steps per epoch with a per-device batch size of $6$ and gradient accumulation of $2$. The policy is optimized at learning rate $1\text{e-}6$ with gradient-norm clipping of $0.2$ and weight decay $0.1$, and the AJD loss uses reward scaling $\alpha=0.01$, regularization coefficient $0.5$, and antithetic sampling of the masking schedule. Buffer generation and the inference-time schedule use quality-guided inference using a sliding window of size $64$. Sequences are weighted by their log Radon–Nikodym derivative added to the task reward. RL fine-tuning was performed on 2 B200 NVIDIA GPUs per run. Key hyperparameters are summarized in Table \ref{table:hyperparameters-language-finetune}.

\begin{table*}[t]
\caption{\textbf{Hyperparameters for language fine-tuning with A2D2.}}
\label{table:hyperparameters-language-finetune}
\begin{center}
\begin{small}
\resizebox{0.6\linewidth}{!}{
\begin{tabular}{@{}lcc@{}}
\toprule
 \textbf{Hyperparameter} & \textbf{GSM8K} & \textbf{HumanEval-infill} \\
\midrule
Number of Generations & $12$ & $12$ \\
Iterations Between Buffer Generation $N_{\text{resample}}$ & $12$ & $12$ \\
Gradient Steps Per Iteration $N_{\text{update}}$ & $256$ & $256$ \\
Alternation Frequency $N_{\text{alt}}$ & $5$ & $5$ \\
Warmup Iterations $N_{\text{warmup}}$ & $50$ & $50$ \\
Number of Sampling Steps for Buffer Generation $N_{\text{steps}}$ & $256$ & $256$ \\
Per-Device Batch Size & $6$ & $6$ \\
Gradient Accumulation Steps & $2$ & $2$ \\
Policy Learning Rate & $1\text{e-}6$ & $1\text{e-}6$ \\
Reward Scaling $\alpha$ & $0.01$ & $0.01$ \\
Max Length $L$ & $1024$ & $1024$ \\
\bottomrule
\end{tabular}
}
\end{small}
\end{center}
\end{table*}

\subsection{Datasets and Rewards}

\paragraph{GSM8K} GSM8K \citep{cobbe2021training} is a benchmark of grade-school arithmetic word problems that test multi-step numerical reasoning. The model is trained on the training partition and evaluated on the held-out test partition. Because A2D2 optimizes the policy through reward signals rather than supervised targets, we require a scalar reward that reflects how well a sampled response solves the problem. We adopt the same reward decomposition as d1 \citep{zhao2026d1}, which jointly rewards correct formatting and correct final answers:
\begin{enumerate}
    \item [(i)] \textit{XML Structure Reward}: If the opening and closing tag (\texttt{<reasoning>},\texttt{</reasoning>}, \texttt{<answer>},\texttt{</answer>}) is correctly placed, the reward gets $+0.125$ and for each token after the \texttt{</answer>} tag, the reward gets $-0.001$.
    \item[(ii)] \textit{Soft Format Reward}: If the response matchs the pattern \texttt{<reasoning>\dots </reasoning><answer>\dots</answer>}, the reward gets $+0.5$.
    \item[(iii)] \textit{Strict Format Reward}: If the response matches the format exactly with correct line breaks, the reward gets $+0.5$.
    \item[(iv)] \textit{Integer Answer Reward}: If the extracted value in the answer tags is an integer, the reward gets $+0.5$.
    \item[(v)] \textit{Correctness Reward}: For answers that are exactly correct, the reward gets $+2$.
\end{enumerate}

\paragraph{HumanEval-infill} For the code infilling experiment, we fine-tune on the \texttt{educational\_instruct} subset of OpenCoder \texttt{opc-sft-stage2} \citep{huang2025opencoder} and evaluate on the single-line infilling split of HumanEval \citep{chen2021evaluating, bavarian2022efficient}, where each task provides a prefix, a suffix, and a held-out canonical infill. Since the canonical solution is available as a reference, we define a reward that proxies infilling pass@1 by combining executability with similarity to the reference. Let $c$ denote the generated completion, $g$ the ground-truth infill, and $\texttt{edit\_sim}(c,g) = 1 - \frac{\text{Lev}(c,g)}{\max(|c|,|g|)}$ the normalized Levenshtein edit similarity. The reward decomposes as:
\begin{enumerate}
    \item [(i)] \textit{Non-Empty Reward}: If the completion is non-empty, the reward gets $+0.25$.
    \item[(ii)] \textit{Completion Compile Reward}: If the completion fragment compiles as valid Python, the reward gets $+0.5$, gated on $\texttt{edit\_sim}(c,g) > 0.3$ and the ground-truth infill compiling standalone. The edit-similarity gate prevents the model from gaming the reward by emitting arbitrary compiling code.
    \item[(iii)] \textit{Full-Program Compile Reward}: If the full reconstructed program \texttt{(prefix + completion + suffix)} compiles, the reward gets $+0.75$, gated on $\texttt{edit\_sim}(c,g) > 0.3$.
    \item[(iv)] \textit{Edit Similarity Reward}: A soft gradient toward the reference of $+1.5 \cdot \texttt{edit\_sim}(c,g)$.
\end{enumerate}
The total reward lies in the range $[0, 3.0]$, comparable to the GSM8K reward range.

\section{Ablation Studies}
We perform ablations for the small molecule experiment on two key hyperparameters: the KL regularization weight $\alpha$ and the number of sampling steps. Results are reported in Table \ref{table:molecule-ablation}.

\paragraph{KL Regularization Weight}
The coefficient $\alpha$ controls the trade-off between reward maximization and adherence to the base generative distribution. We observe a clear tension between these objectives. The smallest value ($\alpha=0.01$) achieves the highest quality and QED, reflecting strong reward optimization, but at the cost of slightly lower validity, uniqueness, and diversity. Increasing to $\alpha=0.1$ inverts this pattern: validity, uniqueness, and diversity all reach their best values, while quality drops to $44.6\%$ and QED to $0.642$, indicating that excessive regularization toward the prior dilutes the reward signal. 

\paragraph{Gradient Steps Per Iteration}
Reducing the number of gradient steps per iteration from 25 to 10 lowers reward values, while validity, uniqueness, and diversity remain comparable, suggesting that the additional steps primarily refine samples toward higher-reward regions rather than affecting coverage of the output space. These results indicate that a sufficient sampling budget is necessary to realize the benefits of reward guidance, and we use $25$ steps in all main experiments.

\begin{table*}[t]
\caption{\textbf{Ablations for \textit{de novo} small molecule design experiment.} Metrics were computed for $1000$ generated molecule sequences, and the mean and standard deviations were computed across $3$ seeds.}
\label{table:molecule-ablation}
\begin{center}
\begin{small}
\resizebox{\linewidth}{!}{
\begin{tabular}{@{}lcccccc@{}}
\toprule
\textbf{Method} & Validity ($\%$ total) & Uniqueness ($\%$ valid) & Quality ($\%$ total) & QED ($\uparrow$) & Synthetic Accessibility ($\downarrow$) & Diversity ($\uparrow$)  \\
\midrule
Pre-trained (Any Length) &  $95.800_{\pm 0.294}$ & $99.582_{\pm 0.086}$ & $44.167_{\pm 0.574}$ & $0.641_{\pm 0.002}$ & $3.401_{\pm 0.032}$  & $0.908_{\pm 0.000}$  \\
\midrule
\textbf{Reward Scaling $\alpha$} & & & & & &  \\
$\alpha=0.005$ & $97.833_{\pm 0.368}$ & $92.982_{\pm 0.363}$ & $77.633_{\pm 1.237}$ & $0.780_{\pm 0.003}$ & $2.670_{\pm 0.009}$ & $0.828_{\pm 0.001}$ \\
$\alpha=0.01$ (default) & $94.533_{\pm 0.665}$ & $93.695_{\pm 0.925}$ & $71.300_{\pm 0.852}$ & $0.762_{\pm 0.004}$ & $2.870_{\pm 0.018}$ & $0.843_{\pm  0.001}$ \\
$\alpha=0.1$ & $98.400_{\pm 0.245}$ & $57.692_{\pm 1.049}$ & $83.567_{\pm 0.591}$ & $0.783_{\pm 0.003}$ & $2.424_{\pm 0.025}$ & $0.787_{\pm 0.001}$ \\
\midrule
\textbf{Gradient Steps Per Iteration}  & & & & & & \\
$10$ & $97.333_{\pm 0.206}$ & $78.938_{\pm 0.299}$ & $64.967_{\pm 0.189}$ & $0.737_{\pm 0.004}$ & $2.733_{\pm 0.023}$ & $0.829_{\pm 0.003}$ \\
$25$ (default) & $94.533_{\pm 0.665}$ & $93.695_{\pm 0.925}$ & $71.300_{\pm 0.852}$ & $0.762_{\pm 0.004}$ & $2.870_{\pm 0.018}$ & $0.843_{\pm  0.001}$ \\ 
$50$ & $96.600_{\pm 0.572}$ & $88.989_{\pm 0.613}$ & $82.633_{\pm 0.618}$ & $0.820_{\pm 0.002}$ & $2.480_{\pm 0.001}$ & $0.771_{\pm 0.002}$ \\
\bottomrule
\end{tabular}
}
\end{small}
\end{center}
\end{table*}

\begin{table*}[t]
\caption{\textbf{Ablations for multi-objective peptide design.} Metrics were computed for $1000$ generated peptide sequences, and the mean and standard deviations were computed across $3$ seeds.}
\label{table:peptide-ablations}
\begin{center}
\begin{small}
\resizebox{\linewidth}{!}{
\begin{tabular}{@{}lcccccccc@{}}
\toprule
\textbf{Method} & Validity (\%) & Uniqueness ($\uparrow$) & Diversity ($\uparrow$) & Binding Affinity to TfR ($\uparrow$) & Solubility ($\uparrow$) & Non-hemolysis ($\uparrow$) & Non-fouling ($\uparrow$) & Permeability ($\uparrow$) \\
\midrule
 A2D2 w/o quality & $41.267_{\pm 1.520}$ & $100.000_{\pm 0.000}$ & $0.711_{\pm 0.000}$ & $10.241_{\pm 0.012}$ & $0.792_{\pm 0.011}$ & $0.831_{\pm  0.004}$ & $0.096_{\pm 0.001}$ & $-6.102_{\pm 0.006}$ \\
 A2D2 w/o insertion quality & $39.267_{\pm 1.886}$ & $100.000_{\pm 0.000}$ & $0.732_{\pm 0.004}$ & $9.746_{\pm 0.017}$ & $0.772_{\pm 0.007}$ & $0.837_{\pm 0.003}$ & $0.113_{\pm 0.001}$ & $-6.418_{\pm 0.025}$ \\
 A2D2 w/o unmasking quality & $31.600_{\pm 1.219}$ & $100.000_{\pm 0.000}$ & $0.603_{\pm 0.005}$ & $9.670_{\pm 0.058}$ & $0.864_{\pm 0.009}$ & $0.777_{\pm 0.004}$ & $0.119_{\pm 0.001}$ & $-6.337_{\pm 0.016}$ \\
 A2D2 w/ both quality & $48.600_{\pm 1.445}$ & $99.868_{\pm 0.187}$ & $0.578_{\pm 0.003}$ & $10.190_{\pm 0.033}$ & $0.776_{\pm 0.003}$ & $0.831_{\pm 0.003}$ & $0.118_{\pm 0.002}$ & $-6.338_{\pm 0.003}$ \\
\bottomrule
\end{tabular}
}
\end{small}
\end{center}
\end{table*}

\clearpage
\begin{figure*}
    \centering
    \includegraphics[width=\linewidth]{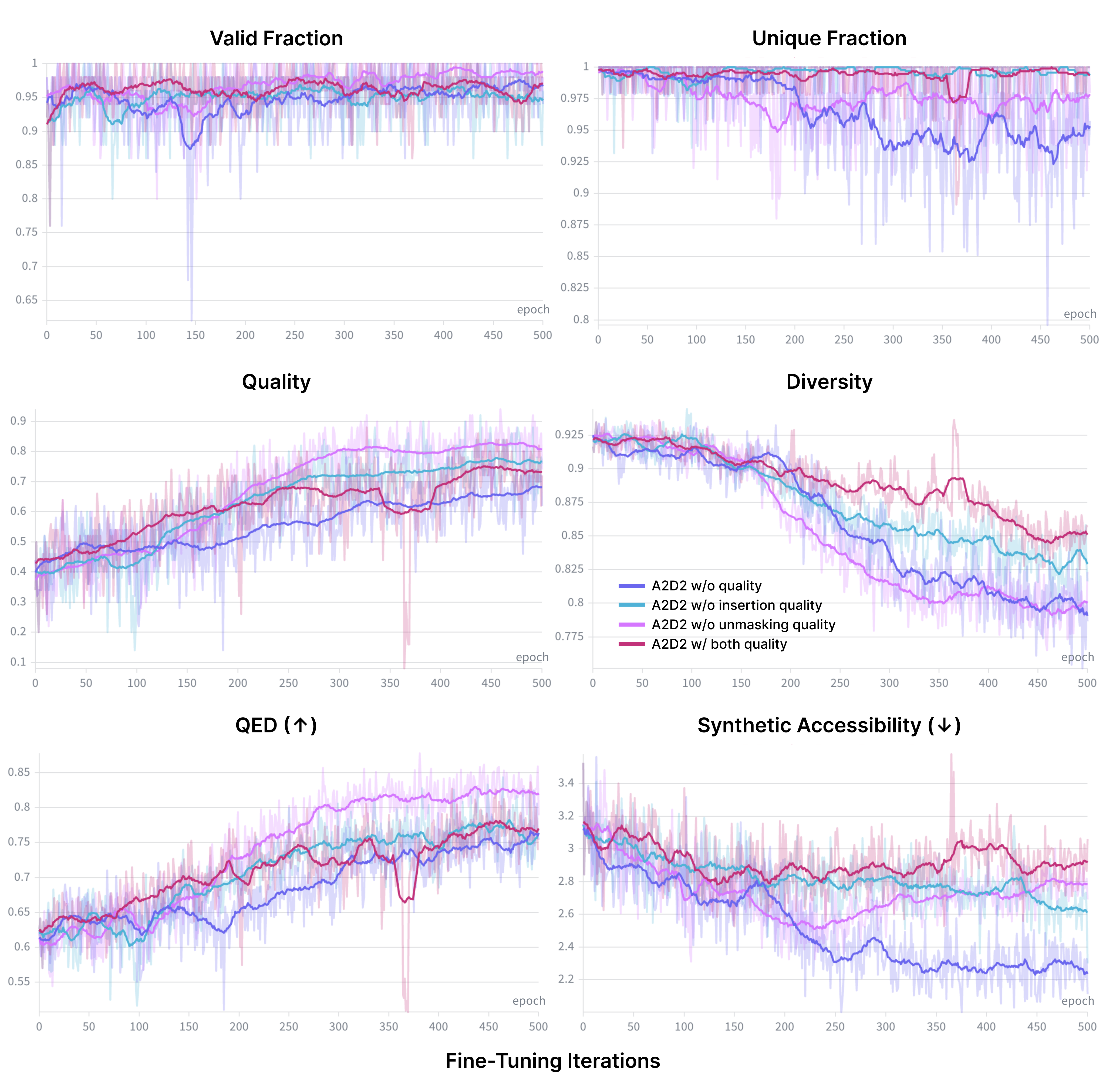}
    \caption{\textbf{Ablation study on insertion and unmasking quality for the small molecule experiment.} Average reward values of $50$ sequences sampled from the fine-tuned model after each fine-tuning epoch are plotted over a total of $500$ iterations, and a running average is shown with the smooth line. }
    \label{fig:mol-ablation}
\end{figure*}

\clearpage
\begin{figure*}
    \centering
    \includegraphics[width=\linewidth]{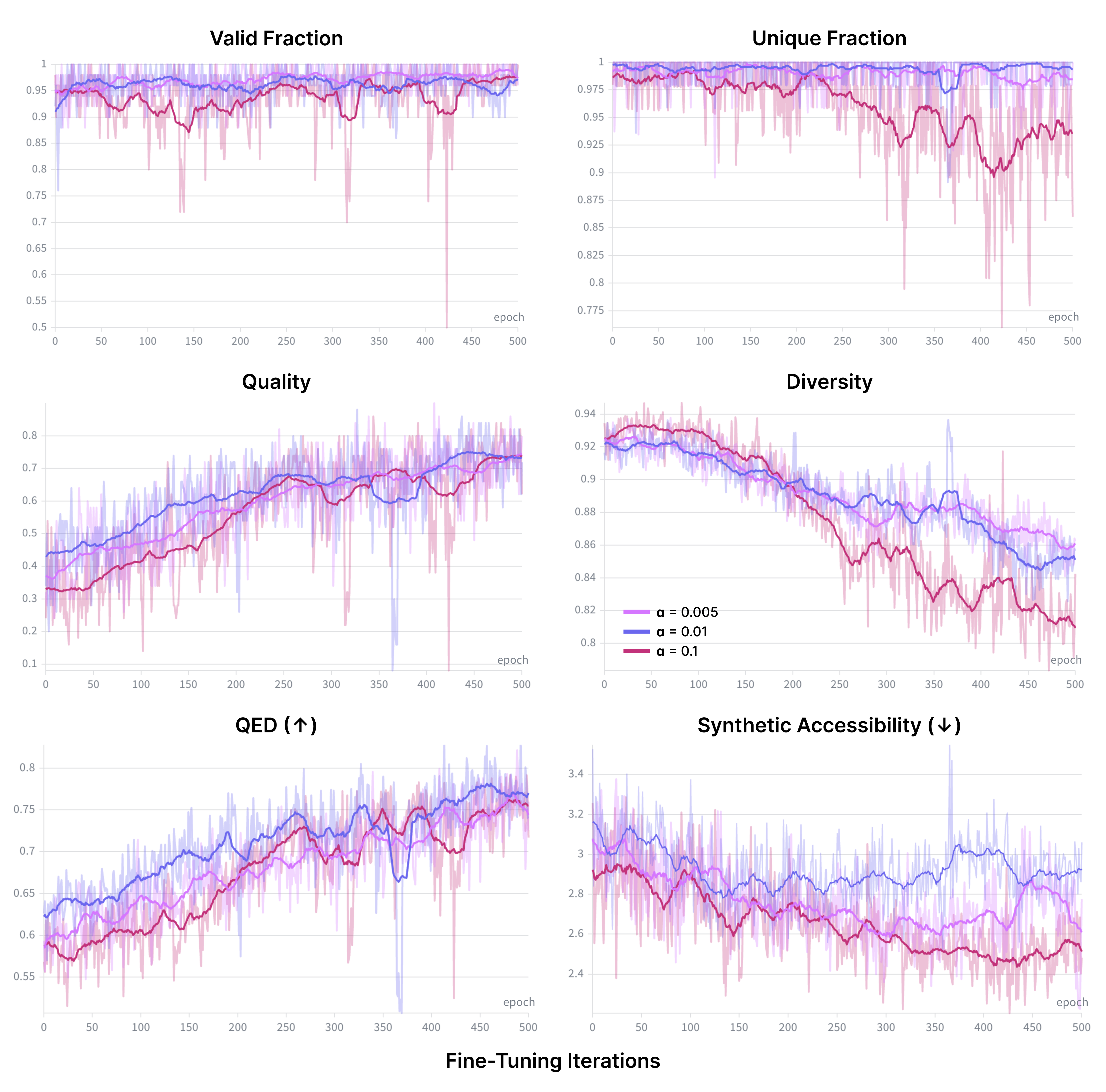}
    \caption{\textbf{Ablation study on reward scaling $\alpha$ for the small molecule experiment.} Average reward values of $50$ sequences sampled from the fine-tuned model after each fine-tuning epoch are plotted over a total of $500$ iterations, and a running average is shown with the smooth line. }
    \label{fig:alpha-ablation}
\end{figure*}

\clearpage
\begin{figure*}
    \centering
    \includegraphics[width=\linewidth]{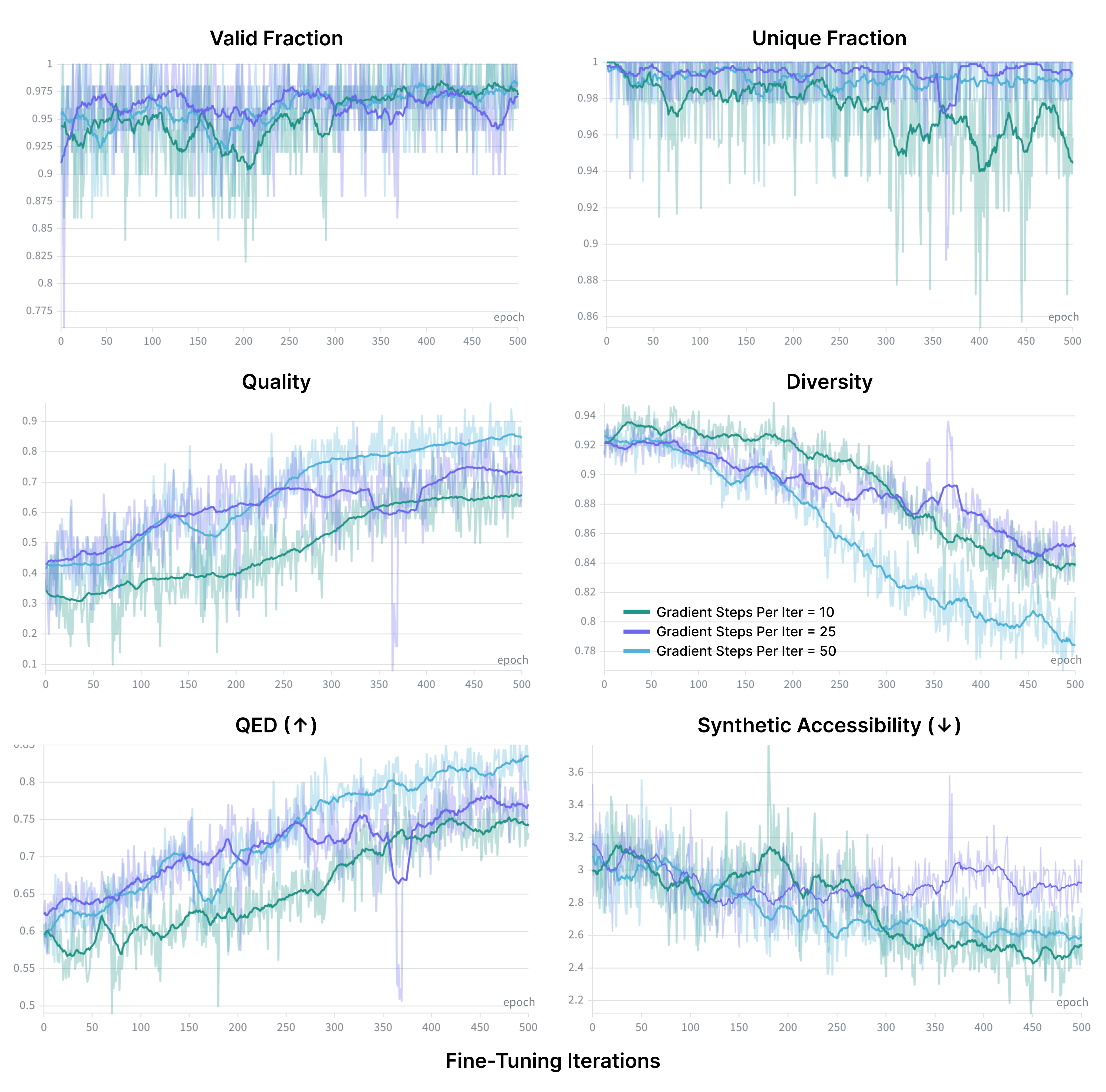}
    \caption{\textbf{Ablation study on the number of gradient steps per iteration for the small molecule experiment.} Average reward values of $50$ sequences sampled from the fine-tuned model after each fine-tuning epoch are plotted over a total of $500$ iterations, and a running average is shown with the smooth line. }
    \label{fig:steps-ablation}
\end{figure*}

\clearpage
\section{Algorithms}
\label{app:algorithms}
Here, we provide the complete pseudo-code for the algorithms used in the \textbf{A2D2} framework. Algorithm \ref{alg:a2d2} describes the joint fine-tuning procedure that alternates between fine-tuning the policy model and the quality planner model to generate from a target reward-tilted distribution. Algorithm \ref{alg:quality-sample} describes the adaptive any-length inference procedure using our trained quality predictors. Algorithm \ref{alg:sample-interpolant} describes the method used to sample an intermediate state $\boldsymbol{x}_t$ from the interpolant given a clean sequence $\boldsymbol{x}_1\sim p_{\text{target}}$. Algorithm \ref{alg:schedule-remask} describes the procedure for adaptive remasking, and Algorithm \ref{alg:schedule-delete} describes the procedure for adaptive deletion.

\begin{algorithm}[h]
\caption{\texttt{A2D2}: Adaptive Any-Length Discrete Diffusion}\label{alg:a2d2}
    \begin{algorithmic}[1]
        \State \textbf{Input:} Pretrained model $f^{\text{pre}}(\boldsymbol{x}_t, t), g^{\text{pre}}(\boldsymbol{x}_t, t)$, fine-tuned model $f_\theta(\boldsymbol{x}_t, t), g_\theta(\boldsymbol{x}_t, t)$, alternation frequency $N_{\text{alt}}$, buffer size $B$, number of repeats $R$
        \While{not converged}
        \State $\{\boldsymbol{x}^\star_{i}, W^{\bar{\theta},\bar{\phi}}\}_{i=1}^{B}\gets \texttt{BatchQualitySample}(f^{\text{pre}}, g^{\text{pre}}, f_{\theta}, g_\theta)$
        \State $\mathcal{B}\gets \{\boldsymbol{x}_{t, i}, W^{\bar{\theta},\bar{\phi}}\}_{i=1}^{B}$
        \For{\texttt{epoch} in $1, \dots, N_{\text{alt}}$}
        \LComment{Fine-tune policy model with frozen planner}
            \State $\{\tilde{\boldsymbol{x}}_{t, i}, W^{\bar{\theta},\bar{\phi}}\}\gets \texttt{SampleInterpolant}(\mathcal{B}; R)$
            \State Predict unmasking posterior $f_\theta(\tilde{\boldsymbol{x}}_{t, i}, t)$
            \State $\mathcal{L}_{\text{unmask}}(\theta; \tilde{\boldsymbol{x}}_{t, i})\gets -\frac{\dot{\beta}_t}{1-\beta_t}\sum_{\ell: \tilde{\boldsymbol{x}}_{t, j}=\boldsymbol{M}}\log f_\theta(\tilde{\boldsymbol{x}}_{t, i}, t)[\ell, \boldsymbol{x}^\star_i]$
            \State Predict insertion expectation $g_\theta(\tilde{\boldsymbol{x}}_{t, i}, t)$
            \State $\mathcal{L}_{\text{insert}}(\theta; \tilde{\boldsymbol{x}}_{t, i})\gets -\frac{\dot{\alpha}_t}{1-\alpha_t}\sum_{\ell=1}^{\text{len}(\tilde{\boldsymbol{x}}_{t, j})+1}\phi\left(s_t[\ell]-s_t[\ell-1]-1,  g_\theta(\tilde{\boldsymbol{x}}_{t, i}, t)[\ell]\right)$
            \State $\mathcal{L}_{\text{AJD}}(\theta)\gets\frac{1}{BR} \sum_{(\tilde{\boldsymbol{x}}_{t, i}, W^{\bar{\theta},\bar{\phi}})\sim \mathcal{B}}\left[e^{W^{\bar{\theta},\bar{\phi}}}(\mathcal{L}_{\text{unmask}}(\theta; \tilde{\boldsymbol{x}}_{t, i}) +\mathcal{L}_{\text{insert}}(\theta; \tilde{\boldsymbol{x}}_{t, i}))\right]$
            \State Update $\theta$ with $\nabla_\theta\mathcal{L}_{\text{AJD}}(\theta)$ 
        \EndFor
        \For{\texttt{epoch} in $1, \dots, N_{\text{alt}}$}
        \LComment{Fine-tune planner model with frozen policy}
            \State $\{\tilde{\boldsymbol{x}}_{t, i}, W^{\bar{\theta},\bar{\phi}}\}\gets \texttt{SampleInterpolant}(\mathcal{B}; R)$
            \State Predict $f_\theta(\tilde{\boldsymbol{x}}_{t, j}, t), g_\theta(\tilde{\boldsymbol{x}}_{t, j}, t)$
            \State $\tilde{\boldsymbol{x}}_{s, i}^{\text{unmask}},\mathcal{M}, \tilde{\boldsymbol{x}}_{s, i}^{\text{insert}},\mathcal{I}\gets \texttt{OneStepSampler}(\tilde{\boldsymbol{x}}_{t, i}, t,f_\theta(\tilde{\boldsymbol{x}}_{t, j}, t), g_\theta(\tilde{\boldsymbol{x}}_{t, j}, t)) $
            \State Predict insertion quality $\nu_\phi(\tilde{\boldsymbol{x}}_{t, i}^{\text{insert}}, s)$
            \State $\mathcal{L}_{\text{IQL}}(\phi;\tilde{\boldsymbol{x}}_{t, i}^{\text{insert}})\gets\sum_{i\in \mathcal{I}}\text{BCE}(\nu_\star^i,\nu_\phi^i(\tilde{\boldsymbol{x}}_{t, i}^{\text{insert}}))$
            \State Predict unmasking quality $\mu _\phi(\tilde{\boldsymbol{x}}_{t, i}^{\text{unmask}}, s)$
            \State $\mathcal{L}_{\text{UQL}}(\phi;\tilde{\boldsymbol{x}}_{t, i}^{\text{unmask}})\gets\sum_{\ell\in \mathcal{M}}\text{BCE}(\boldsymbol{1}[\tilde{\boldsymbol{x}}_{s, i}^{\text{unmask}, \ell}=\boldsymbol{x}^{\star, \ell}_i],\mu_\phi^\ell(\tilde{\boldsymbol{x}}_{s, i}^{\text{unmask}}))$
            \State $\mathcal{L}_{\text{AJD}}(\phi)\gets\frac{1}{BR} \sum_{(\tilde{\boldsymbol{x}}_{t, i}, W^{\bar{\theta},\bar{\phi}}_i)\sim \mathcal{B}}\left[e^{W^{\bar{\theta},\bar{\phi}}}(\mathcal{L}_{\text{IQL}}(\phi; \tilde{\boldsymbol{x}}^{\text{insert}}_{s, i}) +\mathcal{L}_{\text{UQL}}(\phi; \tilde{\boldsymbol{x}}_{s, i}^{\text{unmask}}))\right]$
            \State Update $\phi$ with $\nabla_\phi\mathcal{L}_{\text{AJD}}(\phi)$ 
        \EndFor
        \EndWhile
    \end{algorithmic}
\end{algorithm}

\begin{algorithm}[h]
\caption{\texttt{BatchQualitySample}: Sample a batch of sequences from the fine-tuned model and compute the Radon-Nikodym derivative importance weight}\label{alg:quality-sample}
    \begin{algorithmic}[1]
        \State \textbf{Input:} Pretrained model $f^{\text{pre}}(\boldsymbol{x}_t, t), g^{\text{pre}}(\boldsymbol{x}_t, t)$, fine-tuned model $f_\theta(\boldsymbol{x}_t, t), g_\theta(\boldsymbol{x}_t, t)$, reward model $r:\mathcal{X}\to\mathbb{R}$, reward scaling $\alpha$, sampling steps $N_{\text{steps}}$, max length $L$, tokens to remask $N_{\text{remask}}$
        \State $\Delta t\gets T/N_{\text{steps}}$
        \State $t\gets 0$
        \For{$i=1$ to $N_{\text{steps}}$}
        \State $\boldsymbol{R}^\theta_t(\boldsymbol{x}_t, \boldsymbol{x}_s)\gets\frac{\dot{\alpha}_t}{1-\alpha_t}\cdot g_\theta(\boldsymbol{x}_t,t)$\Comment{Predict insertion expectation}
        \State $\boldsymbol{Q}^\theta_t(\boldsymbol{x}_t, \boldsymbol{x}_s)\gets \frac{\dot{\beta}_t}{1-\beta_t}\cdot f_\theta(\boldsymbol{x}_t,t)$\Comment{Predict unmasking posterior}
        \State $\boldsymbol{R}^{\text{pre}}_t(\boldsymbol{x}_t, \boldsymbol{x}_s)\gets\frac{\dot{\alpha}_t}{1-\alpha_t}\cdot g^{\text{pre}}(\boldsymbol{x}_t,t)$\Comment{Predict pretrained insertion expectation}
        \State $\boldsymbol{Q}^{\text{pre}}_t(\boldsymbol{x}_t, \boldsymbol{x}_s)\gets \frac{\dot{\beta}_t}{1-\beta_t}\cdot f^{\text{pre}}(\boldsymbol{x}_t,t)$\Comment{Predict pretrained unmasking posterior}
        \State $\boldsymbol{x}_s\gets \texttt{SampleCategorical}(\boldsymbol{Q}^\theta_t(\boldsymbol{x}_t, \boldsymbol{x}_s))$\Comment{Unmasking step}
        \If{$i<N_{\text{steps}}-1$}
        \State $\boldsymbol{x}_s\gets \texttt{ScheduleAwareRemasking}(\boldsymbol{x}_s, \mu_\phi)$
        \EndIf
        \State $\texttt{log\_rnd}\gets \texttt{log\_rnd}+\sum_{\ell}\boldsymbol{1}[\boldsymbol{x}^\ell_s\neq \boldsymbol{x}^\ell_t]\cdot( \log \boldsymbol{ Q}_t^{\text{pre}}(\boldsymbol{x}_s, \boldsymbol{x}_t)[\ell, \boldsymbol{x}_s^\ell]-\log \boldsymbol{ Q}_t^\theta(\boldsymbol{x}_s, \boldsymbol{x}_t)[\ell, \boldsymbol{x}_s^\ell]) $
        \If{$i<N_{\text{steps}}-1$}
        \State $L_t\gets \text{len}(\boldsymbol{x}_t)$
        \State $I_t\sim \texttt{Poisson}(\boldsymbol{R}^\theta_t(\boldsymbol{x}_t, \boldsymbol{x}_s)\cdot \Delta t)$\Comment{Sample number of tokens to insert at each gap}
        \State $\texttt{total\_ext}\gets\sum_\ell L_t^\ell$
        \If{$\texttt{total\_ext}+L_t> L_{\text{max}}$}
        \State $L_t\gets 0$\Comment{Do not insert any masks}
        \EndIf
        \State $L_s\gets L_t+\texttt{total\_ext}$
        \State $\texttt{ext\_cum}[L_t+1]\gets \sum_{j=1}^\ell I_t^\ell$\Comment{Compute cumulative insertions up to position $\ell$}
        \State $\boldsymbol{x}_s'\gets (\text{pad})^{L_{\text{max}}}$\Comment{Initialize sequence of pad tokens}
        \State $\boldsymbol{x}_s'[: L_s+1]\gets \boldsymbol{M}$ \Comment{Initialize $L_s$ tokens to mask}
        \For{$\ell=1$ to $L_t$}
        \State $\boldsymbol{x}_s'[\ell+\texttt{ext\_cum}[\ell]]\gets \boldsymbol{x}_s[\ell]$\Comment{Set unmasked tokens to their new position}
        \EndFor
        \State $\boldsymbol{x}_s', I ^\star_t\gets \texttt{ScheduleAwareDeletion}(\boldsymbol{x}_s', \nu_\phi)$
        \State $\boldsymbol{x}_s\gets \boldsymbol{x}_s'$
        \LComment{Compute log insertion rates under Poisson}
        \State $\texttt{log\_policy\_insert}\gets I^{\star}_t\log \boldsymbol{R}^\theta_t(\boldsymbol{x}_t, \boldsymbol{x}_s)-\boldsymbol{R}^\theta_t(\boldsymbol{x}_t, \boldsymbol{x}_s)$
        \State $\texttt{log\_pre\_insert}\gets I^{\star}_t\log \boldsymbol{R}^{\text{pre}}_t(\boldsymbol{x}_t, \boldsymbol{x}_s)-\boldsymbol{R}^{\text{pre}}_t(\boldsymbol{x}_t, \boldsymbol{x}_s)$
        \State $\texttt{log\_rnd}\gets \texttt{log\_rnd}+\sum_\ell(\texttt{log\_pre\_insert}[\ell ]-\texttt{log\_policy\_insert}[\ell])$
        \EndIf
        \State $\boldsymbol{x}_t\gets \boldsymbol{x}_s$
        \State $t\gets t+\Delta t$
        \EndFor
        \State $r(\boldsymbol{x}_1)\gets \texttt{RewardFunc}(\texttt{decode}(\boldsymbol{x}_1))$
        \State $\texttt{log\_rnd}\gets \texttt{log\_rnd}+(r(\boldsymbol{x}_t)/\alpha)$
        \State \textbf{return} $\boldsymbol{x}_1, \texttt{log\_rnd}, r(\boldsymbol{x}_1)$
    \end{algorithmic}
\end{algorithm}

\begin{algorithm}[h]
\caption{\texttt{SampleInterpolant}: Partially remasks and deletes tokens in batch}\label{alg:sample-interpolant}
    \begin{algorithmic}[1]
        \State \textbf{Input:} Batch of clean sequences $\{\boldsymbol{x}_{T, i}, W^{\bar{\theta},\bar{\phi}}\}_{i=1}^B$, number of replicates $R$
        \State Repeat each sequence in batch $\{\boldsymbol{x}_{T, j}, W^{\bar{\theta},\bar{\phi}}\}_{j=1}^{B\times R}$
        \State Initialize batch of corrupted sequences $\{\boldsymbol{x}_{t, j}\}_{j=1}^{B\times R}$
        \For{$\boldsymbol{x}_{T, j}$ in $\{\boldsymbol{x}_{T, j}\}_{j=1}^{B\times R}$} 
        \State $t\sim \mathcal{U}(0,1)$
        \For{$\ell=1$ to $L$}
        \State $t_i^\ell\sim \dot{\alpha}_tdt$\Comment{Sample insertion time}
        \State $t_u^\ell\sim \boldsymbol{1}[t\geq t_i^\ell]\frac{\dot{\beta}_t}{1-\beta_{t_i^\ell}}dt$ \Comment{Sample unmasking time}
        \If{$t<t_i^\ell$}
        \State $\boldsymbol{x}_{t, j}^\ell \gets \text{(empty)}$]\Comment{Deleted token at time $t$}
        \ElsIf{$t_i^\ell \leq t< t_u^\ell$}
        \State $\boldsymbol{x}_{t, j}^\ell\gets \boldsymbol{M}$ \Comment{Masked token at time $t$}
        \Else
        \State $\boldsymbol{x}_{t, j}^\ell\gets \boldsymbol{x}_{1, j}^\ell$ \Comment{Clean token at time $t$}
        \EndIf
        \State Remove intermediate (empty) from $\boldsymbol{x}_{t, j}^\ell$ and replace as pads at end of sequence
        \EndFor
        \EndFor
        \State \textbf{return} $\{\boldsymbol{x}_{t, j}, W^{\bar{\theta},\bar{\phi}}\}_{j=1}^{B\times R}$
    \end{algorithmic}
\end{algorithm}

\begin{algorithm}[h]
\caption{\texttt{ScheduleAwareRemasking}: Ensure the number of masked tokens after remasking matches the expected mask count by remasking low-quality tokens.}\label{alg:schedule-remask}
    \begin{algorithmic}[1]
        \State \textbf{Input:} Draft sequence after unmasking $\tilde{\boldsymbol{x}}_s$, time $s$, unmasking quality model $\mu_\phi$, unmasked indices $\mathcal{M}_t$
        \State $\texttt{exp\_num\_mask}\gets \texttt{interpolant.exp\_mask\_fraction}(\tilde{\boldsymbol{x}}_s, s)$
        \State $\texttt{num\_mask}\gets |\{\ell: \tilde{\boldsymbol{x}}^\ell_s= \boldsymbol{M}\}|$
        \State $\texttt{mask\_to\_add}\gets \max(0, \texttt{exp\_num\_mask}-\texttt{num\_mask})$
        \For{$\ell  \in \mathcal{M}_t$}
        \State $\texttt{remask\_scores}[\ell]\gets -\mu_\phi^\ell(\tilde{\boldsymbol{x}}_s)$\Comment{lower confidence has higher probability of remask}
        \EndFor
        \State $\mathcal{R}_s\gets \texttt{topk\_indices}(\texttt{remask\_scores}, \texttt{mask\_to\_add})$\Comment{set of indices to remask}
        \For{$i \in \mathcal{R}_s$}
        \State $\tilde{\boldsymbol{x}}_s^\ell\gets \boldsymbol{M}$
        \EndFor 
        \State \textbf{return} $\tilde{\boldsymbol{x}}_s$
    \end{algorithmic}
\end{algorithm}

\begin{algorithm}[h]
\caption{\texttt{ScheduleAwareDeletion}: Ensure the sequence length after insertion matches the expected length by deleting low-quality insertions}\label{alg:schedule-delete}
    \begin{algorithmic}[1]
        \State \textbf{Input:} Draft sequence after insertion $\tilde{\boldsymbol{x}}_s$, time $s$, insertion quality model $\nu_\phi$, inserted indices $\mathcal{I}_t$, threshold quality $\mu_{\text{min}}$
        \State $\texttt{exp\_length}\gets \texttt{interpolant.exp\_length}(\tilde{\boldsymbol{x}}_s, s)$
        \If{$\texttt{len}(\tilde{\boldsymbol{x}}_s)\geq \texttt{exp\_length}$}
        \State $\texttt{to\_delete}\gets \texttt{len}(\tilde{\boldsymbol{x}}_s)-\texttt{exp\_length}$
        \State $\texttt{insert\_scores}[i]\gets -\nu_\phi^i(\tilde{\boldsymbol{x}}_s)$\Comment{lower confidence has higher probability of deletion}
        \State $\mathcal{D}_s\gets \texttt{topk\_indices}(\texttt{insert\_scores}, \texttt{to\_delete})$\Comment{set of indices to delete}
        \State $\tilde{\boldsymbol{x}}_s^i\gets\text{(empty)}$
        \EndIf
        \State \textbf{return} $\tilde{\boldsymbol{x}}_s$, $\mathcal{I}_t/\mathcal{D}_s$
    \end{algorithmic}
\end{algorithm}

\clearpage

\section{Example Generations}
\begin{figure}[h!]
\centering
\begin{figurepanel}

\noindent
\begin{minipage}[t]{0.485\linewidth}
\nfeheader{Sampling Steps = 128}

\vspace{8pt}
\begin{samplebox}
\noindent
\begin{minipage}[t]{0.56\linewidth}
\methodtitle{A2D2}
\end{minipage}\hfill
\begin{minipage}[t]{0.40\linewidth}
\raggedleft \textbf{answer 4} \checkmark
\end{minipage}

\vspace{4pt}
\begin{lstlisting}[style=gsmtext]
Q: A juggler can juggle 16 balls. Half of
the balls are golf balls, and half of the
golf balls are blue. How many blue golf
balls are there?

Golf balls: 16/2 = 8
Blue: 8/2 = 4 balls
Answer: \boxed{4}
\end{lstlisting}
\end{samplebox}

\vspace{8pt}
\nfeheader{Sampling Steps = 512}

\vspace{8pt}
\begin{samplebox}
\noindent
\begin{minipage}[t]{0.56\linewidth}
\methodtitle{A2D2}
\end{minipage}\hfill
\begin{minipage}[t]{0.40\linewidth}
\raggedleft \textbf{answer 16} \checkmark
\end{minipage}

\vspace{4pt}
\begin{lstlisting}[style=gsmtext]
Q: Jenny goes to the florist to buy some
flowers. Roses cost $2 each and $15 for a
dozen. If she bought 15 roses and arrived
with five 5 dollar bills and they only have
quarters for change, how many quarters does
she leave with?

First find how many individual roses Jenny
bought: 15 roses - 12 roses = 3 roses
Then multiply by the price per rose:
3 roses * $2/rose = $6
Then add the cost of the dozen roses:
$6 + $15 = $21
Then subtract from the amount Jenny paid:
$25 - $21 = $4
Then divide the change by the value of a
quarter: $4 / .25 = 16 quarters
Answer: \boxed{16}
\end{lstlisting}
\end{samplebox}
\end{minipage}
\hfill
\begin{minipage}[t]{0.485\linewidth}
\nfeheader{Sampling Steps = 256}

\vspace{8pt}
\begin{samplebox}
\noindent
\begin{minipage}[t]{0.56\linewidth}
\methodtitle{A2D2}
\end{minipage}\hfill
\begin{minipage}[t]{0.40\linewidth}
\raggedleft \textbf{answer 70} \checkmark
\end{minipage}

\vspace{4pt}
\begin{lstlisting}[style=gsmtext]
Q: Gretchen has 110 coins. There are 30
more gold coins than silver coins. How many
gold coins does Gretchen have?

Let g be the number of gold coins and s be
the number of silver coins. We know that
g + s = 110 and g = s + 30.
Substituting the second equation into the
first, we get s + 30 + s = 110
Combining like terms, we get 2s + 30 = 110
Subtracting 30 from both sides, 2s = 80
Dividing both sides by 2, we get s = 40
Substituting back into g = s + 30, we get
g = 40 + 30 = 70
Answer: \boxed{70}
\end{lstlisting}
\end{samplebox}

\vspace{8pt}
\nfeheader{Sampling Steps = 1024}

\vspace{8pt}
\begin{samplebox}
\noindent
\begin{minipage}[t]{0.56\linewidth}
\methodtitle{A2D2}
\end{minipage}\hfill
\begin{minipage}[t]{0.40\linewidth}
\raggedleft \textbf{answer 40} \checkmark
\end{minipage}

\vspace{4pt}
\begin{lstlisting}[style=gsmtext]
Q: Mico and Marco wanted to get to know
each other. They realized that the sum of
their ages is 20. What will be the sum of
their ages in 10 years?

Let's say Mico's age is m.
So Marco's age is 20 - m.
In 10 years, Mico will be m + 10 years old
and Marco will be 20 - m + 10 years old.
So the sum of their ages will be
m + 10 + (20 - m + 10) = 40.
Answer: \boxed{40}
\end{lstlisting}
\end{samplebox}
\end{minipage}

\end{figurepanel}
\caption{Example GSM8K generations from the any-length model fine-tuned with A2D2 at 128, 256, 512, and 1024 sampling steps. The model is given the problem and generates step-by-step reasoning ending in a \texttt{\textbackslash boxed\{\}} answer. All four reach the correct answer.}
\label{fig:gsm8k-samples}
\end{figure}

\begin{figure}[h!]
\centering
\begin{figurepanel}

\noindent
\begin{minipage}[t]{0.485\linewidth}
\nfeheader{Sampling Steps = 128}

\vspace{8pt}
\begin{samplebox}
\noindent
\begin{minipage}[t]{0.56\linewidth}
\methodtitle{A2D2}
\end{minipage}\hfill
\begin{minipage}[t]{0.40\linewidth}
\raggedleft \textbf{reward 2.25}
\end{minipage}

\vspace{4pt}
\begin{lstlisting}[style=infillpy]
def f(n):
    """ Implement the function f that takes n as a
    parameter, and returns a list of size n, such that
    the value of the element at index i is the factorial
    of i if i is even or the sum of numbers from 1 to i
    otherwise. i starts from 1.
    Example: f(5) == [1, 2, 6, 24, 15]
    """
    ret = []
    for i in range(1,n+1):
        if i%2 == 0:
            x = 1
            for j in range(1,i+1): x *= j  # <-- infilled
            ret += [x]
        else:
            x = 0
            for j in range(1,i+1): x += j
            ret += [x]
    return ret
\end{lstlisting}
\end{samplebox}

\vspace{8pt}
\nfeheader{Sampling Steps = 512}

\vspace{8pt}
\begin{samplebox}
\noindent
\begin{minipage}[t]{0.56\linewidth}
\methodtitle{A2D2}
\end{minipage}\hfill
\begin{minipage}[t]{0.40\linewidth}
\raggedleft \textbf{reward 2.25}
\end{minipage}

\vspace{4pt}
\begin{lstlisting}[style=infillpy]
def intersperse(numbers, delimeter):
    """ Insert a number 'delimeter' between every two
    consecutive elements of input list `numbers`.
    >>> intersperse([1, 2, 3], 4)
    [1, 4, 2, 4, 3]
    """
    if not numbers:
        return []
    result = []
    for n in numbers[:-1]:
        result.append(n)
        result.append(delimeter)
    result.append(numbers[-1])  # <-- infilled
    return result
\end{lstlisting}
\end{samplebox}
\end{minipage}
\hfill
\begin{minipage}[t]{0.485\linewidth}
\nfeheader{Sampling Steps = 256}

\vspace{8pt}
\begin{samplebox}
\noindent
\begin{minipage}[t]{0.56\linewidth}
\methodtitle{A2D2}
\end{minipage}\hfill
\begin{minipage}[t]{0.40\linewidth}
\raggedleft \textbf{reward 2.25}
\end{minipage}

\vspace{4pt}
\begin{lstlisting}[style=infillpy]
def largest_smallest_integers(lst):
    '''
    Returns a tuple (a, b), where 'a' is the largest of
    negative integers, and 'b' is the smallest of positive
    integers in a list. If there is no negative or positive
    integers, return them as None.
    Examples:
    largest_smallest_integers([2, 4, 1, 3, 5, 7]) == (None, 1)
    largest_smallest_integers([]) == (None, None)
    '''
    smallest = list(filter(lambda x: x < 0, lst))
    largest = list(filter(lambda x: x > 0, lst))  # <-- infilled
    return (max(smallest) if smallest else None,
            min(largest) if largest else None)
\end{lstlisting}
\end{samplebox}

\vspace{8pt}
\nfeheader{Sampling Steps = 1024}

\vspace{8pt}
\begin{samplebox}
\noindent
\begin{minipage}[t]{0.56\linewidth}
\methodtitle{A2D2}
\end{minipage}\hfill
\begin{minipage}[t]{0.40\linewidth}
\raggedleft \textbf{reward 2.25}
\end{minipage}

\vspace{4pt}
\begin{lstlisting}[style=infillpy]
def specialFilter(nums):
    """Returns the number of elements in the array that
    are greater than 10 and both first and last digits of
    a number are odd (1, 3, 5, 7, 9).
    specialFilter([15, -73, 14, -15]) => 1
    specialFilter([33, -2, -3, 45, 21, 109]) => 2
    """
    count = 0
    for num in nums:
        if num > 10:
            odd_digits = (1, 3, 5, 7, 9)
            number_as_string = str(num)  # <-- infilled
            if int(number_as_string[0]) in odd_digits and \
               int(number_as_string[-1]) in odd_digits:
                count += 1
    return count
\end{lstlisting}
\end{samplebox}
\end{minipage}

\end{figurepanel}
\caption{Example HumanEval single-line infilling generations from the any-length model fine-tuned with A2D2 at 128, 256, 512, and 1024 sampling steps. The model is given the function prefix and suffix and infills the line marked \texttt{\# <-- infilled}. All four reconstruct the ground-truth line exactly.}
\label{fig:infill-samples}
\end{figure}

\end{document}